\documentclass[10pt]{article} 
\usepackage[preprint]{tmlr}
\usepackage{hyperref}
\usepackage{url}
\usepackage{pifont}

\usepackage[margin=1in]{geometry}
\usepackage{tabularx}
\usepackage{booktabs}
\usepackage{makecell}
\usepackage{amsmath, amssymb}
\usepackage{lipsum}
\usepackage{amsmath}
\usepackage{amsthm,amsxtra,mathtools}
\usepackage{dsfont,bm}
\usepackage{array}
\usepackage{graphicx}
\usepackage{algorithm}
\usepackage{algpseudocode}
\usepackage{physics}
\usepackage{bbm}
\usepackage{array,tabularx,multirow}
\usepackage{makecell}
\usepackage{color, soul}

\usepackage{verbatim}
\usepackage{parskip}
\usepackage{parskip}
\usepackage{thmtools}
\usepackage{hyperref,subcaption}

\theoremstyle{plain}
\theoremstyle{remark}

\definecolor{clemson-orange}{RGB}{234,106,32}
\definecolor{chicago-maroon}{RGB}{128,0,0}
\definecolor{cincinnati-red}{RGB}{190,0,0}
\definecolor{soft-cyan}{RGB}{68,85,90}
\definecolor{firebrick}{RGB}{178,34,34}
\definecolor{crimson}{RGB}{220,20,60}
\definecolor{cerrulean}{rgb}{0.165,0.322,0.745}
\definecolor{jaam}{rgb}{0.45,0.0,0.45}

\hypersetup{
  colorlinks   = true,    
  urlcolor     = jaam,    
  linkcolor    = firebrick,    
  citecolor    = blue      
}

\declaretheoremstyle[
    headfont=\bfseries, 
    bodyfont=\normalfont\itshape, spaceabove=10pt,
    spacebelow=10pt]{mystyle}
\theoremstyle{mystyle}
\theoremstyle{definition}
\newtheorem{theorem}{Theorem}[section]
\newtheorem{lemma}[theorem]{Lemma}

\newtheorem{definition}{Definition}
\newtheorem*{remark}{Remark}
\newtheorem{assumption}{Assumption}

\newif\ifsolutions \solutionstrue

\def\final{0}
\newcommand{\reviewer}[3]{
  \expandafter\newcommand\csname #1\endcsname[1]{
    \ifthenelse{\equal{\final}{1}} {
      \textcolor{#3}{}
    } {
      \textcolor{#3}{\begin{center} \textbf{#2} ##1 \end{center}}
    }
  }
}

\reviewer{anirbit}{}{jaam}
\newcolumntype{M}[1]{>{\centering\arraybackslash}m{#1}}

\renewcommand{\Pr}{\mathbb{P}}

\newcommand{\vxi}{\boldsymbol{\xi}}

\def\1{\bm{1}}

\newcommand{\E}{\mathbb{E}}

\def\vzero{{\bm{0}}}

\def\vtheta{{\bm{\theta}}}

\def\vb{{\bm{b}}}

\def\vg{{\bm{g}}}

\def\vu{{\bm{u}}}
\def\vv{{\bm{v}}}
\def\vw{{\bm{w}}}
\def\vx{{\bm{x}}}
\def\vy{{\bm{y}}}

\DeclareMathAlphabet{\mathsfit}{\encodingdefault}{\sfdefault}{m}{sl}
\SetMathAlphabet{\mathsfit}{bold}{\encodingdefault}{\sfdefault}{bx}{n}

\newcommand{\R}{\mathbb R}

\title{Convergence of Stochastic Gradient Methods under Heavy-Tailed Noise and H\"{o}lder Smoothness}
\author{\name Misbah Uz Zaman \email muz21ms243@iiserkol.ac.in \\
       \addr Department of Mathematics and Statistics\\
       Indian Institute of Science Education and Research Kolkata
     \AND   
     \name Anirbit Mukherjee\email anirbit.mukherjee@manchester.ac.uk \\
      \addr Department of Computer Science\\
      The University of Manchester
      }

\def\month{MM}  
\def\year{YYYY} 
\def\openreview{\url{https://openreview.net/forum?id=XXXX}} 
\begin{document}
\maketitle
\begin{abstract}

Classical convergence guarantees for stochastic gradient methods typically assume Lipschitz-smooth
objectives and finite-variance gradient noise, both frequently violated in practice. In contrast, we study
nonconvex stochastic optimization under the joint relaxation of these assumptions: objectives with
$(L,s)$-H\"older continuous gradients, $s\in(0,1]$, and gradient noise satisfying only a bounded
$\alpha$-th moment condition for $\alpha\in(1,2]$. We establish three
convergence results. Firstly, that standard SGD converges at rate $O(T^{-s/(1+s)})$ whenever $\alpha\ge1+s$,
extending the classical nonconvex SGD rate to heavy-tailed noise and Hölder smoothness simultaneously.
Secondly, we analyze $\delta$-regularized gradient clipping ($\delta$-GClip), a provable trainer of wide and deep nets, and establish a stationarity rate of
$O(T^{-2s(\alpha-1)/[(1+s)(2\alpha-1)]})$ under the same condition. Thirdly, we analyze standard
gradient clipping (G-Clip) and show that it recovers the above rate for $\alpha\ge1+s$ while in the very heavy-tailed regime $\alpha<1+s$, it has a convergence rate
$O(T^{-2s(\alpha-1)/[(\alpha-1)+s(2\alpha-1)]})$ --- the first convergence guarantee in this
regime for any stochastic gradient based method. 

\end{abstract}
    
\section{Introduction}

Many tasks in machine learning can be formulated as Empirical Risk Minimization (ERM), where the objective is to minimize
$\hat{R}_n(\theta) = \frac{1}{n} \sum_{i=1}^n \ell(f(x_i;\theta), y_i)$,
based on a dataset $\{(x_i,y_i)\}_{i=1}^n$. In large-scale settings, computing the full gradient is often prohibitively expensive. Stochastic Gradient Descent (SGD) addresses this challenge by using a single sample (or mini-batch) to form an unbiased gradient estimate, reducing the per-iteration cost from $O(n)$ to $O(1)$. The theoretical foundations of such stochastic methods date back to \citet{robbins1951} --- the seminal work that established convergence of iterative procedures using noisy observations under suitable diminishing step sizes. 




A finite-time analysis of SGD in nonconvex settings was developed by \citet{ghadimi2013}, who showed that, under Lipschitz smoothness and bounded variance assumptions, a randomized iterate $x_R$ selected from the first $N$ iterations satisfies $\mathbb{E}\|\nabla f(x_R)\|^2 = O\!\left(\frac{1}{\sqrt{N}}\right)$,
yielding an iteration complexity of $O(\varepsilon^{-4})$ for finding an $\varepsilon$-stationary point.

However, these classical guarantees rely on assumptions that are often violated in modern machine learning applications. \citet{simsekli2019tail} provides empirical evidence across a wide range of architectures, loss functions, and datasets, showing that gradient noise is well-described by $\alpha$-stable distributions with $\alpha<2$ --- where $\alpha \in (1, 2]$ acts as the heavy-tail index characterizing the noise, specifically defining the highest finite moment of the gradient noise that we assume to be bounded, indicating heavy-tailed behavior and suggesting potentially infinite variance.
Furthermore, \citet{suvrit1912} show that  the noise in BERT pretraining exhibits pronounced heavy-tailed behavior, with empirical variance failing to stabilize even at very large sample sizes. Crucially, they also establish theoretical lower bounds for optimization under heavy-tailed noise and prove that gradient clipping mechanisms achieve statistically optimal convergence rates in this regime.
Recently, \citet{gong2025adaptive} proposed a stochastic optimization method that dynamically adapts the tail index of injected $\alpha$-stable noise based on the evolving sharpness of the loss landscape, with empirical results suggesting improved optimization behavior in practice. Taken together, these findings highlight heavy-tailed gradient noise as a practically significant phenomenon in modern deep learning, and motivate the need for optimization algorithms whose convergence guarantees do not rely on finite-variance assumptions.

Beyond the noise model, classical analyses also typically assume that the objective function has a globally Lipschitz continuous 
gradient. This assumption, however, is frequently violated in practice. To capture 
a broader class of objectives, the following generalization has often been considered: a 
differentiable function $f$ is said to be $(L,s)$-H\"{o}lder smooth for 
parameters $L > 0$ and $s \in (0,1]$ if
    $\|\nabla f(\vx) - \nabla f(\vy)\| 
    \leq 
    L\|\vx - \vy\|^s 
    \quad \forall\, \vx, \vy \in \mathbb{R}^d$,
which recovers standard Lipschitz smoothness when $s = 1$.

The above condition arises naturally in machine learning. \citet{wang2021differentially} 
explicitly identify several widely used loss functions that fail to satisfy global 
Lipschitz smoothness: the $q$-norm soft-margin hinge loss 
$(1 - y\mathbf{w}^\top \mathbf{x})_+^q$ for SVMs and the $q$-norm regression loss 
$|y - \mathbf{w}^\top \mathbf{x}|^q$ for $q \in (1,2]$, both of which satisfy $(q-1)$-H\"older continuity of the gradient, i.e., satisfy the above condition with $s = q - 1$. 
When the  smooth hinge loss with parameter $s \in (0,1)$ is used in SVM training, the 
gradient of the resulting empirical risk satisfies the $s$-H\"{o}lder condition 
but fails to be Lipschitz continuous \citet{wang2023holder}. These examples demonstrate that H\"older smoothness arises naturally
in standard supervised learning problems and therefore constitutes a
practically relevant relaxation of classical smoothness assumptions.

Importantly, \citet{wang2021differentially} show that $s \ge 1/2$ is already sufficient to achieve an excess population risk of $\mathcal{O}\big(\frac{\sqrt{d \log(1/\delta)}}{n\epsilon} + \frac{1}{\sqrt{n}}\big)$ with linear gradient complexity $\mathcal{O}(n)$. This upper bound perfectly matches the known information-theoretic lower bounds for differentially private stochastic convex optimization \citet{bassily2019private}, demonstrating that full Lipschitz smoothness is not strictly necessary to attain statistically optimal learning rates.

To the best of our knowledge, there are no existing results that analyze the convergence of SGD under heavy-tailed noise for objectives with $s$-H\"older continuous gradients --- the existing literature has largely treated these challenges separately.

Standard gradient clipping (GClip), studied by \citet{suvrit1912}, carries a fundamental theoretical limitation despite its favorable convergence guarantees under heavy-tailed noise in the Lipschitz smooth setting --- that it has no convergence guarantee for training neural networks.


To overcome this limitation,
\citet{tucat2025regularized} introduced the
$\delta$-regularized gradient clipping ($\delta$-GClip) algorithm which 
modifies standard GClip by enforcing a strictly positive lower bound on the
adaptive scaling factor, thereby preventing the effective step size from
collapsing to zero when gradient norms become very large. Specifically,
$\delta$-GClip applies the update
$
\vx_{k+1}
=
\vx_k
-
\eta
\min\!\left\{
1,\,
\max\!\left\{
\delta,\,
\frac{\gamma}{\|\vg_k\|}
\right\}
\right\}
\vg_k,
$
for a parameter $\delta\in(0,1)$. This modification admits a natural
interpretation as a continuous interpolation between gradient descent
($\delta=1$) and standard gradient clipping ($\delta=0$), with intermediate
values $\delta\in(0,1)$ inheriting desirable properties from both extremes. Crucially, \citet{tucat2025regularized} proved that for sufficiently wide
deep neural networks trained with the squared loss, whose loss landscape
satisfies the $\mu$-PL$^{*}$ condition in a neighborhood of initialization,
$\delta$-GClip admits geometric convergence guarantees to global minima. Thus \citet{tucat2025regularized}
demonstrated the first adaptive
gradient method  with provable global convergence guarantee for deep neural networks of arbitrary depth.

 In the stochastic setting,
\citet{tucat2025regularized} establish stationarity guarantees only under the
bounded noise assumption
$
\|g(\vw)-\nabla L(\vw)\|\leq\theta,
$
which rules out heavy-tailed gradient noise. Moreover, the smoothness
assumption used throughout \citet{tucat2025regularized} is also restricted to the
Lipschitz continuity of the gradient. 

{\em Thus motivated, we ask
whether convergence properties of $\delta$-GClip extend to
settings where simultaneously the gradient noise is heavy-tailed and the objective is only
$(L,s)$-H\"older smooth. To the best of our knowledge, this setting albeit natural has
not been previously analyzed for clipping-based stochastic optimization
methods, which can also train nets provably.}
{\em Secondly, we also ask whether for standard gradient clipping it admits convergence
guarantees under the simultaneous presence of heavy-tailed noise and $(L,s)$-H\"older smoothness.}


\paragraph{\bf Organization}
The remainder of this report is organized as follows: 
Section \ref{sec:summary} of this chapter summarizes the three  main convergence results presented in our work. 
Section \ref{chap:prelim} establishes the mathematical preliminaries, detailing the formal definitions and core assumptions that we use regarding the noise model and the objective landscape. 
Section \ref{chap:sgd} presents the convergence proof for the standard SGD operating in the regime ($\alpha \ge 1+s$). 
Section \ref{chap:dgclip} presents the convergence analysis of the $\delta$-GClip algorithm, for $s \in (0,1]$, $\alpha \ge 1+s$. 
Section \ref{app:gclip} presents the convergence analysis of the G-Clip algorithm, giving a unified guarantee across both
$1+s \le \alpha$ and $1+s > \alpha$ which in the latter regime yields a first-of-its-kind convergence guarantee. 
Section \ref{chap:conclusion} concludes the report and discusses promising avenues for future research.

Appendix \ref{sec:lit_holder} provides a comprehensive literature review of existing convergence guarantees for stochastic gradient methods under H\"older smoothness and heavy-tailed noise separately.
Appendix \ref{subsec:sgdproof1} gives the proofs deferred from Section \ref{chap:sgd}.
Appendix \ref{app:lemma-delta} gives the detailed proofs of the lemmas required for the main theorem in Section \ref{chap:dgclip}.
Appendix \ref{subsec:gclip-proof} gives the proofs deferred from Section \ref{app:gclip}, covering both the $\alpha \ge 1+s$ and $\alpha < 1+s$ regimes for G-Clip.
Appendix \ref{app:lemma-gproof} gives the detailed proofs of the auxiliary lemmas required in Appendix \ref{subsec:gclip-proof}.
Appendix \ref{subsec:experiments} reports an empirical validation of the two central claims of this work: that the convergence rate of SGD is independent of the heavy-tail index $\alpha$ once $\alpha \ge 1+s$, and that G-Clip is the only one of the analysed methods with a provable guarantee once $\alpha < 1+s$.

\section{Summary of Main Results}\label{sec:summary}
\label{subsec:main-results}
This work establishes convergence guarantees for stochastic gradient methods 
under the simultaneous presence of heavy-tailed gradient noise and $(L,s)$-H\"{o}lder 
smooth objective functions. Next we present shortened, informal versions of our two main results. 


\medskip 
\begin{theorem}[{\bf Informal Statement of Theorem ~\ref{thm:SGD-proof}: SGD Under $(L,s)$-H\"older Smoothness and Heavy-Tailed Gradients}]
Consider standard SGD applied to an $(L,s)$-H\"older smooth, possibly nonconvex objective $f$, using the update $\vx_{k+1} = \vx_k - \eta \vg_k$, where $\vg_k = \nabla f(\vx_k) + \vxi_k$ is a stochastic gradient with zero-mean noise $\vxi_k$ satisfying $\mathbb{E}[\|\vxi_k\|^{\alpha} \mid \vx_k] \le \sigma^{\alpha}$ for some $\sigma > 0$  i.e., the noise has a finite $\alpha$-th moment. Assuming the heavy-tail index $\alpha$ of the noise satisfies $\alpha \ge 1+s$, this moment condition holds and a classical descent analysis, without clipping or robustification, applies directly. Running SGD for $T$ iterations with the optimal constant step size $\eta = O(T^{-1/(1+s)})$ yields the convergence rate, 
\begin{equation}
\frac{1}{T}\sum_{k=1}^{T}\mathbb{E}\|\nabla f(\vx_k)\|^2 = O\!\left(T^{-\frac{s}{1+s}}\right).
\end{equation}
\end{theorem}
When $s=1$ and $\alpha=2$, this recovers the classical $O(T^{-1/2})$ nonconvex SGD rate of \citep{ghadimi2013} under Lipschitz smoothness and finite-variance noise; more generally, the rate degrades continuously as $s \to 0$, reflecting the reduced regularity of the objective.

\medskip 
\begin{theorem}[{\bf Informal Statement of Theorem ~\ref{thm:delta-gclip-s-3}: $\delta$-GClip Under $(L,s)$-H\"older Smoothness and Heavy-Tailed Gradients}]\label{thm:dgclip-informal}
For parameters $\gamma > 0$ and $\delta \in (0,1)$, the $\delta$-clipped stochastic gradient estimator is $\hat{\vg}_\delta(\vx_k) = \min\{1, \max\{\delta, \gamma/\|\vg_k\|\}\}\,\vg_k$, where $\vg_k = \nabla f(\vx_k) + \vxi_k$ satisfies the same $\alpha$-th moment bound $\mathbb{E}[\|\vxi_k\|^{\alpha} \mid \vx_k] \le \sigma^{\alpha}$, and $\delta$-GClip updates via $\vx_{k+1} = \vx_k - \eta \hat{\vg}_\delta(\vx_k)$. Assuming that the heavy-tail index satisfies $\alpha \ge 1+s$, running $\delta$-GClip for $T$ iterations with step size $\eta = O(T^{-1/(1+s)})$  yields the convergence rate,
\begin{equation}
\frac{1}{T}\sum_{k=1}^{T}\mathbb{E}\left[\min\{\|\nabla f(\vx_k)\|^2, \gamma\|\nabla f(\vx_k)\|\}\right] = O\!\left(T^{-\frac{2s(\alpha-1)}{(1+s)(2\alpha-1)}}\right).
\end{equation}
\end{theorem}

In the small-gradient regime $\|\nabla f(\vx_k)\| \le \gamma$, the LHS of above reduces to the standard stationarity measure.

\medskip
\begin{theorem}[{\bf Informal Statement of Theorem~\ref{thm:gclip}: G-Clip Under $(L,s)$-H\"older Smoothness and Heavy-Tailed Gradients}]
For a threshold $\gamma>0$, the standard clipped stochastic gradient estimator (the $\delta=0$ case
of $\delta$-GClip) is $\hat{\vg}(\vx_k)=\min\{1,\ \gamma/\|\vg_k\|\}\,\vg_k$, where $\vg_k=\nabla f(\vx_k)+\vxi_k$
satisfies the same $\alpha$-th moment bound $\mathbb{E}[\|\vxi_k\|^{\alpha}\mid \vx_k]\le\sigma^{\alpha}$, and G-Clip
updates via $\vx_{k+1}=\vx_k-\eta\,\hat{\vg}(\vx_k)$. Then running G-Clip for $T$ iterations with a suitably chosen constant step
size $\eta$ and clipping threshold $\gamma$ yields
\begin{equation}
\frac{1}{T}\sum_{k=1}^{T}\mathbb{E}\!\left[\min\{\|\nabla f(\vx_k)\|^2,\ \gamma\|\nabla f(\vx_k)\|\}\right]
=
\begin{cases}
O\!\left(T^{-\frac{2s(\alpha-1)}{(1+s)(2\alpha-1)}}\right), & 1+s\le\alpha,\\[8pt]
O\!\left(T^{-\frac{2s(\alpha-1)}{(\alpha-1)+s(2\alpha-1)}}\right), & 1+s\ge\alpha.
\end{cases}
\end{equation}
\end{theorem}

Thus in the regime $1+s\le\alpha$ the rate coincides with that of $\delta$-GClip (Theorem~\ref{thm:dgclip-informal}), whereas {\em in
the regime $1+s\ge\alpha$ yields a convergence guarantee in the very heavy-tailed setting
$\alpha<1+s$, which lies outside the scope of all preceding and existing results.} As in Theorem~\ref{thm:dgclip-informal}, in the
small-gradient regime $\|\nabla f(\vx_k)\|\le\gamma$ the left-hand side reduces to the standard
stationarity measure.

An empirical study validating the salient aspects of the above results has been given in Appendix~\ref{subsec:experiments}. In particular the experiments are suggestive of there being a regime where $\delta$-GClip converges for $\alpha < 1+s$ but which is not yet visible in theory.

\subsection{Comparison with Existing Results} The behaviour of stochastic gradient methods depends qualitatively on the
relationship between the H\"{o}lder exponent~$s$ and the heavy-tail
index~$\alpha$.
We organise the discussion into four parameter regimes and examine, in each
regime, the convergence rates of the classical SGD~\cite{ghadimi2013},
standard gradient clipping (GClip)~\cite{suvrit1912},
our SGD result (Theorem~\ref{thm:SGD-proof}),
our $\delta$-GClip result (Theorem~\ref{thm:delta-gclip-s-3}),
and our G-Clip result (Theorem~\ref{thm:gclip}).
The last of these analyses the same standard clipping algorithm as~\cite{suvrit1912}
(the $\delta=0$ case of $\delta$-GClip), but under $(L,s)$-H\"older smoothness. 


\medskip
\qquad \noindent\textbf{Regime 1: $s=1,\;\alpha=2$
(Lipschitz Smooth, Finite Variance).}
This is the classical setting where the foundational result of~\cite{ghadimi2013} established that SGD achieves
$O(T^{-1/2})$.
GClip~\cite{suvrit1912} recovers the same rate at $\alpha=2$.

Theorem~\ref{thm:SGD-proof} of the present work also recovers the above raate as a special case.
The $\delta$-GClip algorithm (Theorem~\ref{thm:delta-gclip-s-3}) yields
$O(T^{-1/3})$,
which is strictly slower than the $O(T^{-1/2})$ rate of both SGD and GClip;
the slowdown from exponent $\tfrac{1}{2}$ to $\tfrac{1}{3}$ reflects the
additional bias of the $\delta$-regularisation mechanism, a cost
paid for the provable deep-network training guarantee of~\cite{tucat2025regularized}.
Our G-Clip result (Theorem~\ref{thm:gclip}) recovers the optimal $O(T^{-1/2})$ here, through its
$1+s\ge\alpha$ branch evaluated at $s=1,\alpha=2$. 


\medskip
\qquad \noindent\textbf{Regime 2: $s=1,\;\alpha\in(1,2)$
(Lipschitz Smooth, Heavy-Tailed Noise).}
Both Theorem~\ref{thm:SGD-proof} and Theorem~\ref{thm:delta-gclip-s-3}
share the condition $\alpha \ge 1{+}s$ and hence at $s=1$, neither provides a convergence
guarantee in this regime.
In this classical heavy-tailed regime, standard GClip~\citep{suvrit1912} achieves a provable expected
stationarity rate of $\mathcal{O}(T^{-\frac{2(\alpha-1)}{3\alpha-2}})$, which they proved optimal.
Our G-Clip result (Theorem~\ref{thm:gclip}) does cover this regime recovering as a special case the known optimal GClip rate. Thus,
among the methods analysed in this work, G-Clip is the only one that provides a guarantee in the
classical heavy-tailed Lipschitz regime, and it does so at the optimal rate.

Furthermore, recent advancements have established that gradient normalization, when coupled with
momentum, can also achieve this optimal rate. Specifically, \citet{SunLiuYuan2025GNClip} analyze Normalized
SGD with Clipping (NSGDC), and \citet{liu2025nonconvex} analyze Batched Normalized SGD with Momentum
(Batched NSGDM), and both attain the optimal rate of GClip.
The critical algorithmic difference is that both NSGDC and Batched NSGDM fundamentally rely on
maintaining a running exponential moving average of past gradients (momentum) and strictly
normalizing the update step. Understanding how momentum and gradient normalization interact with
fractional $s$-H\"{o}lder landscapes under heavy-tailed noise remains a highly relevant open question.

\medskip
\qquad \noindent\textbf{Regime 3: $s\in(0,1),\;\alpha\ge1+s$
(H\"{o}lder Smooth, Moderately Heavy-Tailed Noise).}
This regime is a primary novel contribution of the present work.
GClip~\cite{suvrit1912} assumes Lipschitz smoothness ($s=1$) and provides no result here.
Theorems~\ref{thm:SGD-proof}, \ref{thm:delta-gclip-s-3}, and~\ref{thm:gclip} all provide convergence guarantees in this setting.
SGD (Theorem~\ref{thm:SGD-proof}) achieves $O(T^{-\frac{s}{1+s}})$, while both $\delta$-GClip (Theorem~\ref{thm:delta-gclip-s-3})
and G-Clip (Theorem~\ref{thm:gclip}, $1+s\le\alpha$ branch) achieve $O(T^{-\frac{2s(\alpha-1)}{(1+s)(2\alpha-1)}})$;
that is, the two clipping methods coincide in rate here.

Comparing the exponents we always have
$\frac{2s(\alpha-1)}{(1+s)(2\alpha-1)} < \frac{s}{1+s}$, \emph{so SGD is strictly faster than either
clipping method throughout this regime and remains the best available method in this regime.}

Our Theorem \ref{thm:SGD-proof} can also be seen as extending the convergence result for SGD on Holder continous function as given in \cite{wang2023holder} to the heavy-tailed regime of $1+s\le\alpha\le2$.

\medskip
\qquad \noindent\textbf{Regime 4: $s\in(0,1),\;\alpha < 1+s$
(H\"{o}lder Smooth, Very Heavy-Tailed Noise).}
When $\alpha<1{+}s$, the shared condition $\alpha \ge 1{+}s$
(Assumption~\ref{ass:exponent-s}) fails, so neither our SGD result (Theorem~\ref{thm:SGD-proof}) nor our
$\delta$-GClip result (Theorem~\ref{thm:delta-gclip-s-3}) applies, and neither does the GClip analysis
of~\cite{suvrit1912}.

\emph{Our G-Clip result (Theorem~\ref{thm:gclip}), separately handles this regime $2 > 1+s>\alpha > 0$, yielding a rate of
$O(T^{-\frac{2s(\alpha-1)}{(\alpha-1)+s(2\alpha-1)}})$, giving the first convergence guarantee in this
combination which combines heavy-tailed noise and H\"older smoothness}, a setting previously left entirely open. Whether this rate is optimal remains an open problem.

\bigskip
\begin{table}[ht]
\centering
\small
\setlength{\tabcolsep}{7pt}
\renewcommand{\arraystretch}{2.0}
\begin{tabular}{lccc}
\toprule
\textbf{Regime $(s, \alpha)$}
  & \textbf{SGD (Thm.~\ref{thm:SGD-proof})}
  & \textbf{GClip}
  & \textbf{$\delta$-GClip (Thm.~\ref{thm:delta-gclip-s-3})} \\
\midrule
$s=1,\;\alpha=2$
  & $O(T^{-\frac{1}{2}})$
  & $O(T^{-\frac{1}{2}})$ (\cite{suvrit1912},Thm \ref{thm:gclip})
  & {\color{firebrick}$O(T^{-\frac{1}{3}})$}
\\
$s=1,\;\alpha\in(1,2)$
  & ---
  & $O(T^{-\frac{2(\alpha-1)}{3\alpha-2}})$ (\cite{suvrit1912},Thm \ref{thm:gclip})
  & ---
\\
$s\in(0,1),\;\alpha\ge1{+}s$
  & {\color{firebrick}$O(T^{-\frac{s}{1+s}})$}
  & {\color{firebrick}$O(T^{-\frac{2s(\alpha-1)}{(1+s)(2\alpha-1)}})$} (Thm \ref{thm:gclip})
  & {\color{firebrick}$O(T^{-\frac{2s(\alpha-1)}{(1+s)(2\alpha-1)}})$}
\\
$s\in(0,1),\;\alpha < 1{+}s$
  & ---
  & {\color{firebrick}$O(T^{-\frac{2s(\alpha-1)}{(\alpha-1)+s(2\alpha-1)}})$} (Thm \ref{thm:gclip})
  & ---
\\
\bottomrule
\end{tabular}
\caption{Comparison of expected stationarity convergence rates across four
parameter regimes defined by the H\"{o}lder exponent~$s$ and heavy-tail
index~$\alpha$. All the rates mentioned in red entries are new to this work. A dash~(---) indicates that no provable guarantee
is available for that algorithm in that regime.}
\label{tab:comparison}
\end{table}

A detailed survey of related works can be found in Appendix ~\ref{sec:lit_holder}



\section{Mathematical Preliminaries}
\label{chap:prelim}
\begin{definition}[Hölder Continuous Gradient]\label{def:holder-descent}
An at least once differentiable $f : \R^d \rightarrow \R$ is said to be $(L,s)$-Hölder continuous for some $L>0$ and 
$s\in(0,1]$, if $\forall ~\vx,\vy\in\mathbb{R}^d$,
\[
\|\nabla f(\vx) - \nabla f(\vy)\|
\;\le\;
L\,\|\vx-\vy\|^{\,s}.
\]
Equivalently, $f$ satisfies the upper bound,
\[
f(\vy) \;\le\; f(\vx)
+ \langle \nabla f(\vx),\, \vy-\vx \rangle
+ \frac{L}{1+s}\,\|\vy-\vx\|^{\,1+s}.
\]
\end{definition}




\smallskip 
\begin{definition}[$\delta$-GClip Gradient Estimator of \cite{tucat2025regularized}]
For parameters $\gamma>0$ and $\delta \in (0,1)$, the $\delta$-clipped stochastic gradient estimator is defined as,
\[
\hat g_{\delta}(\vx)
=
\min\left\{1, \max\left\{\delta, \frac{\gamma}{\|g(\vx)\|}\right\}\right\} g(\vx).
\]
\end{definition}

\begin{definition}[$\delta$-GClip Iteration of \cite{tucat2025regularized}]
Given a step size $\eta_k > 0$, the $\delta$-GClip update is
\[
\vx_{k+1} = \vx_k - \eta_k \hat g_{\delta}(\vx_k).
\]
\end{definition}

We recall our motivations to focus on the method of \cite{tucat2025regularized}, as that is the only known adapative step-length scheme for gradient based algorithms, which has the dual property, that in its full-batch settings it is a provable trainer of deep-nets and which in its stochastic set-up also competes heuristics like Adam. 

\medskip 
\begin{definition}[G-Clip Gradient Estimator]\label{def:gclip-est}
For a threshold $\gamma>0$, the clipped stochastic gradient estimator is
$\hat{\vg}(\vx)=\min\{1,\ \gamma/\|\vg(\vx)\|\}\,\vg(\vx)$; in particular $\|\hat{\vg}(\vx)\|\le\gamma$.
\end{definition}

\medskip 
\begin{definition}[G-Clip Iteration]\label{def:gclip-iter}
Given a step size $\eta_k>0$, the G-Clip update is $\vx_{k+1}=\vx_k-\eta_k\,\hat{\vg}(\vx_k)$.
\end{definition}

\medskip
\begin{assumption}[Bounded $\alpha$-Moment]\label{ass:noise-s}
The stochastic gradient $\vg$ of the at least once differentiable function $f$ is said to satisfy a finite $\alpha$-moment condition for some $\alpha \in (1,2]$ if for some $\sigma >0$,
\[
\mathbb{E}\|g(\vx) - \nabla f(\vx)\|^{\alpha} \leq \sigma^{\alpha},
\qquad \forall \vx \in \mathbb{R}^d.
\]
\end{assumption}

When $\alpha < 2$, the above corresponds to ``heavy-tailed noise" in the stochastic gradient that may have infinite variance, while satisfying the above at $\alpha = 2$ recovers the classical finite-variance setting.

\medskip
\begin{assumption}[Lowerbound on the Noise Moment Exponent]\label{ass:exponent-s}
Corresponding to $s$ defined for a function  $f$, as given in Definition \ref{def:holder-descent}, the noise exponent \(\alpha\) of the corresponding stochastic gradient in Assumption~\(\ref{ass:noise-s}\) is also assumed to satisfy,
\[
\alpha \;\ge\; 1+s,
\]
\end{assumption}

\begin{assumption}[Unbiased Gradient Oracle]
\label{ass:unbiased}
At each iterate $\vx \in \mathbb{R}^d$, the stochastic gradient oracle corresponding to an at least once differentiable function $f$ returns a random vector $\vg(\vx)$ satisfying,
\[
\mathbb{E}[\vg(\vx)\mid \vx] = \nabla f(\vx).
\]
Equivalently, if the stochastic gradient noise is defined by
$\vxi(\vx) := \vg(\vx)-\nabla f(\vx)$, then
$\mathbb{E}[\vxi(x)\mid \vx] = 0.$
In particular when $\vx_1,\dots, \vx_t$ are t consecutive iterates of the stochastic optimization algorithm then $\vxi(\vx_t)$ is assumed to be independent of the filtration generated by the previous iterates.
\end{assumption}

\medskip
\begin{assumption}[Objective Function Bounded Below]
\label{ass:lower}
The objective function $f:\mathbb{R}^d \to \mathbb{R}$ is bounded below; that is, there exists a constant $f^\star > -\infty$
such that
\[
f(\vx) \geq f^\star,
\qquad \forall \vx \in \mathbb{R}^d.
\]
This assumption prevents the objective from decreasing indefinitely and guarantees that the descent inequalities obtained throughout the analysis remain meaningful.
\end{assumption}




\section{Convergence of Standard SGD under s-Hölder and Heavy-Tailed Noise}
\label{chap:sgd}
In this section we establish convergence of standard SGD under $(L,s)$-H\"{o}lder
smoothness and heavy-tailed noise. 
\begin{theorem}[Convergence of SGD for $\alpha \ge 1+s$]\label{thm:SGD-proof}
Under Assumptions \ref{ass:noise-s}-\ref{ass:lower}, if we run standard SGD with the update rule $\vx_{k+1} = \vx_k - \eta \vg_k$ for $T$ iterations using the constant step size $\eta = \mathcal{O}(T^{-\frac{1}{1+s}})$, the expected squared gradient norm converges at the rate:
\begin{equation*}
\frac{1}{T} \sum_{k=1}^T \mathbb{E}[\|\nabla f(\vx_k)\|^2] = \mathcal{O}\left( T^{-\frac{s}{1+s}} \right)
\end{equation*}
\end{theorem}

The proof of the above theorem is given in Appendix \ref{subsec:sgdproof1}.

\begin{remark}\label{rem:sgd-s1}
The proof of Theorem~\ref{thm:SGD-proof} in an intermediate step applies Young's inequality with conjugate exponent
$q=\frac{2}{1-s}$, which is finite only for $s\in(0,1)$ and hence the endpoint $s=1$ (the Lipschitz-smooth
case) is handled directly. When $s=1$, the fractional term in the descent inequality equals
$\frac{L2^s\eta^{1+s}}{1+s}\|\nabla f(\vx_k)\|^{1+s}=L\eta^2\|\nabla f(\vx_k)\|^2$, which can be absorbed
into the descent term $-\eta\|\nabla f(\vx_k)\|^2$ whenever $\eta\le\frac{1}{2L}$ (so that
$L\eta^2\le\frac{\eta}{2}$), giving
\begin{equation*}
\mathbb{E}[f(\vx_{k+1})\mid\vx_k]\le f(\vx_k)-\frac{\eta}{2}\|\nabla f(\vx_k)\|^2+L\sigma^2\eta^2 .
\end{equation*}
Telescoping exactly as in the proof and choosing $\eta=\mathcal{O}(T^{-1/2})$ (which satisfies
$\eta\le\frac{1}{2L}$ for all large $T$) yields
$\frac{1}{T}\sum_{k=1}^T\mathbb{E}[\|\nabla f(\vx_k)\|^2]=\mathcal{O}(T^{-1/2})$, the classical
nonconvex SGD rate, in agreement with $\mathcal{O}(T^{-\frac{s}{1+s}})$ evaluated at $s=1$.
\end{remark}


\section{Convergence of \texorpdfstring{$\delta$}{delta}-GClip under s-Hölder and Heavy-Tailed Noise}
\label{chap:dgclip}
\begin{theorem}[Non-convex convergence of $\delta$-GClip under $(L,s)$-H\"{o}lder 
smoothness]
\label{thm:delta-gclip-s-3}
 
Suppose $f : \mathbb{R}^d \to \mathbb{R}$ is bounded below by $f_\star$ and has 
$(L,s)$-H\"{o}lder continuous gradients for some $s \in (0,1]$. Let the stochastic 
gradients satisfy Assumptions~\ref{ass:noise-s}, ~\ref{ass:exponent-s} and ~\ref{ass:unbiased} with 
heavy-tail index $\alpha \in (1,2]$ and noise scale parameter $\sigma$. Let 
$\{\mathbf{x}_k\}_{k=0}^{T-1}$ be the iterates generated by the $\delta$-GClip 
algorithm with parameter $\delta \in (0, 1]$. If the algorithm is run for $T$ 
iterations with the constant clipping threshold $\gamma$ and learning rate $\eta$ specified in Remark~\ref{rem:theorem-constants}, then for every $T \ge 1$ the iterates satisfy:
\begin{equation*}\label{eq:main-rate}
\frac{1}{T}\sum_{k=1}^{T} \mathbb{E}\!\left[ \min\!\left\{ \|\nabla f(\mathbf{x}_k)
\|^2,\; \gamma\|\nabla f(\mathbf{x}_k)\| \right\} \right] = \mathcal{O}\!\left( 
T^{-\frac{2s(\alpha-1)}{(1+s)(2\alpha-1)}} \right).
\end{equation*}
\end{theorem}
 
\begin{remark}[Explicit Constants]
\label{rem:theorem-constants}
The explicit forms of the constants used in the theorem cab be made explicit as follows. Fix any 
``gradient threshold parameter'' $\theta \in (0,1)$. Correspondingly the auxiliary algorithmic constants 
are $M = \min\!\left\{\tfrac{15\theta}{16},\, \tfrac{1}{3}\right\}$ and 
$c = \min\!\left\{\tfrac{1}{2},\, \tfrac{M}{16},\, \tfrac{1}{4}\right\}$. The variance 
bounding constant $K$, which is independent of $\gamma$ and $T$, is defined such that $2^s\theta^{1+s} + 2^s\sigma^{1+s}\gamma^{-(1+s)} + C_{\mathrm{tail}}(\theta, 
\delta)\sigma^\alpha\gamma^{-\alpha} + 2^s\delta^{1+s}\sigma^{1+s}\gamma^{-(1+s)} 
\leq K$.
The macro-constants for the optimal schedules are then defined as 
$C_1 = \frac{F_0}{c}$, $C_2 = \frac{C_{\mathrm{bias}}(\theta, \delta)\,\sigma^{2\alpha}}
{2\,c}$, $C_3 = \frac{L K}{c(1+s)}$, and $C_4 = 2\!\left(C_1^s C_3\right)^{\!
\frac{1}{1+s}}$, where $F_0 = f(\mathbf{x}_0) - f_\star$. 
 
\medskip
With $K_v:=\theta^{-(1+s)}+2^s\delta^{1+s}$, the clipping threshold and learning rate are set, for each
$T\ge1$, as
\[
\gamma=\max\Big\{\,1,\ \tfrac{16\cdot2^{1/\alpha}}{\theta}\,\sigma,\ \big(\tfrac{4\cdot16^{\alpha-1}}{M\theta^{\alpha-1}}\big)^{1/\alpha}\sigma,\ \tfrac{64}{3M\theta}\,\sigma,\ (C_2/C_4)^{\frac{1}{2\alpha-1}}T^{\frac{s}{(1+s)(2\alpha-1)}}\Big\},
\]
\[
\eta=\min\Big\{\,\tfrac{\delta}{2}\big(\tfrac{M(1+s)}{16LK_v}\big)^{1/s}\gamma^{\frac{1-s}{s}},\ \ \delta\big(\tfrac{(1+s)2^{1-s}}{8LK_v}\big)^{1/s}\gamma^{\frac{1-s}{s}},\ \ (C_1/C_3)^{\frac{1}{1+s}}\gamma^{-1}T^{-\frac{1}{1+s}}\Big\}.
\]
Taking $\gamma$ as a maximum and $\eta$ as a minimum makes both hold at every $T\ge1$; the final entry
of each dominates for large $T$, where $(\eta,\gamma)$ reduce to the schedule that gives the displayed
bound.

\end{remark}

\subsection{Auxiliary Lemmas}
\begin{lemma}[Conditional Probability Bounds for Tail Regions]
\label{lem:prob_bounds-3}
Let Assumption~\ref{ass:noise-s} and ~\ref{ass:exponent-s} hold, meaning the noise satisfies the bounded moment condition $\E[\|\vxi\|^\alpha] \le \sigma^\alpha$ for $\alpha \in (1, 2]$. 
Suppose the true gradient satisfies the condition $\|\nabla f(\vx)\| \le \theta\gamma$ for a fixed constant $\theta \in (0, 1)$. 
Define the stochastic gradient $\vg = \nabla f + \vxi$ and the tail regions $B=\{\gamma<\|\vg\|\le \gamma/\delta\}$ and $C=\{\|\vg\|>\gamma/\delta\}$ for $\gamma > 0$ and $\delta \in (0, 1]$. 
Then the conditional probabilities of the stochastic gradient falling into these regions, given $\nabla f(\vx)$, satisfy
\begin{align}
\Pr\!\left(B \mid \|\nabla f(\vx)\| \le \theta\gamma \right)
&\le
\frac{\sigma^\alpha}{(1-\theta)^\alpha \gamma^\alpha},
\label{eq:prob_B_lem}
\\
\Pr\!\left(C \mid \|\nabla f(\vx)\| \le \theta\gamma\right)
&\le
\left(\frac{\delta}{1-\theta\delta}\right)^\alpha
\frac{\sigma^\alpha}{\gamma^\alpha}.
\label{eq:prob_C_lem}
\end{align}
\end{lemma}

The proof is provided in \autoref{subsec:probability-bounds-3}.

\begin{lemma}[Conditional Bias and Variance Bounds for $\|\nabla f(\vx)\| \le \theta\gamma$]
\label{lem:clipped-bias-delta-s-3}
Let Assumption~\ref{ass:noise-s} and ~\ref{ass:exponent-s} hold, 
Consider the $\delta$-GClip estimator $\hat{\vg}_\delta(\vx)$ with clipping threshold $\gamma > 0$ and parameter $\delta \in [0,1]$, and suppose the true gradient satisfies $\|\nabla f(\vx)\| \le \theta \gamma$ for some fixed $\theta \in (0,1)$ (Case 1). Define the conditional bias $B_\delta(\vx) := \E[\hat{\vg}_\delta(\vx) \mid \nabla f(\vx)] - \nabla f(\vx)$. Then, for any $s \in (0,\alpha-1]$, the following conditional bounds hold:
\begin{enumerate}
    \item \textbf{Squared Bias Bound:} The squared norm of the estimator bias, $B_\delta = \mathbb{E}[\hat{\vg}_\delta \mid \vx] - \nabla f(\vx)$, is strictly bounded by:
    \begin{equation}\label{eq:lemma1_bias}
        \|B_\delta(\vx)\|^2 \le \underbrace{\left( \frac{2}{(1-\theta)^{2\alpha}} + 2 (1-\delta)^2 \left[ \theta\left(\frac{\delta}{1-\theta\delta}\right)^{\!\alpha} + \left(\frac{\delta}{1-\theta\delta}\right)^{\!\alpha-1} \right]^2 \right)}_{C_{bias}(\theta, \delta)} \sigma^{2\alpha} \gamma^{2(1-\alpha)}.
    \end{equation}
    
    \item \textbf{Moment Bound:} The $(1+s)$-th moment of the estimator is bounded by:
    \begin{multline}\label{eq:lemma1_variance}
      \mathbb{E}\!\left[\|\hat{\mathbf{g}}_\delta(\mathbf{x})\|^{1+s}
        \mid \|\nabla f(\mathbf{x})\| \le \theta\gamma\right]
      \le 2^s\|\nabla f(\mathbf{x})\|^{1+s} + 2^s\sigma^{1+s} \\
      + \underbrace{\left(
          \frac{1}{(1-\theta)^\alpha}
          + 2^s\delta^{1+s}\!\left[
              \theta^{1+s}\!\left(\frac{\delta}{1-\theta\delta}\right)^{\!\alpha}
              +\left(\frac{\delta}{1-\theta\delta}\right)^{\!\alpha-(1+s)}
            \right]
        \right)}_{C_{\mathrm{tail}}(\theta,\delta)}
      \sigma^\alpha\gamma^{1+s-\alpha}.
    \end{multline}
\end{enumerate}
\end{lemma}

The proof is provided in \autoref{subsec:case1-bounds-3}.


\begin{lemma}[Inner-product lower bound for $\delta$-GClip for $\|\nabla f(\vx)\| > \theta\gamma$]
\label{lem:inner-gclip-case2-3}
Let Assumption~\ref{ass:noise-s} and ~\ref{ass:exponent-s} hold, 
Suppose the true gradient satisfies the (Case 2) condition $\|\nabla f(\vx)\| > \theta\gamma$ for a fixed constant $\theta \in (0, 1)$. 
Define the constant
$M = \min\big\{\frac{15}{16}\theta, \frac{1}{3}\big\}$. If the clipping threshold $\gamma$ satisfies the following conditions:
\[
\gamma \;\ge\; \max\!\left\{ 
    \frac{16 \cdot 2^{1/\alpha}}{\theta},\; 
    \left(\frac{4 \cdot 16^{\alpha-1}}{M \theta^{\alpha-1}}\right)^{\!1/\alpha},\;
    \frac{64}{3 M \theta}
\right\} \sigma,
\]
then the expected inner product guarantees a strictly positive linear descent direction:
\[
\E\!\big[\langle \nabla f(\vx), \hat{\vg}_{\delta}(\vx)\rangle \mid \vx\big] \;\ge\; \frac{1}{8} M \,\gamma\,\|\nabla f(\vx)\|.
\]
The proof is provided in \autoref{subsec:ip-bound-case2}.
\end{lemma}


\begin{lemma}[Moment bound for $\delta$-GClip for $\|\nabla f(\vx)\| > \theta\gamma$]
\label{lem:moment-gclip-case2-3}
Let Assumption~\ref{ass:noise-s} and ~\ref{ass:exponent-s} hold, 
Suppose the true gradient satisfies the (Case 2) condition $\|\nabla f(\vx)\| > \theta\gamma$ for a fixed constant $\theta \in (0, 1)$. 
Then, for any $s \in (0, \alpha-1]$, the $(1+s)$-th moment of the estimator satisfies the following upper bound:
\[
\E\big[\|\hat{\vg}_\delta(\vx)\|^{1+s} \mid \vx\big] \le \left( \frac{1}{\theta^{1+s}} + 2^s \delta^{1+s} \right) \|\nabla f(\vx)\|^{1+s} + 2^s \delta^{1+s} \sigma^{1+s}.
\]
\end{lemma}

The proof is provided in \autoref{subsec:var-bound-case2}.

\begin{lemma}[Inner Product Bound {\normalfont (Cutkosky and Mehta, 2020)}] \label{lem:cutkoskyandmehta-ss}
For any $\vv \in \mathbb{R}^d$ and point $\vx$:
\[
\left\langle \frac{\vv}{\|\vv\|}, \nabla f(\vx) \right\rangle 
\ge \tfrac{1}{3}\|\nabla f(\vx)\| - \tfrac{8}{3}\|\vv - \nabla f(\vx)\|.
\]
This result is adapted from Lemma 2 of~\cite{cutkosky2020momentum}.
\end{lemma}

\subsection{Proof of Theorem \ref{thm:delta-gclip-s-3}}
\begin{proof}[Proof of Theorem \ref{thm:delta-gclip-s-3}]
~\\
\medskip\noindent\textbf{Descent under $s$-H\"older smoothness.}
Since $\nabla f$ is $(L, s)$-H\"older continuous, meaning that for all $\vx, \vy \in \R^d$,
\[
\|\nabla f(\vx)-\nabla f(\vy)\| \;\le\; L\|\vx-\vy\|^{s}, \qquad s\in(0,1],
\]
the generalized descent lemma implies that for any $\vx, \vy \in \R^d$,
\begin{equation} \label{eq:sholder-descent-lemma}
f(\vy) \;\le\; f(\vx)
+ \langle\nabla f(\vx), \vy-\vx\rangle
+ \frac{L}{1+s}\,\|\vy-\vx\|^{\,1+s}.
\end{equation}

We apply \eqref{eq:sholder-descent-lemma} to the update
\[
\vx_k \;=\; \vx_{k-1} - \eta_{k-1}\,\hat{\vg}_{\delta,k-1}(\vx_{k-1}),
\]
where $\hat{\vg}_{\delta,k-1}(\vx_{k-1})$ is the $\delta$-clipped stochastic gradient at $\vx_{k-1}$.
For a fixed iterate $\vx_{k-1}$ and a specific realization of $\hat{\vg}_{\delta,k-1}$,
\[
f(\vx_k)
\;\le\; f(\vx_{k-1})
- \eta_{k-1}\langle\nabla f(\vx_{k-1}), \hat{\vg}_{\delta,k-1}(\vx_{k-1})\rangle
+ \frac{L}{1+s}\,\eta_{k-1}^{\,1+s}\,
\|\hat{\vg}_{\delta,k-1}(\vx_{k-1})\|^{\,1+s}.
\]

Taking the conditional expectation with respect to all randomness in $\hat{\vg}_{\delta,k-1}$ yields
\begin{equation} \label{eq:main-descent-sholder-start-3}
\begin{aligned}
\E\!\left[f(\vx_k)\mid\vx_{k-1}\right]
&\le\;
f(\vx_{k-1})
- \eta_{k-1}
\Big\langle\nabla f(\vx_{k-1}),
\E[\hat{\vg}_{\delta,k-1}(\vx_{k-1})\mid\vx_{k-1}]\Big\rangle \\
&\qquad
+ \frac{L}{1+s}\,\eta_{k-1}^{\,1+s}\,
\E[\|\hat{\vg}_{\delta,k-1}(\vx_{k-1})\|^{\,1+s}\mid\vx_{k-1}].
\end{aligned}
\end{equation}
To obtain a descent inequality from \eqref{eq:main-descent-sholder-start-3}, we must bound the two stochastic contributions: the expected inner product (descent signal) and the expected $(1+s)$-th moment (variance penalty). 

These quantities behave differently depending on the size of the true gradient relative to the clipping threshold $\gamma$ and the scaling parameter $\delta$. Accordingly, we partition the state space into three disjoint regimes:
\[
\resizebox{\linewidth}{!}{$\displaystyle
  \textbf{Case~1:}\ \|\nabla f(\mathbf{x}_{k-1})\| \le \theta\gamma,\quad
  \textbf{Case~2a:}\ \theta\gamma < \|\nabla f(\mathbf{x}_{k-1})\| \le 2\gamma/\delta,\quad
  \textbf{Case~2b:}\ \|\nabla f(\mathbf{x}_{k-1})\| > 2\gamma/\delta.
$}
\]
In all cases, we bound the two terms separately by decomposing the unclipped stochastic gradient $\vg_{k-1}$ into three distinct geometric regions $A$, $B$, and $C$, defined as,
\begin{equation}\label{defnABC-s}
A := \{\|\vg_{k-1}\|\le \gamma\},\qquad
B := \{\gamma<\|\vg_{k-1}\|\le \gamma/\delta\},\qquad
C := \{\|\vg_{k-1}\|>\gamma/\delta\}.
\end{equation}
We now systematically bound the expected progress in each of the three cases. For brevity in the following derivations, we drop the iteration index $k-1$ and denote the current iterate as $\vx = \vx_{k-1}$ and the learning rate as $\eta = \eta_{k-1}$.

\medskip\noindent\textbf{Case 1: $\|\nabla f(\vx_{k-1})\| \le \theta\gamma$.}
We define the estimator bias as
\[
\vb_{k-1} := \E[\hat{\vg}_{\delta,k-1}(\vx_{k-1})\mid\vx_{k-1}] - \nabla f(\vx_{k-1}).
\]
Then from \eqref{eq:main-descent-sholder-start-3} we have
\begin{align}
\E[f(\vx_k)\mid\vx_{k-1}]
&\le f(\vx_{k-1})
- \eta_{k-1}
\Big\langle \nabla f(\vx_{k-1}),
\E[\hat{\vg}_{\delta,k-1}(\vx_{k-1})\mid\vx_{k-1}] \Big\rangle
\nonumber\\
&\qquad
+ \frac{L}{1+s}\,\eta_{k-1}^{\,1+s}\,
\E\!\left[
\|\hat{\vg}_{\delta,k-1}(\vx_{k-1})\|^{\,1+s}
\,\middle|\,
\vx_{k-1}
\right] \nonumber\\
&= f(\vx_{k-1})
- \eta_{k-1}
\big\langle \nabla f(\vx_{k-1}),
\nabla f(\vx_{k-1}) + \vb_{k-1} \big\rangle
\nonumber\\
&\qquad
+ \frac{L}{1+s}\,\eta_{k-1}^{\,1+s}\,
\E\!\left[
\|\hat{\vg}_{\delta,k-1}(\vx_{k-1})\|^{\,1+s}
\,\middle|\,
\vx_{k-1}
\right] \nonumber\\
&= f(\vx_{k-1})
- \eta_{k-1}\|\nabla f(\vx_{k-1})\|^2
- \eta_{k-1}\langle\nabla f(\vx_{k-1}),\vb_{k-1}\rangle
\nonumber\\
&\qquad
+ \frac{L}{1+s}\,\eta_{k-1}^{\,1+s}\,
\E\!\left[
\|\hat{\vg}_{\delta,k-1}(\vx_{k-1})\|^{\,1+s}
\,\middle|\,
\vx_{k-1}
\right].
\label{eq:case1-pre-young-3}
\end{align}

Using $-\langle u,v\rangle \le \tfrac{1}{2}(\|u\|^2+\|v\|^2)$ with
$u=\nabla f(\vx_{k-1})$ and $v=\vb_{k-1}$, we obtain
\begin{equation}\label{eq:case1-biasbound-s-3}
-\eta_{k-1}\langle\nabla f(\vx_{k-1}),\vb_{k-1}\rangle
\le \frac{\eta_{k-1}}{2}\|\nabla f(\vx_{k-1})\|^2
+ \frac{\eta_{k-1}}{2}\|\vb_{k-1}\|^2.
\end{equation}
Substituting \eqref{eq:case1-biasbound-s-3} into \eqref{eq:case1-pre-young-3} yields
\begin{equation}\label{eq:case1-descent-s}
\begin{aligned}
\E[f(\vx_k)\mid\vx_{k-1}]
&\le f(\vx_{k-1})
- \frac{\eta_{k-1}}{2}\|\nabla f(\vx_{k-1})\|^2
+ \frac{\eta_{k-1}}{2}\|\vb_{k-1}\|^2 
+ \frac{L}{1+s}\,\eta_{k-1}^{\,1+s}\,
\E[\|\hat{\vg}_{\delta,k-1}(\vx_{k-1})\|^{\,1+s}\mid\vx_{k-1}].
\end{aligned}
\end{equation}

Using Lemma \ref{lem:clipped-bias-delta-s-3} on the last two terms of the RHS above,
\begin{equation}\label{eq:case1-descent-s-final-explicit-3}
\begin{aligned}
\E[f(\vx_k)\mid\vx_{k-1}]
&\le f(\vx_{k-1})
- \frac{\eta_{k-1}}{2}\|\nabla f(\vx_{k-1})\|^2
+ \frac{\eta_{k-1}}{2} C_{bias}(\theta, \delta) \sigma^{2\alpha} \gamma^{2(1-\alpha)} \\[6pt]
&\quad
+ \frac{L\,\eta_{k-1}^{\,1+s}}{1+s}
\Bigg(
2^{s}\|\nabla f(\vx_{k-1})\|^{1+s}
+ 2^{s}\sigma^{1+s}
+ C_{tail}(\theta, \delta) \sigma^\alpha \gamma^{1+s-\alpha}
\Bigg)\\ 
&\;\le\;
f(\vx_{k-1})
- \frac{\eta_{k-1}}{2}\|\nabla f(\vx_{k-1})\|^2
+ \frac{\eta_{k-1}}{2} C_{bias}(\theta, \delta) \sigma^{2\alpha} \gamma^{2(1-\alpha)} \\[6pt]
&\quad
+\frac{L\,\eta_{k-1}^{\,1+s}}{1+s}
\Big(
2^{s}(\theta\gamma)^{1+s}
+ 2^{s}\sigma^{1+s}
+ C_{tail}(\theta, \delta) \sigma^\alpha \gamma^{1+s-\alpha}
\Big)
\end{aligned}
\end{equation}

In the last line above, since we are in Case~1, we have used $\|\nabla f(\vx_{k-1})\|^{1+s} \le (\theta\gamma)^{1+s}$ as $\|\nabla f(\vx_{k-1})\|\le \theta\gamma$ for Case 1.

Defining the remainder term
\begin{equation}
\begin{aligned}
\mathcal{R}_{1,k-1}
&:=
\frac{\eta_{k-1}}{2} C_{bias}(\theta, \delta) \sigma^{2\alpha} \gamma^{2(1-\alpha)}
+ \frac{L\,\eta_{k-1}^{\,1+s}}{1+s}
\Bigg(
2^{s}(\theta\gamma)^{1+s}
+ 2^{s}\sigma^{1+s}
+ C_{tail}(\theta, \delta) \sigma^\alpha \gamma^{1+s-\alpha}
\Bigg),
\end{aligned}
\label{eq:R1k}
\end{equation}
the inequality \eqref{eq:case1-descent-s-final-explicit-3} becomes
\begin{equation}\label{eq:case1-descent}
\E\big[f(\vx_k)\mid\vx_{k-1}\big]
\le
f(\vx_{k-1})
- \frac{\eta_{k-1}}{2}\,\|\nabla f(\vx_{k-1})\|^2
+ \mathcal{R}_{1,k-1}.
\end{equation}

This establishes the simplified descent inequality for Case~1.

\medskip\noindent\textbf{Case 2a: $\theta\gamma < \|\nabla f(\vx_{k-1})\| \le 2\gamma/\delta$.}
~\\
From the descent inequality \eqref{eq:main-descent-sholder-start-3} we have 
\begin{align}
    \E[f(\vx_k)\mid\vx_{k-1}]
    &\le f(\vx_{k-1})
    - \eta_{k-1}
    \Big\langle \nabla f(\vx_{k-1}),
    \E[\hat{\vg}_{\delta,k-1}(\vx_{k-1})\mid\vx_{k-1}] \Big\rangle
    \nonumber\\
    &\qquad
    + \frac{L}{1+s}\,\eta_{k-1}^{\,1+s}\,
    \E\!\left[
    \|\hat{\vg}_{\delta,k-1}(\vx_{k-1})\|^{\,1+s}
    \,\middle|\,
    \vx_{k-1}
    \right] 
    \label{eq:case2a-descent-raw}
\end{align}

We directly bound the inner product and the variance terms. Using Lemma \ref{lem:inner-gclip-case2-3} for the expected descent and Lemma \ref{lem:moment-gclip-case2-3} for the expected $(1+s)$-th moment and under the Case 2a condition $\|\nabla f(\vx_{k-1})\| \leq \frac{2\gamma}{\delta}$, we split the gradient norm to absorb it into the linear descent term:
\begin{align}
    \E[f(\vx_k) \mid \vx_{k-1}] &\leq f(\vx_{k-1}) - \eta_{k-1}c_0\gamma\|\nabla f(\vx_{k-1})\| \notag \\
    &\quad + \frac{L\eta_{k-1}^{1+s}}{1+s} \left( K_v\|\nabla f(\vx_{k-1})\|^{1+s} + 2^s\delta^{1+s}\sigma^{1+s} \right) \notag \\
    &= f(\vx_{k-1}) - \eta_{k-1}c_0\gamma\|\nabla f(\vx_{k-1})\| \notag \\
    &\quad + \frac{L\eta_{k-1}^{1+s}}{1+s} K_v \|\nabla f(\vx_{k-1})\|^s \|\nabla f(\vx_{k-1})\| + \frac{L\eta_{k-1}^{1+s}}{1+s} 2^s\delta^{1+s}\sigma^{1+s} \notag \\
    &\leq f(\vx_{k-1}) - \eta_{k-1}c_0\gamma\|\nabla f(\vx_{k-1})\| \notag \\
    &\quad + \frac{L\eta_{k-1}^{1+s}}{1+s} K_v \left(\frac{2\gamma}{\delta}\right)^s \|\nabla f(\vx_{k-1})\| + \frac{L\eta_{k-1}^{1+s}}{1+s} 2^s\delta^{1+s}\sigma^{1+s} \notag \\
    &= f(\vx_{k-1}) - \eta_{k-1} \gamma \left( c_0 - \frac{L\eta_{k-1}^s}{1+s} K_v \left(\frac{2}{\delta}\right)^s \gamma^{s-1} \right) \|\nabla f(\vx_{k-1})\| \notag \\
    &\quad + \frac{L\eta_{k-1}^{1+s}}{1+s} 2^s\delta^{1+s}\sigma^{1+s}. \label{eq:case2a_absorption}
\end{align}

where $c_0 = \frac{1}{8}M$ and $M = \min\big\{\frac{15}{16}\theta, \frac{1}{3}\big\}$.\\
To ensure the negative descent term dominates the gradient-dependent smoothness penalty, we require the learning rate to satisfy:
\begin{equation}
    \frac{L\eta_{k-1}^s}{1+s} K_v \left(\frac{2}{\delta}\right)^s \gamma^{s-1} \leq \frac{c_0}{2} \implies \eta_{k-1}^s \leq \frac{c_0 (1+s)}{2 L K_v} \left(\frac{\delta}{2}\right)^s \gamma^{1-s}. \label{eq:eta_condition_case2a}
\end{equation}

Assuming the condition in \eqref{eq:eta_condition_case2a} holds, substituting it back into \eqref{eq:case2a_absorption} yields the final bounded descent inequality for this case:
\begin{equation}
    \E[f(\vx_k) \mid \vx_{k-1}] \leq f(\vx_{k-1}) - \frac{c_0}{2} \eta_{k-1} \gamma \|\nabla f(\vx_{k-1})\| + \frac{L\eta_{k-1}^{1+s}}{1+s} 2^s\delta^{1+s}\sigma^{1+s}. \label{eq:case2a_final-s-3}
\end{equation}

Defining the remainder term
\begin{equation}
\begin{aligned}
\mathcal{R}_{2a,k-1}
&:=
\frac{L\,\eta_{k-1}^{\,1+s}}{1+s}
\Bigg(
2^{s}\delta^{1+s}\sigma^{1+s}
\Bigg),
\end{aligned}
\label{eq:R2ak}
\end{equation}
the inequality \eqref{eq:case2a_final-s-3} becomes
\begin{equation}\label{eq:case2a-descent}
\E\big[f(\vx_k)\mid\vx_{k-1}\big]
\le
f(\vx_{k-1})
- \eta_{k-1} \tfrac{c_0}{2} \gamma \|\nabla f(\vx_{k-1})\|
+ \mathcal{R}_{2a,k-1}.
\end{equation}

This establishes the simplified descent inequality for Case~2a.

\medskip\noindent\textbf{Case 2b: $\|\nabla f(\vx_{k-1})\| > 2\gamma/\delta$.}
~\\
From the descent inequality \eqref{eq:main-descent-sholder-start-3} we have 
\begin{align}
\E[f(\vx_k)\mid\vx_{k-1}]
&\le f(\vx_{k-1})
- \eta_{k-1}
\Big\langle \nabla f(\vx_{k-1}),
\E[\hat{\vg}_{\delta,k-1}(\vx_{k-1})\mid\vx_{k-1}] \Big\rangle
\nonumber\\
&\qquad
+ \frac{L}{1+s}\,\eta_{k-1}^{\,1+s}\,
\E\!\left[
\|\hat{\vg}_{\delta,k-1}(\vx_{k-1})\|^{\,1+s}
\,\middle|\,
\vx_{k-1}
\right] \nonumber\\.
\label{eq:case1-pre-young-b}
\end{align}

In this extreme scaling regime, the linear descent bound from Lemma \ref{lem:moment-gclip-case2-3} is insufficient to overcome the fractional variance. Instead, we re-derive the expected inner product directly using the unified scaling region $C$ and the geometric buffer zone. 

Partitioning the expectation and utilizing the zero-mean property of the noise ($\E[\vxi_{k-1}\mid\vx_{k-1}] = 0$), we have:
We partition the expectation and expand the inner product through a single chain of bounds:
\begin{align}
\Big\langle \nabla f(\vx_{k-1}),
\E[\hat{\vg}_{\delta,k-1}\mid\vx_{k-1}]
\Big\rangle
&=
\E\!\left[
\begin{aligned}
&
\langle \nabla f(\vx_{k-1}),
\hat{\vg}_{\delta,k-1} \rangle
\mathbf 1_C
\end{aligned}
\;\middle|\;
\vx_{k-1}
\right]
\nonumber
\\
&\quad +
\E\!\left[
\begin{aligned}
&
\langle \nabla f(\vx_{k-1}),
\hat{\vg}_{\delta,k-1} \rangle
\mathbf 1_{A \cup B}
\end{aligned}
\;\middle|\;
\vx_{k-1}
\right]
\nonumber
\\
&\ge
\E\!\left[
\langle \nabla f(\vx_{k-1}),
\delta(\nabla f(\vx_{k-1}) + \vxi_{k-1})
\rangle
\mathbf 1_C
\;\middle|\;
\vx_{k-1}
\right]
\nonumber
\\
&\quad -
\E\!\left[
\|\nabla f(\vx_{k-1})\|
\|\hat{\vg}_{\delta,k-1}\|
\mathbf 1_{A \cup B}
\;\middle|\;
\vx_{k-1}
\right]
\nonumber
\\
&\ge
\delta \|\nabla f(\vx_{k-1})\|^2 \Pr(C)
+
\delta
\Big\langle
\nabla f(\vx_{k-1}),
\E[\vxi_{k-1}\mathbf 1_C \mid \vx_{k-1}]
\Big\rangle
\nonumber
\\
&\quad -
\gamma \|\nabla f(\vx_{k-1})\| \Pr(A \cup B)
\nonumber
\\
&=
\delta \|\nabla f(\vx_{k-1})\|^2 \Pr(C)
-
\delta
\Big\langle
\nabla f(\vx_{k-1}),
\E[\vxi_{k-1}\mathbf 1_{A \cup B} \mid \vx_{k-1}]
\Big\rangle
\nonumber
\\
&\quad -
\gamma \|\nabla f(\vx_{k-1})\| \Pr(A \cup B)
\nonumber
\\
&\ge
\delta \|\nabla f(\vx_{k-1})\|^2 \Pr(C)
-
\delta \|\nabla f(\vx_{k-1})\|
\E\!\left[
\|\vxi_{k-1}\|
\mathbf 1_{A \cup B}
\;\middle|\;
\vx_{k-1}
\right]
\nonumber
\\
&\quad -
\gamma \|\nabla f(\vx_{k-1})\|.
\label{eq:case2b-inner-raw}
\end{align}

The derivation proceeds as follows. First, we decompose the total conditional expectation into two disjoint complementary regions, namely the scaling region $C$ and the bounded clipping regions $A \cup B$, yielding an exact equality by partitioning the sample space. Next, in Region $C$ we substitute the explicit form of the estimator $\hat{\vg}_{\delta,k-1} = \delta(\nabla f(\vx_{k-1}) + \vxi_{k-1})$, while in Regions $A \cup B$ we apply the inequality, $\langle u, v \rangle \ge -\|u\|\|v\|$, to obtain a lower bound on the inner product. We then expand the inner product inside Region $C$ and use the fact that $\nabla f(\vx_{k-1})$ is deterministic conditioned on $\vx_{k-1}$ in order to extract it from the expectation; at the same time, we bound the estimator norm in Regions $A \cup B$ using the structural constraint $\|\hat{\vg}_{\delta,k-1}\| \le \gamma$. Subsequently, we invoke the zero-mean property of the stochastic noise, $\E[\vxi_{k-1}\mid\vx_{k-1}] = 0$, which implies that the total expected noise over all partitioned regions must vanish; therefore,
\[
\E[\vxi_{k-1} \mathbf 1_C \mid \vx_{k-1}]
=
-\,\E[\vxi_{k-1} \mathbf 1_{A \cup B} \mid \vx_{k-1}],
\]
allowing the noise contribution to be transferred entirely outside Region $C$. Finally, we apply the inequality, $\langle u, v \rangle \ge -\|u\|\|v\|$ once more to bound the resulting inner product involving the transferred noise term, and we use the trivial bound $\Pr(A \cup B) \le 1$ to obtain the worst-case linear penalty, completing the argument.

To evaluate $\E\!\left[\|\vxi_{k-1}\|\mathbf 1_{A \cup B}\;\middle|\;\vx_{k-1}\right]$, we first determine the absolute minimum noise magnitude required for the algorithm to land in the clipping regions $A \cup B$. By definition, if the stochastic gradient triggers clipping, its magnitude is bounded by $\|\vg_{k-1}\| \le \gamma/\delta$. Applying the reverse triangle inequality, we lower-bound the necessary noise vector:
\begin{equation}
\|\vxi_{k-1}\| = \|\vg_{k-1} - \nabla f(\vx_{k-1})\| \ge \|\nabla f(\vx_{k-1})\| - \|\vg_{k-1}\| \ge \|\nabla f(\vx_{k-1})\| - \frac{\gamma}{\delta}.
\end{equation}

Because we are strictly analyzing Case 2b, the true gradient satisfies the geometric boundary $\|\nabla f(\vx_{k-1})\| > 2\gamma/\delta$. Rearranging this yields $\gamma/\delta < \frac{1}{2}\|\nabla f(\vx_{k-1})\|$. Substituting this strict upper bound into our noise inequality establishes a massive geometric buffer zone:
\begin{equation}\label{eq:geometric-buffer}
\|\vxi_{k-1}\| > \|\nabla f(\vx_{k-1})\| - \frac{1}{2}\|\nabla f(\vx_{k-1})\| = \frac{1}{2}\|\nabla f(\vx_{k-1})\|.
\end{equation}

This lower bound physically guarantees that any noise capable of forcing the estimator into the clipping regions must be extremely large. Because it is bounded safely away from zero, it strictly prevents any division-by-zero singularities. 

We now evaluate the adversarial noise expectation. To extract the bounded $\alpha$-th moment ($\E[\|\vxi_{k-1}\|^\alpha \mid \vx_{k-1}] \le \sigma^\alpha$), we multiply and divide the integrand by $\|\vxi_{k-1}\|^{\alpha-1}$. Since we are integrating strictly over the event $A \cup B$, we safely substitute the geometric buffer from \eqref{eq:geometric-buffer} into the denominator, allowing us to pull the deterministic gradient out of the expectation:
\begin{align}
\E\big[\|\vxi_{k-1}\|\1_{A \cup B}\mid\vx_{k-1}\big] 
&= \E\!\left[ \frac{\|\vxi_{k-1}\|^\alpha}{\|\vxi_{k-1}\|^{\alpha-1}} \1_{A \cup B} \,\middle|\, \vx_{k-1} \right] 
\le \E\!\left[ \frac{\|\vxi_{k-1}\|^\alpha}{\left(\frac{1}{2}\|\nabla f(\vx_{k-1})\|\right)^{\alpha-1}} \1_{A \cup B} \,\middle|\, \vx_{k-1} \right] \nonumber \\
&= \frac{2^{\alpha-1}}{\|\nabla f(\vx_{k-1})\|^{\alpha-1}} \E\big[\|\vxi_{k-1}\|^\alpha \1_{A \cup B}\mid\vx_{k-1}\big] \nonumber \\
&\le \frac{2^{\alpha-1}\sigma^\alpha}{\|\nabla f(\vx_{k-1})\|^{\alpha-1}}.
\label{eq:case2b-moment-extraction}
\end{align}

Substituting the bounded moment \eqref{eq:case2b-moment-extraction} into the inner product expansion \eqref{eq:case2b-inner-raw} yields the penalized descent signal:
\begin{align}
\Big\langle
\nabla f(\vx_{k-1}),
\E[\hat{\vg}_{\delta,k-1}\mid\vx_{k-1}]
\Big\rangle
&\ge
\delta\|\nabla f(\vx_{k-1})\|^2 \Pr(C)
\nonumber
-
\delta\|\nabla f(\vx_{k-1})\|
\left(
\frac{2^{\alpha-1}\sigma^\alpha}
{\|\nabla f(\vx_{k-1})\|^{\alpha-1}}
\right)
-
\gamma\|\nabla f(\vx_{k-1})\|
\nonumber
\\
&=
\delta\|\nabla f(\vx_{k-1})\|^2 \Pr(C)
-
2^{\alpha-1}\delta\sigma^\alpha
\|\nabla f(\vx_{k-1})\|^{2-\alpha}
-
\gamma\|\nabla f(\vx_{k-1})\|.
\label{eq:case2b-inner-sub1}
\end{align}

To establish a strictly quadratic lower bound, we must first control the probability mass $\Pr(C)$. Since Region $C$ is the complement of $A \cup B$, we apply Markov's inequality directly to the geometric buffer event established in \eqref{eq:geometric-buffer}:
\begin{equation}
\Pr(A \cup B) \le \Pr\left(\|\vxi_{k-1}\| > \frac{1}{2}\|\nabla f(\vx_{k-1})\|\right) \le \frac{\E[\|\vxi_{k-1}\|^\alpha\mid\vx_{k-1}]}{\left(\frac{1}{2}\|\nabla f(\vx_{k-1})\|\right)^\alpha} \le \frac{2^\alpha \sigma^\alpha}{\|\nabla f(\vx_{k-1})\|^\alpha}.
\end{equation}
Therefore, the probability of remaining in the scaling zone is lower-bounded by $\Pr(C) = 1 - \Pr(A \cup B) \ge 1 - \frac{2^\alpha \sigma^\alpha}{\|\nabla f(\vx_{k-1})\|^\alpha}$.

Substituting this probability bound into~\eqref{eq:case2b-inner-sub1} and factoring out the leading quadratic term $\delta\|\nabla f(\vx_{k-1})\|^2$ gives
\begin{align}
\Big\langle
\nabla f(\vx_{k-1}),
\E[\hat{\vg}_{\delta,k-1}\mid\vx_{k-1}]
\Big\rangle
&\ge
\delta\|\nabla f(\vx_{k-1})\|^2
\left(
1 -
\frac{2^\alpha\sigma^\alpha}
{\|\nabla f(\vx_{k-1})\|^\alpha}
\right)
\nonumber
-
2^{\alpha-1}\delta\sigma^\alpha
\|\nabla f(\vx_{k-1})\|^{2-\alpha}
-
\gamma\|\nabla f(\vx_{k-1})\|
\nonumber
\\
&=
\delta\|\nabla f(\vx_{k-1})\|^2
\Bigg[
1
-
\frac{2^\alpha\sigma^\alpha}
{\|\nabla f(\vx_{k-1})\|^\alpha}
\nonumber
-
\frac{2^{\alpha-1}\sigma^\alpha}
{\|\nabla f(\vx_{k-1})\|^\alpha}
-
\frac{\gamma}
{\delta\|\nabla f(\vx_{k-1})\|}
\Bigg]
\nonumber
\\
&=
\delta\|\nabla f(\vx_{k-1})\|^2
\left[
1
-
\frac{3\cdot2^{\alpha-1}\sigma^\alpha}
{\|\nabla f(\vx_{k-1})\|^\alpha}
-
\frac{\gamma}
{\delta\|\nabla f(\vx_{k-1})\|}
\right].
\label{eq:case2b-inner-bracket}
\end{align}

Since we are in Case~2b, the gradient satisfies $\|\nabla f(\vx_{k-1})\|>2\gamma/\delta$, which implies
\[
\frac{\gamma}{\delta\|\nabla f(\vx_{k-1})\|}<\frac12 .
\]
Applying this bound to~\eqref{eq:case2b-inner-bracket} yields
\begin{equation}
\Big\langle \nabla f(\vx_{k-1}), \E[\hat{\vg}_{\delta,k-1}\mid\vx_{k-1}] \Big\rangle
>
\delta\|\nabla f(\vx_{k-1})\|^2
\left[
\frac12 - \frac{3\cdot2^{\alpha-1}\sigma^\alpha}{\|\nabla f(\vx_{k-1})\|^\alpha}
\right].
\label{eq:case2b-inner-reduced}
\end{equation}

To ensure the descent signal remains positive, we require
\[
\frac{3\cdot2^{\alpha-1}\sigma^\alpha}{\|\nabla f(\vx_{k-1})\|^\alpha}\le\frac14,
\]
which holds whenever
\begin{equation}
\|\nabla f(\mathbf{x}_{k-1})\|
\ge (12\cdot2^{\alpha-1})^{1/\alpha}\sigma .
\label{eq:trans-condn1}
\end{equation}

Because $\|\nabla f(\mathbf{x}_{k-1})\|>2\gamma/\delta$ in Case~2b, this condition holds throughout the region provided
\begin{equation}
\frac{2\gamma}{\delta}
\ge (12\cdot2^{\alpha-1})^{1/\alpha}\sigma .
\label{eq:trans-condn2}
\end{equation}

Combining~\eqref{eq:trans-condn1} and~\eqref{eq:trans-condn2} yields
\[
\|\nabla f(\mathbf{x}_{k-1})\|>(12\cdot2^{\alpha-1})^{1/\alpha}\sigma,
\]
and therefore
\[
\frac12-\frac{3\cdot2^{\alpha-1}\sigma^\alpha}{\|\nabla f(\mathbf{x}_{k-1})\|^\alpha}
\ge
\frac12-\frac{3\cdot2^{\alpha-1}\sigma^\alpha}{12\cdot2^{\alpha-1}\sigma^\alpha}
=\frac14 .
\]

Consequently,
\begin{equation}
\Big\langle \nabla f(\mathbf{x}_{k-1}),
\mathbb{E}[\hat{\mathbf{g}}_{\delta,k-1}\mid\mathbf{x}_{k-1}]
\Big\rangle
\ge
c_2\,\delta\|\nabla f(\mathbf{x}_{k-1})\|^2,
\qquad
c_2=\tfrac14 .
\label{eq:case2b-inner-final}
\end{equation}

Returning to the master descent inequality \eqref{eq:main-descent-sholder-start-3}, we substitute this strictly quadratic inner product bound \eqref{eq:case2b-inner-final} and the fractional variance bound from Lemma \ref{lem:moment-gclip-case2-3}:
\begin{equation}\label{eq:case2b-descent-pre-absorb}
\begin{aligned}
\E[f(\vx_k)\mid\vx_{k-1}]
&\le f(\vx_{k-1})
- \eta_{k-1} c_2 \delta \|\nabla f(\vx_{k-1})\|^2 \\
&\qquad
+ \frac{L\,\eta_{k-1}^{\,1+s}}{1+s}
\Bigg(
K_v \|\nabla f(\vx_{k-1})\|^{1+s}
+ 2^{s}\delta^{1+s}\sigma^{1+s}
\Bigg).
\end{aligned}
\end{equation}

To mathematically crush the fractional variance, we enforce a learning rate schedule $\eta_{k-1}$ such that the gradient-dependent variance consumes at most half of the quadratic descent signal:
\begin{equation}
\frac{L\,\eta_{k-1}^{\,1+s}}{1+s} K_v \|\nabla f(\vx_{k-1})\|^{1+s} \le \frac{\eta_{k-1}}{2} c_2 \delta \|\nabla f(\vx_{k-1})\|^2.
\end{equation}
Since $\|\nabla f(\vx_{k-1})\| > 2\gamma/\delta$ in Case 2b, satisfying the strict condition $\frac{L\eta_{k-1}^s}{1+s} K_v \le \frac{1}{2}c_2 \delta (2\gamma/\delta)^{1-s}$ guarantees this absorption globally. Substituting this absorption into \eqref{eq:case2b-descent-pre-absorb} yields a pure net quadratic descent:
\begin{equation}\label{eq:case2b-descent-absorbed}
\begin{aligned}
\E[f(\vx_k)\mid\vx_{k-1}]
&\le f(\vx_{k-1})
- \frac{\eta_{k-1}}{2} c_2 \delta \|\nabla f(\vx_{k-1})\|^2
+ \frac{L\,\eta_{k-1}^{\,1+s}}{1+s}
\Big( 2^{s}\delta^{1+s}\sigma^{1+s} \Big).
\end{aligned}
\end{equation}

To unify this with the linear metric of Case 2a without retaining the scaling parameter $\delta$ in the unified coefficient, we apply the geometric substitution enabled exclusively by the Case 2b boundary ($\delta\|\nabla f(\vx_{k-1})\| > 2\gamma$):
\begin{equation}
- \frac{\eta_{k-1}}{2} c_2 \delta \|\nabla f(\vx_{k-1})\|^2
= - \frac{\eta_{k-1}}{2} c_2 \big(\delta \|\nabla f(\vx_{k-1})\|\big) \|\nabla f(\vx_{k-1})\|
\le - \eta_{k-1} c_2 \gamma \|\nabla f(\vx_{k-1})\|.
\end{equation}

Substituting this upper bound into \eqref{eq:case2b-descent-absorbed}, we obtain the finalized form:
\begin{equation}\label{eq:case2b-descent-s-explicit}
\begin{aligned}
\E[f(\vx_k)\mid\vx_{k-1}]
&\le f(\vx_{k-1})
- \eta_{k-1} c_2 \gamma \|\nabla f(\vx_{k-1})\|
+ \frac{L\,\eta_{k-1}^{\,1+s}}{1+s}
\Big(
2^{s}\delta^{1+s}\sigma^{1+s}
\Big).
\end{aligned}
\end{equation}

Defining the remainder term
\begin{equation}
\begin{aligned}
\mathcal{R}_{2b,k-1}
&:=
\frac{L\,\eta_{k-1}^{\,1+s}}{1+s}
\Big(
2^{s}\delta^{1+s}\sigma^{1+s}
\Big),
\end{aligned}
\label{eq:R2bk}
\end{equation}
the inequality \eqref{eq:case2b-descent-s-explicit} becomes
\begin{equation}\label{eq:case2b-descent}
\E\big[f(\vx_k)\mid\vx_{k-1}\big]
\le
f(\vx_{k-1})
- \eta_{k-1} c_2 \gamma \|\nabla f(\vx_{k-1})\|
+ \mathcal{R}_{2b,k-1}.
\end{equation}

This establishes the simplified descent inequality for Case~2b.

\medskip\noindent\textbf{Unified single-step descent inequality.}
Having established the expected progress in all three gradient regimes, we now unify them into a single global descent inequality valid for the entire state space.

Recall the individual descent bounds derived for each case:
\begin{align}
\textbf{Case 1: } \quad 
&\E\big[f(\vx_k)\mid\vx_{k-1}\big]
\le f(\vx_{k-1})
- \tfrac{\eta_{k-1}}{2} \|\nabla f(\vx_{k-1})\|^2
+ \mathcal{R}_{1}, \label{eq:recall-case1} \\
\textbf{Case 2a: } \quad 
&\E\big[f(\vx_k)\mid\vx_{k-1}\big]
\le f(\vx_{k-1})
- \eta_{k-1} \tfrac{c_0}{2} \gamma \|\nabla f(\vx_{k-1})\|
+ \mathcal{R}_{2a}, \label{eq:recall-case2a} \\
\textbf{Case 2b: } \quad 
&\E\big[f(\vx_k)\mid\vx_{k-1}\big]
\le f(\vx_{k-1})
- \eta_{k-1} c_2 \gamma \|\nabla f(\vx_{k-1})\|
+ \mathcal{R}_{2b}, \label{eq:recall-case2b}
\end{align}
where the corresponding remainder terms, tracking the bias and variance penalties, are defined explicitly as:
\begin{align}
\mathcal{R}_{1} &= \frac{\eta_{k-1}}{2} C_{bias}(\theta, \delta) \sigma^{2\alpha} \gamma^{2(1-\alpha)} + \frac{L\,\eta_{k-1}^{\,1+s}}{1+s}
\Big(
2^{s}\theta^{1+s}\gamma^{1+s}
+ 2^{s}\sigma^{1+s}
+ C_{tail}(\theta, \delta) \frac{\sigma^\alpha}{\gamma^{\alpha-s-1}}
\Big), \label{eq:def-R1} \\
\mathcal{R}_{2a} &= \frac{L\,\eta_{k-1}^{\,1+s}}{1+s} \Big( 2^s \delta^{1+s} \sigma^{1+s} \Big), \label{eq:def-R2a} \\
\mathcal{R}_{2b} &= \frac{L\,\eta_{k-1}^{\,1+s}}{1+s} \Big( 2^s \delta^{1+s} \sigma^{1+s} \Big). \label{eq:def-R2b}
\end{align}

Notice that because the gradient-dependent variance terms were completely absorbed into the linear descent signals in both Case 2a and Case 2b, their respective remainders consist strictly of the irreducible noise floor. Consequently, $\mathcal{R}_{2b} = \mathcal{R}_{2a}$. 

To bound the maximum per-step penalty $\widetilde{\mathcal{R}}_{k-1}$ globally, we observe that the worst-case error must be bounded by the sum of the maximum penalties from the small-gradient and large-gradient regimes. Therefore, we define the unified global remainder as:
\begin{equation}\label{eq:unified-remainder-3}
\widetilde{\mathcal{R}}_{k-1} := \mathcal{R}_{1} + \mathcal{R}_{2a}.
\end{equation}

Next, we unify the descent signal. Let $c$ be the smallest of the three per-case descent coefficients,
\begin{equation}\label{eq:unified-coeff}
c := \min\Big\{ \tfrac{1}{2},\ \tfrac{c_0}{2},\ c_2 \Big\} = \tfrac{c_0}{2},
\end{equation}
where the last equality uses $c_0=\tfrac{1}{8}M$ with $M=\min\{\tfrac{15}{16}\theta,\tfrac13\}$ and $c_2=\tfrac14$, so that $\tfrac{c_0}{2}=\tfrac{M}{16}\le\tfrac{1}{48}$ is the minimum. Since $c\le\tfrac12$, $c\le\tfrac{c_0}{2}$, and $c\le c_2$, each case's descent bound in \eqref{eq:recall-case1}--\eqref{eq:recall-case2b} lower-bounds $c\,\eta_{k-1}\min\{\|\nabla f(\vx_{k-1})\|^2,\gamma\|\nabla f(\vx_{k-1})\|\}$: in Case~1 the quadratic term dominates $\min\{\cdot\}$ and $\tfrac12\ge c$, while in Cases~2a and~2b the linear term dominates $\min\{\cdot\}$ and $\tfrac{c_0}{2},c_2\ge c$.

Comparing the isolated descent inequalities \eqref{eq:recall-case1}, \eqref{eq:recall-case2a}, and \eqref{eq:recall-case2b}, we observe the algorithm's geometric behavior: it guarantees a strictly quadratic descent proportional to $\|\nabla f(\vx_{k-1})\|^2$ when the true gradient is small (Case 1), and a linear descent proportional to $\gamma\|\nabla f(\vx_{k-1})\|$ when the true gradient is large (Cases 2a and 2b).

Because the smaller of these two terms always correctly lower-bounds the descent in its respective regime, we can universally lower-bound the algorithmic progress across the entire state space using the generalized metric function:
\begin{equation}\label{eq:generalized-metric}
\min\Big\{ \|\nabla f(\vx_{k-1})\|^2, \gamma\|\nabla f(\vx_{k-1})\| \Big\}.
\end{equation}

Substituting this unified metric \eqref{eq:generalized-metric}, the minimum descent coefficient $c$, and the worst-case remainder $\widetilde{\mathcal{R}}_{k-1}$ from \eqref{eq:unified-remainder-3} into our regime-specific bounds yields the global single-step descent inequality, valid for all iterations $k \ge 1$:
\begin{equation}\label{eq:unified-descent}
\E[f(\vx_k) \mid \vx_{k-1}] \le f(\vx_{k-1}) - \eta_{k-1} c \min\Big\{ \|\nabla f(\vx_{k-1})\|^2, \gamma\|\nabla f(\vx_{k-1})\| \Big\} + \widetilde{\mathcal{R}}_{k-1}.
\end{equation}

\medskip\noindent\textbf{Summation and stationarity.}
We now sum the unified descent inequality over iterations to derive a stationarity guarantee.

Taking total expectation and using the tower property, the unified single-step inequality implies
\begin{equation}\label{eq:unified-descent-expectation-s}
\mathbb{E}\big[f(\vx_k)\big]
\le
\mathbb{E}\big[f(\vx_{k-1})\big]
-
\eta_{k-1} c\,
\mathbb{E}\!\left[
\min\!\left\{
\|\nabla f(\vx_{k-1})\|^2,\,
\gamma\|\nabla f(\vx_{k-1})\|
\right\}
\right]
+
\mathbb{E}\big[\widetilde{\mathcal{R}}_{k-1}\big].
\end{equation}

Summing \eqref{eq:unified-descent-expectation-s} over $k=1,\dots,T$ yields
\begin{equation}\label{eq:summed-descent}
\begin{aligned}
c
\sum_{k=1}^T
\eta_{k-1}
\mathbb{E}\!\left[
\min\!\left\{
\|\nabla f(\vx_{k-1})\|^2,\,
\gamma\|\nabla f(\vx_{k-1})\|
\right\}
\right]
\le
\mathbb{E}\big[f(\vx_0)\big]
-
\mathbb{E}\big[f(\vx_T)\big]
+
\sum_{k=1}^T
\widetilde{\mathcal{R}}_{k-1}.
\end{aligned}
\end{equation}

Assuming that the objective function is bounded below by $f_\star$, we have
$\mathbb{E}[f(\vx_T)] \ge f_\star$, and therefore
\begin{equation}\label{eq:summed-descent-final-s}
c
\sum_{k=1}^T
\eta_{k-1}
\mathbb{E}\!\left[
\min\!\left\{
\|\nabla f(\vx_{k-1})\|^2,\,
\gamma\|\nabla f(\vx_{k-1})\|
\right\}
\right]
\le
f(\vx_0) - f_\star
+
T\cdot\widetilde{\mathcal{R}}_{k-1},
\end{equation}

Dividing both sides of \eqref{eq:summed-descent-final-s} by
$c T \eta_{k-1}$ gives
\begin{equation}\label{eq:average-stationarity-s}
\frac{
\sum_{k=1}^T
\eta_{k-1}
\mathbb{E}\!\left[
\min\!\left\{
\|\nabla f(\vx_{k-1})\|^2,\,
\gamma\|\nabla f(\vx_{k-1})\|
\right\}
\right]
}{
T \eta_{k-1}
}
\le
\frac{f(\vx_0) - f_\star}{c T \eta_{k-1}}
+
\frac{\widetilde{\mathcal{R}}_{k-1}
}{c \eta_{k-1}}.
\end{equation}

In particular, for a constant step size $\eta_{k-1}=\eta$, inequality
\eqref{eq:average-stationarity-s} implies
\begin{equation}\label{eq:average-stationarity-constant-s}
\frac{1}{T}
\sum_{k=1}^T
\mathbb{E}\!\left[
\min\!\left\{
\|\nabla f(\vx_{k-1})\|^2,\,
\gamma\|\nabla f(\vx_{k-1})\|
\right\}
\right]
\le
\frac{f(\vx_0) - f_\star}{c \eta T}
+
\frac{1}{c \eta}
\widetilde{\mathcal{R}}_{k-1}.
\end{equation}

Define $f(\vx_0) - f_\star := F_0$, then \eqref{eq:average-stationarity-constant-s} becomes,
\begin{equation}\label{eq:average-stationarity-mod-s}
\frac{1}{T}
\sum_{k=1}^T
\mathbb{E}\!\left[
\min\!\left\{
\|\nabla f(\vx_{k-1})\|^2,\,
\gamma\|\nabla f(\vx_{k-1})\|
\right\}
\right]
\le
\underbrace{\frac{F_0}{c \eta T}}_{T_1}
+
\underbrace{\frac{1}{c \eta}\,\widetilde{\mathcal{R}}_{k-1}}_{T_2}.
\end{equation}

Consider the explicit expansion of $T_1 + T_2$. Recalling that the gradient-dependent variance in Cases 2a and 2b is fully absorbed, the maximum remainder is bounded by $\widetilde{\mathcal{R}}_{k-1} \le \mathcal{R}_1 + \mathcal{R}_{2a}$:
\begin{align}
T_1 + T_2
&= \frac{F_0}{c \eta T} + \frac{1}{c \eta}\,\Big( \mathcal{R}_1 + \mathcal{R}_{2a} \Big) \nonumber\\[4pt]
&= \frac{F_0}{c \eta T}
+ \frac{1}{c}
\left[
\frac{C_{bias}(\theta, \delta) \sigma^{2\alpha}}{2 \gamma^{2(\alpha-1)}}
\right] \nonumber\\[4pt]
&\quad +
\frac{L\,\eta^{s}}{c (1+s)}
\Bigg(
2^s \theta^{1+s}\gamma^{1+s}
+ 2^{s}\sigma^{1+s}
+ C_{tail}(\theta, \delta) \frac{\sigma^\alpha}{\gamma^{\alpha-s-1}}
+ 2^s \delta^{1+s} \sigma^{1+s}
\Bigg) .
\label{eq:detailed-form-s}
\end{align}

We observe that in the last bracket of \eqref{eq:detailed-form-s}, the term $2^s \theta^{1+s}\gamma^{1+s}$ is the dominant factor as $\gamma$ increases. The term proportional to $\gamma^{-(\alpha-s-1)}$ decays (since $\alpha > 1+s$), and the remaining noise terms are constant with respect to $\gamma$.
Therefore, assuming $\gamma \ge 1$, we can upper bound the bracketed expression as:
\begin{align*}
&2^s \theta^{1+s}\gamma^{1+s}
+ 2^{s}\sigma^{1+s}
+ C_{tail}(\theta, \delta) \frac{\sigma^\alpha}{\gamma^{\alpha-s-1}}
+ 2^s \delta^{1+s} \sigma^{1+s} \\
&\qquad= 
\gamma^{1+s}
\Bigg[
2^s \theta^{1+s}
+ 
\frac{2^{s}\sigma^{1+s}}{\gamma^{1+s}}
+ 
\frac{C_{tail}(\theta, \delta)\sigma^\alpha}{\gamma^{\alpha}}
+ 
\frac{2^s \delta^{1+s} \sigma^{1+s}}{\gamma^{1+s}}
\Bigg]  \\
&\qquad\le 
K\,\gamma^{1+s},
\end{align*}
where $K$ is a constant depending on $\sigma, \delta, \theta, s$ but independent of $\gamma$ and $T$.
Substituting this bound into \eqref{eq:detailed-form-s}, we obtain the simplified upper bound:
\begin{equation}
\label{eq:abstract-form-s}
T_1 + T_2
\le
\frac{C_1}{\eta T}
+ \frac{C_2}{\gamma^{2(\alpha-1)}}
+ C_3\,\eta^{s}\gamma^{1+s},
\end{equation}
where we define the constants:
\[
C_1 = \frac{F_0}{c},
\qquad
C_2 = \frac{C_{bias}(\theta, \delta) \sigma^{2\alpha}}{2 c},
\qquad
C_3 = \frac{L K}{c(1+s)} .
\]

Substituting the parameter choices $\eta = \left(\frac{C_1}{C_3}\right)^{\!\frac{1}{1+s}} T^{-\frac{1}{1+s}} \gamma^{-1}$ and $\gamma = \left(\frac{C_2}{C_4}\right)^{\!\frac{1}{2\alpha-1}} T^{\frac{s}{(1+s)(2\alpha-1)}}$ (where $C_4 = 2(C_1^{s}C_3)^{\frac{1}{1+s}}$) into inequality \eqref{eq:abstract-form-s}, we obtain
\[
\begin{aligned}
T_1 + T_2
&\le
\frac{C_1}{\eta T}
+ \frac{C_2}{\gamma^{2(\alpha-1)}}
+ C_3\, \eta^{s}\gamma^{1+s}
=
\frac{C_1}{T}
\left[
\left(\frac{C_1}{C_3}\right)^{\!\frac{1}{1+s}}
T^{-\frac{1}{1+s}}\gamma^{-1}
\right]^{-1}
+ \frac{C_2}{\gamma^{2(\alpha-1)}}
+ C_3
\left[
\left(\frac{C_1}{C_3}\right)^{\!\frac{1}{1+s}}
T^{-\frac{1}{1+s}}\gamma^{-1}
\right]^{s}
\gamma^{1+s}.
\end{aligned}
\]

Simplifying the first and third terms jointly gives
\[
\begin{aligned}
T_1 + T_2
&=
C_1^{\frac{s}{1+s}} C_3^{\frac{1}{1+s}}
\gamma\, T^{-\frac{s}{1+s}}
+ \frac{C_2}{\gamma^{2(\alpha-1)}} \
+ C_1^{\frac{s}{1+s}} C_3^{\frac{1}{1+s}}
\gamma\, T^{-\frac{s}{1+s}}.
\end{aligned}
\]

Combining the identical contributions yields
\[
T_1 + T_2
=
2\left(C_1^{s}C_3\right)^{\!\frac{1}{1+s}}
\gamma\, T^{-\frac{s}{1+s}}
+ \frac{C_2}{\gamma^{2(\alpha-1)}} .
\]

By the definition \(C_4 = 2(C_1^{s}C_3)^{1/(1+s)}\), this simplifies to
\begin{equation}
T_1 + T_2
\le
\frac{C_2}{\gamma^{2(\alpha-1)}}
+ C_4 \gamma\, T^{-\frac{s}{1+s}}.
\label{eq:rate-intermediate-s}
\end{equation}

Finally, substituting the choice of $\gamma$ into \eqref{eq:rate-intermediate-s} gives
\[
\begin{aligned}
T_1 + T_2
&\le
C_2
\left[
\left(\frac{C_2}{C_4}\right)^{\!\frac{1}{2\alpha-1}}
T^{\frac{s}{(1+s)(2\alpha-1)}}
\right]^{-2(\alpha-1)}  
+ C_4
\left[
\left(\frac{C_2}{C_4}\right)^{\!\frac{1}{2\alpha-1}}
T^{\frac{s}{(1+s)(2\alpha-1)}}
\right]
T^{-\frac{s}{1+s}},
\end{aligned}
\]
which simplifies identically in both exponents to
\begin{equation}\label{eqn:final_rate-s}
T_1 + T_2
=
\mathcal{O}\!\left( 
T^{-\frac{2s(\alpha-1)}{(1+s)(2\alpha-1)}}
\right).
\end{equation}

This establishes the desired convergence rate. Combining \eqref{eqn:final_rate-s} and \eqref{eq:average-stationarity-mod-s}, we obtain
\begin{equation}
\label{eq:final-rate-clean}
\frac{1}{T}\sum_{k=1}^{T} 
\mathbb{E}\!\left[\min\!\left\{\|\nabla f(\vx_k)\|^2,\,
\gamma\|\nabla f(\vx_k)\|\right\}\right]
=
\mathcal{O}\!\left(
T^{-\frac{2s(\alpha-1)}{(1+s)(2\alpha-1)}}
\right).
\end{equation}
\end{proof}

\section{Convergence of G-Clip under \texorpdfstring{$s$}{s}-H\"older and Heavy-Tailed Noise}
\label{app:gclip}

We analyse G-Clip under $(L,s)$-H\"older continuous gradients and heavy-tailed noise across
\emph{both} exponent regimes, $1+s\le\alpha$ and $1+s>\alpha$. Throughout, $\gamma>0$ is the
clipping threshold and $\theta\in(0,1)$ is a fixed gradient-threshold parameter, and every step is
split into a \emph{small-gradient case} $\|\nabla f(\vx)\|\le\theta\gamma$ and a
\emph{large-gradient case} $\|\nabla f(\vx)\|>\theta\gamma$. The two exponent regimes use the same
algorithm, the same case split, and the same set of auxiliary lemmas; the only place they diverge is
the $(1+s)$-th moment bound of Lemma~\ref{lem:gclip-bv}, which is stated separately for each regime.
That single difference is exactly the step that controls $\|\nabla f(\vx)\|^{1+s}$ in the
small-gradient case, and it is what yields the two different rates.


\medskip 
\begin{theorem}[Non-convex convergence of G-Clip in both regimes]\label{thm:gclip}
Suppose $f:\R^d\to\R$ is bounded below by $f^\star$, set $F_0:=f(\vx_0)-f^\star$, and let $\nabla f$
be $(L,s)$-H\"older continuous for some $s\in(0,1]$. Let the stochastic gradients satisfy
Assumption~\ref{ass:noise-s} with heavy-tail index $\alpha\in(1,2]$ and noise scale $\sigma$, and let
$\{\vx_k\}_{k=0}^{T-1}$ be generated by G-Clip run for $T$ iterations at a constant step size and
threshold. There exists a constant $C_1$ and parametric families of constants $C_2,C_3,C_4$,
$C_3'',C_4''$, and $\gamma_0$ (defined explicitly in the remark below). Then the following hold, for
every $T\ge1$.

\medskip
\noindent\textbf{(i) (For $1+s\le\alpha$.)}\; With
\[
\gamma=\max\Big\{\gamma_0,\ (C_2/C_4)^{\frac{1}{2\alpha-1}}T^{\frac{s}{(1+s)(2\alpha-1)}}\Big\},\qquad
\eta=\min\Big\{\Big(\tfrac{5M\theta(1+s)}{16L}\Big)^{1/s}\gamma^{\frac{1-s}{s}},\
(C_1/C_3)^{\frac1{1+s}}\gamma^{-1}T^{-\frac1{1+s}}\Big\},
\]
\[
\frac1T\sum_{k=1}^T\E\Big[\min\big\{\|\nabla f(\vx_k)\|^2,\ \gamma\|\nabla f(\vx_k)\|\big\}\Big]
= O\Big(T^{-\frac{2s(\alpha-1)}{(1+s)(2\alpha-1)}}\Big).
\]

\medskip
\noindent\textbf{(ii) (For $1+s > \alpha$.)}\; With
\[
\gamma=\max\Big\{\gamma_0,\ (C_2/C_4'')^{\frac{1+s}{(\alpha-1)+s(2\alpha-1)}}T^{\frac{s}{(\alpha-1)+s(2\alpha-1)}}\Big\},\qquad
\eta=\min\Big\{\Big(\tfrac{5M\theta(1+s)}{16L}\Big)^{1/s}\gamma^{\frac{1-s}{s}},\
(C_1/C_3'')^{\frac1{1+s}}\gamma^{-\frac{1+s-\alpha}{1+s}}T^{-\frac1{1+s}}\Big\},
\]
\[
\frac1T\sum_{k=1}^T\E\Big[\min\big\{\|\nabla f(\vx_k)\|^2,\ \gamma\|\nabla f(\vx_k)\|\big\}\Big]
= O\Big(T^{-\frac{2s(\alpha-1)}{(\alpha-1)+s(2\alpha-1)}}\Big).
\]
Taking $\gamma$ as a maximum and $\eta$ as a minimum in each part makes the threshold and
absorption requirements hold at every $T\ge1$; the final entry of each dominates for large $T$, where
$(\eta,\gamma)$ reduce to the schedule that gives the displayed bound.
\end{theorem}

The proof of the above theorem is given in Appendix \ref{subsec:gclip-proof}.

\medskip 
\begin{remark}
    The rate for $1+s\le\alpha$ for GClip as derived above is the same as the rate for $\delta$-GClip as reported in Table ~\ref{tab:comparison}.
    The rate for GClip derived above specialized to the case $s=1$ and $\alpha \le 2$ recovers the known rate for GClip from ~\cite{suvrit1912} as was also reviewed in Table ~\ref{tab:comparison}.
    It is interesting to note that there is a discontinuous jump in the rate of GClip at the boundary $\alpha = 1+s$.
\end{remark}

\medskip 
\begin{remark}[Explicit constants]\label{rem:gclip-const}
Fix $\theta\in(0,1)$ and set $M=\min\{\tfrac{15}{16}\theta,\tfrac13\}$. The small-gradient case
contributes a quadratic descent with coefficient $\tfrac12$ in regime~(i) and $\tfrac14$ in
regime~(ii); the large-gradient case contributes a linear descent with coefficient $\tfrac{5M}{16}$;
since $\tfrac{5M}{16}\le\tfrac14\le\tfrac12$, the unified coefficient is $c=\tfrac{5M}{16}$. The
same coefficient $c$ also governs the step-size absorption condition
$\eta^s\le\tfrac{5M\theta(1+s)}{16L}\gamma^{1-s}$ used in both regimes to keep the large-gradient
smoothness penalty below half the linear descent signal. The
bias, tail, and variance constants are
\begin{equation}\label{eq:C_const}
    C_{bias}(\theta)=\frac{2}{(1-\theta)^{2\alpha}},\qquad
    C_{tail}(\theta)=\frac{1}{(1-\theta)^\alpha},\qquad
    C_{var}(\theta)=2^s(1+\theta)^{1+s-\alpha}+\frac{1}{(1-\theta)^\alpha},
\end{equation}

the Young constant used in regime~(ii) is $C_0=\dfrac{2^s\big(2(1+s)\big)^{(1+s)/2}}{1+s}$, and the
constant threshold floor is
$\gamma_0=\max\big\{1,\ \tfrac{32}{3M\theta}\,\sigma\big\}$.
The shared schedule constants are $C_1=\dfrac{F_0}{c}$ and $C_2=\dfrac{C_{bias}(\theta)\,\sigma^{2\alpha}}{2c}$.
For regime~(i) we additionally use the variance-bounding constant $K$ (independent of $\gamma,T$),
defined for $\gamma\ge1$ by
$2^s\theta^{1+s}+2^s\sigma^{1+s}\gamma^{-(1+s)}+C_{tail}(\theta)\sigma^\alpha\gamma^{-\alpha}\le K$,
together with $C_3=\dfrac{LK}{c(1+s)}$ and $C_4=2\,(C_1^s C_3)^{\frac1{1+s}}$.
For regime~(ii), the variance term scales as $\gamma^{1+s-\alpha}$ rather than $\gamma^{1+s}$, which
changes the balance of exponents; the corresponding constants are
\[
C_3''=\frac{L\,C_{var}(\theta)\,\sigma^\alpha}{c(1+s)},\qquad C_4''=2\,(C_1^s C_3'')^{\frac1{1+s}}.
\]
(One can check directly that balancing $C_1/(\eta T)$ against $C_3''\eta^s\gamma^{1+s-\alpha}$ and then
against $C_2\gamma^{-2(\alpha-1)}$ yields the exponent $\tfrac{s}{(\alpha-1)+s(2\alpha-1)}$ on $T$ for
$\gamma$ above, consistent with the rate in part~(ii); the Young's-inequality remainder term
$\eta^{2s/(1-s)}$ is dominated by the other three terms at this schedule, exactly as in the proof of
part~(ii) below.)
\end{remark}


\subsection{Auxiliary Lemmas}\label{subsec:gclip-lemmas}

The lemmas below are grouped by the case they serve. Lemmas~\ref{lem:gclip-prob}
and~\ref{lem:gclip-bv} are used in the small-gradient case; Lemmas~\ref{lem:gclip-corr} and \ref{lem:gclip-c2mom}
in the large-gradient case. All are common to both regimes except the moment bound of
Lemma~\ref{lem:gclip-bv}, which is stated once per regime.

\medskip
\noindent\paragraph{Lemmas for the small-gradient case ($\|\nabla f(\vx)\|\le\theta\gamma$).}

\begin{lemma}[Conditional probability of the clipped region]\label{lem:gclip-prob}
Let Assumption~\ref{ass:noise-s} hold with $\alpha\in(1,2]$. If $\|\nabla f(\vx)\|\le\theta\gamma$ for a
fixed $\theta\in(0,1)$ and $B=\{\|\vg\|>\gamma\}$, then
$\Pr(B\mid\vx)\le\sigma^\alpha/((1-\theta)^\alpha\gamma^\alpha)$.
\end{lemma}

The proof is provided in \autoref{subsec:gclip-prob-sec}.

\begin{lemma}[Conditional bias and moment for $\|\nabla f(\vx)\|\le\theta\gamma$]\label{lem:gclip-bv}
Let Assumption~\ref{ass:noise-s} hold, and suppose $\|\nabla f(\vx)\|\le\theta\gamma$ for a fixed
$\theta\in(0,1)$. Let $B(\vx):=\E[\hat{\vg}(\vx)\mid\vx]-\nabla f(\vx)$ be the conditional bias. Then
the squared bias obeys, \textbf{in both regimes},
\begin{equation}\label{eq:gclip-bias}
\|B(\vx)\|^2 \;\le\; C_{bias}(\theta)\,\sigma^{2\alpha}\gamma^{2(1-\alpha)},
\end{equation}
and the conditional $(1+s)$-th moment satisfies:
\begin{align}
\textbf{(i)}\ \ 1+s\le\alpha:\quad
&\E\big[\|\hat{\vg}(\vx)\|^{1+s}\mid\vx\big]\le
2^s\|\nabla f(\vx)\|^{1+s}+2^s\sigma^{1+s}+C_{tail}(\theta)\,\sigma^\alpha\gamma^{1+s-\alpha},
\label{eq:gclip-mom-i}\\[2pt]
\textbf{(ii)}\ \ 1+s>\alpha:\quad
&\E\big[\|\hat{\vg}(\vx)\|^{1+s}\mid\vx\big]\le
2^s\|\nabla f(\vx)\|^{1+s}+C_{var}(\theta)\,\sigma^\alpha\gamma^{1+s-\alpha}.
\label{eq:gclip-mom-ii}
\end{align}
Where $C_{bias}(\theta), C_{tail}(\theta), C_{var}(\theta)$ are as defined in equation \ref{eq:C_const}.
\end{lemma}

The proof is provided in \autoref{subsec:gclip-bv-sec}.

\medskip
\noindent\paragraph{Lemmas for the large-gradient case ($\|\nabla f(\vx)\|>\theta\gamma$).}




\begin{lemma}[Inner-product lower bound for the large-gradient case]\label{lem:gclip-corr}
Let Assumption~\ref{ass:noise-s} hold, let $M=\min\{\tfrac{15}{16}\theta,\tfrac13\}$, and suppose
$\|\nabla f(\vx)\|>\theta\gamma$ with $\gamma\ge\gamma_0$ (Remark~\ref{rem:gclip-const}). Then, in
\textbf{both regimes},
\begin{equation}\label{eq:gclip-corr}
\E\big[\langle\nabla f(\vx),\hat{\vg}(\vx)\rangle\mid\vx\big]\;\ge\;\frac{5M}{8}\,\gamma\,\|\nabla f(\vx)\|.
\end{equation}
\end{lemma}

The proof is provided in \autoref{subsec:gclip-corr-sec}.

\begin{lemma}[Moment bound for the large-gradient case]\label{lem:gclip-c2mom}
In \textbf{both regimes}, $\E[\|\hat{\vg}(\vx)\|^{1+s}\mid\vx]\le\gamma^{1+s}$.
\end{lemma}

The proof is provided in \autoref{subsec:gclip-c2mom-sec}.

\section{Conclusion and Future Directions}
\label{chap:conclusion}
This work established first-of-its-kind convergence guarantees for stochastic gradient methods
under the simultaneous presence of $(L,s)$-H\"{o}lder smoothness and
heavy-tailed gradient noise, for standard SGD, the $\delta$-GClip 
algorithm of~\cite{tucat2025regularized} and the standard GClip. Several natural directions remain open as next steps.

To the best of our knowledge there are no known lowerbounds to the rate of convergence of stochastic algorithms for $(L,s)$-H\''{o}lder smooth functions and particularly so, in the presence of heavy-tailed noise. So firstly, we envisage finding this bound as the next step from here. 


One particularly important open direction is the development of
high-probability convergence guarantees under $(L,s)$-H\"{o}lder smoothness.
Recent work by~\cite{madden2024high} studied SGD for nonconvex
optimization under norm sub-Weibull noise --- a broad class of heavy-tailed
distributions parametrized by a tail index $\theta \ge 1/2$. Specifically,
the stochastic gradient noise $e_t = \nabla f(x_t) - g_t$ is assumed to satisfy
$
\mathbb{E}\!\left[
\exp\!\left(
\left(
\frac{\|e_t\|}{K}
\right)^{\!1/\theta}
\right)
\,\Big|\,
x_t
\right]
\le 2,
$
for some scale parameter $K>0$. Under standard $L$-smoothness and a
diminishing step size schedule $\eta_t = c/\sqrt{t+1}$ with $c \le 1/L$,
they prove that, with probability at least $1-\delta$,
$
\frac{1}{\sqrt{T}}
\sum_{t=0}^{T-1}
\frac{1}{\sqrt{t+1}}
\|\nabla f(x_t)\|^2
=
\mathcal{O}\!\left(
\frac{
\log(T)\,
\log(1/\delta)^{2\theta}
}{
\sqrt{T}
}
\right).
$
In the sub-Gaussian case $\theta = 1/2$, this recovers the classical
$\mathcal{O}(1/\sqrt{T})$ convergence behavior up to logarithmic factors.

Motivated by this result, a natural question is whether analogous
high-probability convergence guarantees can be established for SGD or
$\delta$-GClip or GClip when the smoothness assumption is weakened from Lipschitz
continuity to $(L,s)$-H\"{o}lder continuity of the gradient. Extending
martingale concentration arguments to the H\"{o}lder setting appears
nontrivial, since the descent inequalities involve fractional-order terms
that do not arise in the classical smooth case and the opportunity to understanding these technicalities make this an attractive direction for the future.
\clearpage
\bibliography{tmlr}
\bibliographystyle{tmlr}
\clearpage  
\appendix
\section{Literature Review}\label{sec:lit_holder}
\subsection{Review of Convergence Guarantees for Noisy-GD and SGD for Holder Functions} 
In the following, we review three representative works that study SGD and its variants under $s$-H\"older continuity, which provide the closest available context for our investigation.

Departing from classical analyses that rely on global Lipschitz continuity or bounded variance, \citet{patel2021stochastic} study the global behavior of stochastic gradient descent for nonconvex optimization under weak assumptions on both the objective and the noise model. The authors assume that the objective has only a locally $s$-H\"older continuous gradient and that the stochastic gradients satisfy a general second-moment bound $\mathbb{E}\|\nabla f(\vtheta,X)\|^2 \le G(\vtheta)$. Under these conditions, they show that SGD iterates either diverge to infinity or converge to a stationary point with probability one. Moreover, when the iterates remain bounded, the gradient norm satisfies $\|\nabla F(\vtheta_k)\| \to 0$ almost surely. Under assumptions of global s-H\"older continuity and stronger expected gradient, they further show that $\|\nabla F(\vtheta_k)\| \to 0$ both almost surely and in expectation.
We note that the results in \citet{patel2021stochastic} don't come with a rate of convergence.

Under an $s$-H\"{o}lder continuity assumption that also accommodates sub-gradients, thereby extending to non-differentiable losses such as the hinge and absolute losses, \citet{wang2021differentially} develop a theory of differentially private SGD (DP-SGD) for stochastic convex optimization. Their analysis establishes a precise trade-off between smoothness and computational efficiency. Notably, they show that $s \ge 1/2$ is already sufficient to achieve an excess population risk of $\mathcal{O}\big(\frac{\sqrt{d \log(1/\delta)}}{n\epsilon} + \frac{1}{\sqrt{n}}\big)$ with linear gradient complexity $\mathcal{O}(n)$. This upper bound perfectly matches the known information-theoretic lower bounds for differentially private stochastic convex optimization \citep{bassily2019private}, demonstrating that full Lipschitz smoothness is not strictly necessary to attain statistically optimal learning rates. 



Building upon relaxed smoothness assumptions, \citet{wang2023holder} investigate the convergence of stochastic gradient methods with momentum, including SGD, Heavy Ball, and Nesterov’s accelerated gradient, for nonconvex optimization. Departing from classical analyses the authors assume that the objective function has an $s$-H\"older continuous gradient. Using a stochastic unified momentum (SUM) formulation, they analyze these algorithms simultaneously and show that the expected gradient satisfies $\sum_{t=0}^\infty \gamma_t \mathbb{E}\|\nabla f(\vx_t)\|^2 < \infty$, implying convergence to stationary points in expectation, with rate $\mathbb{E}\|\nabla f(\vx_R)\|^2 = O(T^{-s/(1+s)})$. 

\begin{table}[htbp!]
\centering
\caption{Stochastic optimization under $s$-H\"older smoothness.}
\label{tab:holder_sgd}
\renewcommand{\arraystretch}{1.5}
\begin{tabularx}{\linewidth}{
>{\raggedright\arraybackslash}p{0.18\linewidth}
>{\raggedright\arraybackslash}p{0.18\linewidth}
>{\raggedright\arraybackslash}X
>{\raggedright\arraybackslash}p{0.28\linewidth}
}
\toprule
\textbf{Paper} 
& \textbf{Algorithm} 
& \textbf{Assumptions} 
& \textbf{Guarantee and Rate} \\
\midrule

\citep{wang2021differentially} \newline {\footnotesize Thm. 12} 
& DP-SGD 
& $s$-H\"older smooth gradient; convex loss; $(\varepsilon,\delta)$-DP 
& Excess risk: 
$O\!\left(\frac{\sqrt{d\log(1/\delta)}}{n\varepsilon} + \frac{1}{\sqrt{n}}\right)$; 
gradient complexity $O(n^{\frac{2-s}{1+s}} + n)$ \\

\midrule

\citep{wang2023holder} \newline {\footnotesize Thm. 3.1} 
& SGD, Heavy Ball, NAG 
& $s$-H\"older smooth gradient; nonconvex 
& $\mathbb{E}\|\nabla f(\vx_R)\|^2 = O(T^{-s/(1+s)})$; 
a.s. convergence to stationary points \\

\bottomrule
\end{tabularx}
\end{table}

\subsection{Review of Convergence Guarantees for Noisy-GD and SGD with Heavy-Tailed Noise}




The role of heavy-tailed gradient noise in explaining why adaptive methods outperform SGD in settings such as attention models has been studied empirically and theoretically by \citet{suvrit1912}. The authors show that when stochastic gradients satisfy a bounded $\alpha$-moment condition with $1 < \alpha \le 2$, standard SGD can fail to converge due to the influence of large, rare gradients, whereas \emph{clipped gradient methods (GClip)} remain stable. They prove that for smooth nonconvex objectives, GClip achieves a stationarity rate of $ O\!\left(T^{-\frac{\alpha-1}{3\alpha-2}}\right)$, and for strongly convex functions, an optimal convergence rate of $O\!\left(T^{-\frac{2(\alpha-1)}{\alpha}}\right)$, matching corresponding lower bounds and thus they establish the optimality of clipping-based methods under heavy-tailed noise. Additionally, the paper demonstrates experimentally that gradient noise in attention models is inherently heavy-tailed, thereby linking the structure of noise to optimizer performance.


A complementary relaxation of smoothness, orthogonal to the H\"older-continuity relaxation studied in this work, was proposed by \citet{zhang2020whygradient}, who introduce the $(L_0, L_1)$-smoothness condition, $\|\nabla^2 f(\vx)\| \le L_0 + L_1\|\nabla f(\vx)\|$, motivated by the empirical observation that the local smoothness constant along deep network training trajectories grows with the gradient norm. Under this condition, they show gradient clipping and normalized gradient descent converge arbitrarily faster than fixed-step-size gradient descent as $L_1$ grows. Unlike $(L,s)$-H\"older smoothness, which weakens the global growth rate of the gradient's modulus of continuity, $(L_0, L_1)$-smoothness lets the \emph{local} Lipschitz constant scale with the gradient norm, so the two relaxations are not nested; notably, $(L_0, L_1)$-smoothness underlies the generalized-smoothness assumption used by \citet{liu2025nonconvex}, discussed below.

Under a potentially infinite-variance noise model assuming only a finite $\alpha$-th moment for $\alpha \in [1,2)$ \citet{wang2021convergence}  study SGD and identify a curvature condition on the Hessian, termed \emph{uniform $\alpha$-positive definiteness}.
Under a diminishing step-size $\gamma_t \asymp t^{-\rho}$, they establish convergence rates of $O(t^{-\rho(\alpha-1)})$, showing explicit dependence on the tail index. They further prove that Polyak–Ruppert averaging converges to a multivariate $\beta$-stable distribution (where $\beta \in (1,2]$ denotes the stability index) rather than a Gaussian limit.

Among the earliest works to establish high-probability convergence guarantees for clipping under heavy-tailed noise, \citet{gorbunov2020stochastic} propose clipped-SSTM, an \emph{accelerated} stochastic first-order method for smooth convex optimization with heavy-tailed noise in the stochastic gradients. They show that an $\varepsilon$-solution is reached with probability $1-\beta$ after $O\big(\sqrt{LR_0^2/\varepsilon}\,\ln(LR_0^2/\varepsilon\beta)\big)$ iterations, matching the optimal light-tailed accelerated rate up to logs, and extend the technique to give the first non-trivial high-probability bound for non-accelerated clipped-SGD under heavy-tailed noise. This line directly precedes and motivates \citet{nguyen2023clippedsgd}, discussed next.

High-probability convergence guarantees for clipped gradient methods under heavy-tailed noise, with bounded $\alpha$-th moments for $1 < \alpha \le 2$, are the focus of \citet{nguyen2023clippedsgd}.
 The authors establish optimal high-probability convergence rates for clipped stochastic mirror descent in the convex setting and clipped stochastic gradient descent in the nonconvex setting, achieving rates of $O(T^{\frac{1-\alpha}{\alpha}})$ and $O(T^{\frac{2-2\alpha}{3\alpha-2}})$ respectively, matching known lower bounds. Their approach also allows for time-varying step sizes and clipping parameters when the time horizon is unknown.

Relatedly, \citet{madden2024high} establish high-probability guarantees for SGD under sub-Weibull noise. 


Using a general nonlinear framework  encompassing sign, normalization, and clipping transformations applied to the stochastic gradient \citet{Armacki2025NonlinearSGD} study high-probability convergence of SGD in a streaming setting under heavy-tailed noise. They apply bounded nonlinear transformations such as sign, normalization, and clipping to stochastic gradients. The bounded nonlinear transformation yields an effective noise sequence that is sub-Gaussian, even when the original noise possesses only a finite first moment, which allows the use of moment-generating function bounds to derive high-probability guarantees. Under Lipschitz smoothness and assuming the noise is zero-mean, i.i.d., integrable, with a symmetric density positive around zero, they establish high-probability convergence guarantees. In the nonconvex setting, they obtain a convergence rate arbitrarily close to $O(t^{-1/4})$, i.e., $O(t^{-1/4+\epsilon})$ for any $\epsilon>0$. In the strongly convex case, they establish a rate of $O(t^{-1/2+\epsilon})$ for the squared error of the averaged iterates. 


Contrary to the common intuition that adaptivity confers inherent robustness, \citet{Chezhegov2025ClipAdaptive} show that Adam and AdaGrad can exhibit poor high-probability convergence under heavy-tailed noise with bounded $\alpha$-th moments and iteration complexity depending polynomially on the confidence level $\delta$, due to instability in the adaptive scaling term under large gradient outliers. To address this, they propose clipped variants of Adam-Norm and AdaGrad-Norm, where clipping is applied both to the gradients and the scaling factors. Under standard smoothness assumptions, these methods achieve high-probability convergence with polylogarithmic dependence on $\delta$, requiring $\tilde{O}\!\left(\max\left\{\varepsilon^{-1},\, \varepsilon^{-\frac{\alpha}{\alpha-1}}\right\}\right)$ iterations in the convex setting and $\tilde{O}\!\left(\varepsilon^{-\frac{3\alpha-2}{2\alpha-1}}\right)$ iterations in the nonconvex case, corresponding to the leading-order term in the complexity bound and matching clipped SGD up to logarithmic factors.

The respective roles of gradient clipping and normalization under heavy-tailed noise  with finite $\alpha$-th moments for $\alpha \in (1,2]$ are revisited by \citet{SunLiuYuan2025GNClip} , who show that normalization coupled with momentum alone suffices for convergence under individual smoothness, achieving a leading-order rate of $O\!\left(\sigma^{\frac{2\alpha-2}{3\alpha-2}} T^{-\frac{\alpha-1}{3\alpha-2}}\right)$, matching known lower bounds. Combining normalization with clipping further improves convergence under global smoothness and removes logarithmic factors present in prior analyses. The results extend to variance-reduced methods, yielding a leading-order rate of $O\!\left(\sigma^{\frac{\alpha}{2\alpha-1}} T^{-\frac{\alpha-1}{2\alpha-1}}\right)$, and to accelerated variants under second-order smoothness, with rate $O\!\left(T^{-\frac{2\alpha-2}{4\alpha-1}}\right)$.


Similarly challenging the necessity of clipping, \citet{liu2025nonconvex} explore nonconvex stochastic optimization under a generalized heavy-tailed moment condition of the form $\mathbb{E}[\|\vxi_t\|^\alpha \mid \mathcal{F}_{t-1}] \le \sigma_0^\alpha + \sigma_1^\alpha \|\nabla F(\vx_t)\|^\alpha$ for some $\alpha \in (1,2]$, extending the standard finite $\alpha$-th moment assumption. This work showed that clipping is not necessary by analyzing the Batched Normalized Stochastic Gradient Descent with Momentum (NSGDM) algorithm, which normalizes the momentum direction at each step. Under a generalized smoothness condition of the form $\|\nabla F(\vx)-\nabla F(\vy)\| \le (L_0 + L_1 \|\nabla F(\vx)\|)\|\vx-\vy\|$, the authors prove that NSGDM achieves the same optimal rate $O(T^{\frac{1-\alpha}{3\alpha-2}})$ in expectation, providing the first result to attain optimal convergence without gradient clipping. The paper also considers the practically relevant scenario where the tail index $\alpha$ is unknown and proves a convergence rate of $O(T^{\frac{1-\alpha}{2\alpha}})$, using parameter choices independent of $\alpha$.

Building on the observation that training dynamics tend to operate near the 
\emph{edge of stability} where the leading Hessian eigenvalue $\lambda_{\max}$ 
rises before plateauing near $2/\eta$ \citet{gong2025adaptive} introduce 
Adaptive Heavy-Tailed Stochastic Gradient Descent (AHTSGD), which injects 
L\'evy $\alpha$-stable noise\footnote{The update incorporates L\'evy 
$\alpha$-stable noise $\vxi_t \sim \mathcal{S}_\alpha(0, 1, 0)$, a symmetric 
heavy-tailed distribution parameterized by a tail index $\alpha \in (0, 2]$, 
where $\alpha = 2$ recovers Gaussian noise and smaller values of $\alpha$ 
yield heavier tails with higher probability of large jumps.} into SGD and 
dynamically adapts the tail index $\alpha_t \in [1, 2]$ based on the evolving 
sharpness of the loss landscape. The method injects heavier-tailed noise 
($\alpha_t < 2$) during early training to promote exploration and transitions 
toward Gaussian noise ($\alpha_t \to 2$) as curvature stabilizes. In 
practice, sharpness is estimated using a log-transformed trace of the Hessian via 
Hutchinson's method, rather than the leading eigenvalue directly. The paper presents 
(unproven) theoretical bounds suggesting that AHTSGD achieves tighter expected 
suboptimality than SGD in sharp regions, of the form $\mathbb{E}[f(\vtheta_t) - f(\vtheta^*)] \leq \frac{L\|\vtheta_0 - \vtheta^*\|^2}{2t} + \frac{\eta L \sigma^2}{2} \left(\frac{\lambda_{\max}}{2/\eta}\right)^{2-\alpha_t}$,
and additionally conjectures bounds on the expected squared distance to the optimum 
under an exponential annealing schedule $\alpha_t = 2 - e^{-kt}$, suggesting 
accelerated early-stage convergence when $\alpha_t < 2$. However, these results are 
not formally proved in the paper. Empirical evaluations on MNIST, SVHN, and CIFAR-10 
demonstrate improved early-epoch convergence over SGD, robustness to poor 
initialization including zero initialization where SGD fails to converge, and stable 
performance across a range of learning rates, with the most pronounced gains observed 
on the noisier SVHN dataset.


\begin{table}[htbp!]
\centering
\caption{Provable convergence of noisy GD/SGD under heavy-tailed noise (with clipping).}
\label{tab:ht_sgd_clip}
\setlength{\tabcolsep}{4pt}
\renewcommand{\arraystretch}{1.5} 
\begin{tabularx}{\linewidth}{>{\raggedright\arraybackslash}p{0.18\linewidth} >{\raggedright\arraybackslash}p{0.18\linewidth} >{\raggedright\arraybackslash}X >{\raggedright\arraybackslash}p{0.28\linewidth}}
\toprule
\textbf{Paper} 
& \textbf{Algorithm} 
& \textbf{Assumptions} 
& \textbf{Guarantee and Rate} \\
\midrule
 \citep{suvrit1912}\newline {\footnotesize Thm.~2}
& Clipped SGD
& Smooth nonconvex, \newline heavy-tailed noise ($1  < \alpha \le 2$)
& Expected stationarity: $\mathbb{E}\|\nabla f(\vx_T)\|
= O\!\left(T^{-\frac{\alpha-1}{3\alpha-2}}\right)$ \\

\midrule
\citep{gorbunov2020stochastic}\newline {\footnotesize Thm.~2.1}
& Clipped-SSTM \newline (accelerated) \newline (HP analysis)
& Smooth convex; \newline heavy-tailed noise
& High-probability: $f(\vy^N) - f(\vx^*) = \widetilde{O}(LR_0^2/N^2)$ w.p.\ $\ge 1-\beta$ \\

\midrule
\citep{gorbunov2020stochastic}\newline {\footnotesize Thm.~3.1}
& Clipped SGD \newline (HP analysis)
& Smooth convex; \newline heavy-tailed noise
& High-probability: $f(\bar{\vx}^N) - f(\vx^*) = \widetilde{O}(LR_0^2/N)$ w.p.\ $\ge 1-\beta$ \\

\midrule
\citep{nguyen2023clippedsgd}\newline {\footnotesize Thm.~5}
& Clipped SGD \newline (HP analysis)
& Smooth nonconvex, \newline heavy-tailed noise ($1  < \alpha \le 2$)
& High-probability stationarity: $\frac{1}{T}\sum_{t=1}^T \|\nabla f(\vx_t)\|^2 = O\!\left(T^{\frac{2-2\alpha}{3\alpha-2}}\right)$ w.p.\ $\ge 1-\delta$ \\

\midrule
\citep{nguyen2023clippedsgd}\newline{\footnotesize Thm.~3}
& Clipped SGD \newline (HP analysis)
& Smooth convex; heavy-tailed noise $(1 < \alpha \le 2)$
& High-probability optimality gap:
$\frac{1}{T}\sum_{t=1}^T (f(\vx_t)-f^*)
= O\!\left(T^{\frac{1-\alpha}{\alpha}}\right)$, w.p.\ $\ge 1-\delta$ \\
\midrule

\citep{Armacki2025NonlinearSGD}\newline {\footnotesize Thm.~1}
& Nonlinear SGD\newline (clip/sign/norm) \newline (HP analysis)
& Smooth nonconvex; symmetric heavy-tailed noise (no moment assumption)
& High-probability stationarity:
$\min_{t\le T}\|\nabla f(\vx_t)\|^2
= \widetilde{O}(T^{-1/4})$, w.p.\ $\ge 1-\delta$ \\

\midrule
\citep{Chezhegov2025ClipAdaptive}\newline {\footnotesize Thm.~3.1}
& Clip-AdamD/Clip-AdaGradD-Norm \newline (Averaged iterate)\newline (HP analysis)
& Smooth convex; heavy-tailed noise $(1 < \alpha \le 2)$
& High-probability optimality gap:
$f(\bar{\vx}_T)-f^* = \widetilde{O}\!\left(T^{\frac{1-\alpha}{\alpha}}\right)$, w.p.\ $\ge 1-\delta$ \newline
(where $\bar{x}_T$ is averaged iterate.)\\

\midrule
\citep{Chezhegov2025ClipAdaptive}\newline {\footnotesize Thm.~3.2}
& Clip-AdamD/Clip-M-AdaGradD-Norm \newline (HP analysis)
& Smooth nonconvex ; heavy-tailed noise $(1 < \alpha \le 2)$
& High-probability stationarity:
$\frac{1}{T}\sum_{t=1}^T \|\nabla f(\vx_t)\|^2
= \widetilde{O}\!\left(T^{-\frac{2\alpha-2}{3\alpha-2}}\right)$, w.p.\ $\ge 1-\delta$ \\

\midrule
\citep{SunLiuYuan2025GNClip}\newline {\footnotesize Thm.~3}
& Gradient normalization + clipping
& Smooth nonconvex ; heavy-tailed noise $(1 < \alpha \le 2)$
& Expected stationarity:
$\frac{1}{T}\sum_{t=1}^T \|\nabla f(\vx_t)\|
= \widetilde{O}\!\left(T^{-\frac{\alpha-1}{3\alpha-2}}\right)$ \\
\bottomrule
\end{tabularx}
\end{table}

\begin{table}[htbp!]
\centering
\caption{Provable convergence of noisy GD/SGD under heavy-tailed noise (no clipping).}
\label{tab:ht_sgd_no_clip}
\renewcommand{\arraystretch}{1.5} 
\begin{tabularx}{\linewidth}{>{\raggedright\arraybackslash}p{0.18\linewidth} >{\raggedright\arraybackslash}p{0.18\linewidth} >{\raggedright\arraybackslash}X >{\raggedright\arraybackslash}p{0.28\linewidth}}
\toprule
\textbf{Paper} & \textbf{Algorithm} & \textbf{Assumptions} & \textbf{Guarantee and Rate} \\ \midrule

\citep{wang2021convergence} \newline {\footnotesize Thm.~3} 
& Vanilla SGD \newline (Last Iterate) 
& { Strongly convex; \newline Uniformly $\alpha$-PD Hessian; \newline heavy-tailed noise ($1 < \alpha \le 2$); \newline $\gamma_t \asymp t^{-\rho}, \rho \in (0,1)$} 
& $\mathbb{E}\|\vx_t - \vx^*\|^\alpha = O(t^{-\rho(\alpha-1)})$ \\ \midrule

\citep{liu2025nonconvex}  \newline {\footnotesize Thm.~3.2} 
& Batched NSGDM 
& Nonconvex; generalized smoothness; heavy-tailed noise $(1 < \alpha \le 2)$ 
& $\frac{1}{T}\sum_{t=1}^T \mathbb{E}\|\nabla F(\vx_t)\| = O\!\left(T^{\frac{1-\alpha}{3\alpha-2}}\right)$ \\ 
\midrule

\citep{madden2024high} \newline {\footnotesize Thm.~15} 
& Vanilla SGD 
& nonconvex; heavy-tailed noise (sub-Weibull $\theta \ge 1/2$); $\eta_t \asymp 1/\sqrt{t}$ 
& $\frac{1}{\sqrt{T}}\sum_{t=0}^{T-1}\frac{1}{\sqrt{t+1}}\mathbb{E}\|\nabla f(\vx_t)\|^2 
= O\!\left(\frac{\log(T)\log(1/\delta)^{2\theta} + \gamma(\theta)\log(1/\delta)}{\sqrt{T}}\right)$ \\ \bottomrule
\end{tabularx}
\end{table}
\clearpage
\section{Proof of Theorem \ref{thm:SGD-proof}}\label{subsec:sgdproof1}
\begin{proof}[Proof of Theorem \ref{thm:SGD-proof}]
We begin with the fundamental inequality provided by $(L, s)$-H\"older smoothness (Definition ~\ref{def:holder-descent}). Substituting the SGD update rule, $\vx_{k+1} = \vx_k - \eta \vg_k$, and substituting $\vxi_k \coloneqq \vg_k -\nabla f(\vx_k)$, we obtain,
\begin{align*}
f(\vx_{k+1}) &\le f(\vx_k) - \eta \langle \nabla f(\vx_k), \nabla f(\vx_k) + \vxi_k \rangle + \frac{L}{1+s} \|-\eta (\nabla f(\vx_k) + \vxi_k)\|^{1+s} \\
&= f(\vx_k) - \eta \|\nabla f(\vx_k)\|^2 - \eta \langle \nabla f(\vx_k), \vxi_k \rangle + \frac{L\eta^{1+s}}{1+s} \|\nabla f(\vx_k) + \vxi_k\|^{1+s}
\end{align*}
We take expectation of the above conditioned on the current iterate $\vx_k$ and we recall by Assumption ~\ref{ass:unbiased} that the cross term evaluates to zero. 
Hence we have,
\begin{equation*}
\mathbb{E}[f(\vx_{k+1}) \mid \vx_k] \le f(\vx_k) - \eta \|\nabla f(\vx_k)\|^2 + \frac{L\eta^{1+s}}{1+s} \mathbb{E}[\|\nabla f(\vx_k) + \vxi_k\|^{1+s} \mid \vx_k]
\end{equation*}
We expand the conditional variance term using the algebraic inequality $\|a+b\|^p \le 2^{p-1}(\|a\|^p + \|b\|^p)$ valid for any $p \ge 1$,  
\begin{align*}
\|\nabla f(\vx_k) + \vxi_k\|^{1+s} &\le 2^{(1+s)-1}(\|\nabla f(\vx_k)\|^{1+s} + \|\vxi_k\|^{1+s}) = 2^s\|\nabla f(\vx_k)\|^{1+s} + 2^s\|\vxi_k\|^{1+s}
\end{align*}
Taking the expectation conditioned on $\vx_k$ gives,
\begin{equation*}
\mathbb{E}[\|\nabla f(\vx_k) + \vxi_k\|^{1+s} \mid \vx_k] \le 2^s\|\nabla f(\vx_k)\|^{1+s} + 2^s\mathbb{E}[\|\vxi_k\|^{1+s} \mid \vx_k]
\end{equation*}
To bound the noise moment, we apply Lyapunov's inequality (or monotonicity of $L^p$-norms), which states $(\mathbb{E}[|X|^q])^{\frac{1}{q}} \le (\mathbb{E}[|X|^p])^{\frac{1}{p}}$ for any $0 < q < p$ and a random variable $X$. Setting $X = \|\vxi_k\|$, $q = 1+s$, and $p = \alpha$ (where $1+s < \alpha$), we obtain:
\begin{align*}
\mathbb{E}[\|\vxi_k\|^{1+s} \mid \vx_k] = \left( (\mathbb{E}[\|\vxi_k\|^{1+s} \mid \vx_k])^{\frac{1}{1+s}} \right)^{1+s} 
\le \left( (\mathbb{E}[\|\vxi_k\|^\alpha \mid \vx_k])^{\frac{1}{\alpha}} \right)^{1+s}
= (\mathbb{E}[\|\vxi_k\|^\alpha \mid \vx_k])^{\frac{1+s}{\alpha}}
\end{align*}
Applying Assumption ~\ref{ass:noise-s} 

($\mathbb{E}[\|\vxi_k\|^\alpha \mid \vx_k] \le \sigma^\alpha$), we find, 
$(\mathbb{E}[\|\vxi_k\|^\alpha \mid \vx_k])^{\frac{1+s}{\alpha}} \le (\sigma^\alpha)^{\frac{1+s}{\alpha}} = \sigma^{1+s}
$.
Substituting this finite bound back into our previous descent equation gets us,
\begin{equation*}
\mathbb{E}[f(\vx_{k+1}) \mid \vx_k] \le f(\vx_k) - \eta \|\nabla f(\vx_k)\|^2 + \frac{L 2^s \eta^{1+s}}{1+s} \|\nabla f(\vx_k)\|^{1+s} + \frac{L 2^s \sigma^{1+s}}{1+s} \eta^{1+s}
\end{equation*}

Next, for $s\in(0,1)$ (the endpoint $s=1$ is treated in Remark~\ref{rem:sgd-s1}), we apply Young's inequality, $ab \le \frac{a^p}{p} + \frac{b^q}{q}$, with conjugates $p = \frac{2}{1+s}$ and $q = \frac{2}{1-s}$. (Note that $\frac{1}{p} + \frac{1}{q} = \frac{1+s}{2} + \frac{1-s}{2} = 1$). Corresponding to the above choices of $p$ and $q$ we use $a = \left( \frac{\eta}{1+s} \right)^{\frac{1+s}{2}} \|\nabla f(\vx_k)\|^{1+s}$ and $b = L 2^s \eta^{\frac{1+s}{2}} (1+s)^{-\frac{1-s}{2}}$ to get, 

\begin{equation*}
\frac{L 2^s \eta^{1+s}}{1+s} \|\nabla f(\vx_k)\|^{1+s} \le \frac{\eta}{2} \|\nabla f(\vx_k)\|^2 + C_s L^{\frac{2}{1-s}} \eta^{\frac{1+s}{1-s}}
\end{equation*}

where, $C_s = \frac{1-s}{2(1+s)} 2^{\frac{2s}{1-s}}$. 






Substituting the upper bound into the descent equation and rearranging gets us,



\begin{equation*}
\frac{\eta}{2} \|\nabla f(\vx_k)\|^2 \le f(\vx_k) - \mathbb{E}[f(\vx_{k+1}) \mid \vx_k] + C_s L^{\frac{2}{1-s}} \eta^{\frac{1+s}{1-s}} + \frac{L 2^s \sigma^{1+s}}{1+s} \eta^{1+s}
\end{equation*}
We now take the total expectation of both sides to get, 
\begin{equation*}
\frac{\eta}{2} \mathbb{E}[\|\nabla f(\vx_k)\|^2] \le \mathbb{E}[f(\vx_k)] - \mathbb{E}[f(\vx_{k+1})] + C_s L^{\frac{2}{1-s}} \eta^{\frac{1+s}{1-s}} + \frac{L 2^s \sigma^{1+s}}{1+s} \eta^{1+s}
\end{equation*}
Summing this from $k=1$ to $T$ yields a telescoping series for the function values,
\begin{equation*}
\frac{\eta}{2} \sum_{k=1}^T \mathbb{E}[\|\nabla f(\vx_k)\|^2] \le \sum_{k=1}^T \left( \mathbb{E}[f(\vx_k)] - \mathbb{E}[f(\vx_{k+1})] \right) + T \left( C_s L^{\frac{2}{1-s}} \eta^{\frac{1+s}{1-s}} + \frac{L 2^s \sigma^{1+s}}{1+s} \eta^{1+s} \right)
\end{equation*}
We simplify the telescoping sum as, $\sum_{k=1}^T \left( \mathbb{E}[f(\vx_k)] - \mathbb{E}[f(\vx_{k+1})] \right) = f(\vx_1) - \mathbb{E}[f(\vx_{T+1})] \le f(\vx_1) - f^*$. Further defining, $F_0 \coloneqq f(\vx_1) - f^*$ and dividing the entire inequality by $\frac{\eta T}{2}$ we get,


\begin{equation*}
\frac{1}{T} \sum_{k=1}^T \mathbb{E}[\|\nabla f(\vx_k)\|^2] \le \frac{2F_0}{\eta T} + 2 C_s L^{\frac{2}{1-s}} \eta^{\frac{2s}{1-s}} + \frac{L 2^{s+1} \sigma^{1+s}}{1+s} \eta^s
\end{equation*}
For $s \in (0,1)$, we have $\frac{2s}{1-s} > s$. Therefore, as $\eta \to 0$, the $\eta^{\frac{2s}{1-s}}$ term decays faster and is asymptotically dominated by the $\eta^s$ pure variance term. The bound simplifies asymptotically to:
\begin{equation*}
\frac{1}{T} \sum_{k=1}^T \mathbb{E}[\|\nabla f(\vx_k)\|^2] \le \mathcal{O}\left( \frac{1}{\eta T} \right) + \mathcal{O}\left( \eta^s \right)
\end{equation*}
Plugging the declared step size of $\eta = \mathcal{O}(T^{-\frac{1}{1+s}})$ into the above results in the rate of convergence that was claimed, 
\begin{align*}
\frac{1}{T} \sum_{k=1}^T \mathbb{E}[\|\nabla f(\vx_k)\|^2] = \mathcal{O}\left( \left( T^{-\frac{1}{1+s}} \right)^s \right) = \mathcal{O}\left( T^{-\frac{s}{1+s}} \right)
\end{align*}
\end{proof}
\section{Proofs of Lemmas Needed for Theorem \ref{thm:delta-gclip-s-3}}\label{app:lemma-delta}

\subsection{Proof of lemma \ref{lem:prob_bounds-3}}\label{subsec:probability-bounds-3}
\begin{proof}
For any $\vx \in B$, we have by definition $\|\vg\| > \gamma$. Using the reverse triangle inequality and the condition $\|\nabla f(\vx)\| \le \theta\gamma$, we bound the magnitude of the noise:
\[ \|\vxi\| = \|\vg - \nabla f\| \ge \|\vg\| - \|\nabla f\| > \gamma - \theta\gamma = (1-\theta)\gamma. \]
Applying Markov's inequality to this strictly positive boundary gives the probability bound for region $B$:
\begin{equation}
    \Pr(B) = \Pr(\vx \in B) \le \Pr\big(\|\vxi\| > (1-\theta)\gamma\big) \le \frac{\E[\|\vxi\|^\alpha]}{((1-\theta)\gamma)^\alpha} \le \frac{\sigma^\alpha}{(1-\theta)^\alpha \gamma^\alpha}.
\end{equation}

Similarly, for any $\vx \in C$, we have $\|\vg\| > \gamma/\delta$. The reverse triangle inequality yields:
\[ \|\vxi\| = \|\vg - \nabla f\| \ge \|\vg\| - \|\nabla f\| > \frac{\gamma}{\delta} - \theta\gamma = \gamma\left(\frac{1-\theta\delta}{\delta}\right). \]
Applying Markov's inequality yields the probability bound for region $C$:
\begin{equation}
    \Pr(C) = \Pr(\vx \in C) \le \Pr\left(\|\vxi\| > \gamma\frac{1-\theta\delta}{\delta}\right) \le \frac{\E[\|\vxi\|^\alpha]}{\left(\gamma\frac{1-\theta\delta}{\delta}\right)^\alpha} \le \left(\frac{\delta}{1-\theta\delta}\right)^\alpha \frac{\sigma^\alpha}{\gamma^\alpha}.
\end{equation}
\end{proof}

\subsection{Proof of lemma \ref{lem:clipped-bias-delta-s-3}}\label{subsec:case1-bounds-3}
\begin{proof}
Let $\hat \vg_\delta(\vx) := \min\Big\{1,\; \max\Big\{\delta,\;\frac{\gamma}{\|\vg(\vx)\|}\Big\}\Big\} \cdot \vg(\vx).$
We define the disjoint region $A=\{\|\vg\|\le \gamma\}$ alongside regions $B$ and $C$ defined in Lemma \ref{lem:prob_bounds-3}. The estimator simplifies piecewise to:
\[
\hat g_\delta \;=\;
\begin{cases}
\vg, & \vg \in A, \\
\displaystyle \gamma \,\frac{\vg}{\|\vg\|}, & \vg \in B, \\
\delta \vg, & \vg \in C.
\end{cases}
\]

Throughout the proof, we repeatedly use the elementary inequality for \(a,b\ge0\) and \(p\ge1\):
\begin{equation}\label{eq:ab-ineq-3}
(a+b)^p \le 2^{p-1}(a^p + b^p)
\end{equation}

\noindent\textbf{Moment Bound} \\
We decompose the expectation over the three regions \(A,B,C\):
\begin{equation}
\E\big[\|\hat{\vg}_\delta(\vx)\|^p\big] = \E\big[\|\hat{\vg}_\delta(\vx)\|^p\1_A\big] + \E\big[\|\hat{\vg}_\delta(\vx)\|^p\1_B\big] + \E\big[\|\hat{\vg}_\delta(\vx)\|^p\1_C\big].
\label{eq:ABC-split-s-3}
\end{equation}

\paragraph{Region A: \(\|\vg\|\le\gamma\)}
On \(A\), the estimator is unclipped ($\hat{\vg}_\delta=\vg=\nabla f + \vxi$).
\begin{align}
    \E\big[\|\hat{\vg}_\delta(\vx)\|^p\1_A\big] 
    &= \E\big[\|\nabla f(\vx) + \vxi\|^p\1_A\big] \nonumber \\
    &\le \E\big[2^{p-1}(\|\nabla f(\vx)\|^p + \|\vxi\|^p)\1_A\big] \nonumber \\
    &= 2^{p-1}\|\nabla f(\vx)\|^p \E[\1_A] + 2^{p-1}\E\big[\|\vxi\|^p\1_A\big] \nonumber \\
    &\le 2^{p-1}\|\nabla f(\vx)\|^p + 2^{p-1}\E\big[\|\vxi\|^p\big] \nonumber \\
    &\le 2^{p-1}\|\nabla f(\vx)\|^p + 2^{p-1}\big(\E[\|\vxi\|^\alpha]\big)^{p/\alpha} \nonumber \\
    &\le 2^{p-1}\|\nabla f(\vx)\|^p + 2^{p-1}\sigma^p. \label{eq:regionA-final-s-3}
\end{align}

In the sequence above, the first inequality applies the inequality \ref{eq:ab-ineq-3}.  We then bound the indicator function $\1_A \le 1$ and its expectation $\E[\1_A] \le 1$. Finally, we apply Hölder's inequality (valid since $p \le \alpha$) and substitute the noise moment bound from Assumption 1.

\paragraph{Region B: $\gamma < \|\vg\| \le \gamma/\delta$}
On $B$, the $\delta$-GClip estimator takes the form:
\[ \hat{\vg}_\delta(\vx) = \gamma \frac{\vg(\vx)}{\|\vg(\vx)\|}. \]
Taking the norm of this estimator,
\[ \|\hat{\vg}_{\delta}(\vx)\| = \left\| \gamma \frac{\vg(\vx)}{\|\vg(\vx)\|} \right\| = \gamma \frac{\|\vg(\vx)\|}{\|\vg(\vx)\|} = \gamma. \]
Raising this to the $p$-th power confirms that the estimator has an exact norm of $\|\hat{\vg}_{\delta}(\vx)\|^p = \gamma^p$.
Taking expectation:
\begin{align}
    \E\!\left[\|\hat{\vg}_\delta(\vx)\|^p \1_B\right] 
    &= \E[\gamma^p \1_B] \nonumber \\
    &= \gamma^p \Pr(B) \nonumber \\
    &\le \gamma^p \left( \frac{\sigma^\alpha}{(1-\theta)^\alpha \gamma^\alpha} \right) \nonumber \\
    &= \frac{1}{(1-\theta)^\alpha} \sigma^\alpha \gamma^{p-\alpha}. \label{eq:regionB-final-s-3}
\end{align} 

Here, we factor out the constant $\gamma^p$ and use the definition $\E[\1_B] = \Pr(B)$. The inequality directly substitutes the probability bound for region $B$ derived in Lemma \ref{lem:prob_bounds-3}, Eq. \eqref{eq:prob_B_lem}.

\paragraph{Region C: $\|\vg(\vx)\|>\gamma/\delta$}
On $C$, we have $\hat \vg_\delta(\vx) = \delta \vg(\vx)$.\\
Taking expectation we get,
\begin{align}
    \E\big[\|\hat{\vg}_\delta(\vx)\|^p \1_C\big] 
    &= \delta^p \E\big[\|\nabla f + \vxi\|^p \1_C\big] \nonumber \\
    &\le 2^{p-1}\delta^p \E\big[ (\|\nabla f\|^p + \|\vxi\|^p) \1_C \big] \nonumber \\
    &= 2^{p-1}\delta^p \|\nabla f\|^p \Pr(C) + 2^{p-1}\delta^p \E\big[\|\vxi\|^p \1_C\big] \nonumber \\
    &\le 2^{p-1}\delta^p (\theta\gamma)^p \Pr(C) + 2^{p-1}\delta^p \big(\E\|\vxi\|^\alpha\big)^{p/\alpha} \Pr(C)^{\frac{\alpha-p}{\alpha}} \nonumber \\
    &\le 2^{p-1}\delta^p \theta^p\gamma^p \left[ \left(\frac{\delta}{1-\theta\delta}\right)^\alpha \frac{\sigma^\alpha}{\gamma^\alpha} \right] + 2^{p-1}\delta^p \sigma^p \left[ \left(\frac{\delta}{1-\theta\delta}\right)^\alpha \frac{\sigma^\alpha}{\gamma^\alpha} \right]^{\frac{\alpha-p}{\alpha}} \nonumber \\
    &= 2^{p-1}\delta^p \sigma^\alpha \gamma^{p-\alpha} \left[ \theta^p \left(\frac{\delta}{1-\theta\delta}\right)^\alpha + \left(\frac{\delta}{1-\theta\delta}\right)^{\alpha-p} \right]. \label{eq:regionC-final-s-3}
\end{align}

Here, we first applied inequality \eqref{eq:ab-ineq-3} and the identity $\E[\1_C] = \Pr(C)$. We then bound the true gradient using the Case 1 condition $\|\nabla f\| \le \theta\gamma$ and apply Hölder's inequality to the noise term. Next, we substitute the probability bound for region $C$ from Lemma \ref{lem:prob_bounds-3}, Eq. \eqref{eq:prob_C_lem}, and Assumption 1. 

Substituting the individual regional bounds \eqref{eq:regionA-final-s-3}, \eqref{eq:regionB-final-s-3}, and \eqref{eq:regionC-final-s-3} back into the decomposition \eqref{eq:ABC-split-s-3}, we sum the terms together:
\begin{align}
    \E\big[\|\hat{\vg}_\delta(\vx)\|^p\big] 
    &\le \Big( 2^{p-1}\|\nabla f(\vx)\|^p + 2^{p-1}\sigma^p \Big) \nonumber \\
    &\quad + \left( \frac{1}{(1-\theta)^\alpha} \sigma^\alpha \gamma^{p-\alpha} \right) \nonumber \\
    &\quad + \left( 2^{p-1}\delta^p \sigma^\alpha \gamma^{p-\alpha} \left[ \theta^p \left(\frac{\delta}{1-\theta\delta}\right)^\alpha + \left(\frac{\delta}{1-\theta\delta}\right)^{\alpha-p} \right] \right).
\end{align}
By factoring out the common term $\sigma^\alpha \gamma^{p-\alpha}$ from the Region B and Region C bounds, we obtain the complete, fully expanded variance bound:
\begin{align}\label{eq:combined_p_bound_case1-3}
    \E\big[\|\hat{\vg}_\delta(\vx)\|^p\big] 
    &\le 2^{p-1}\|\nabla f(\vx)\|^p + 2^{p-1}\sigma^p \nonumber \\
    &\quad + \left( \frac{1}{(1-\theta)^\alpha} + 2^{p-1}\delta^p \left[ \theta^p\left(\frac{\delta}{1-\theta\delta}\right)^\alpha + \left(\frac{\delta}{1-\theta\delta}\right)^{\alpha-p} \right] \right) \sigma^\alpha \gamma^{p-\alpha}.
\end{align}
\bigskip
\noindent\textbf{Bias Bound} \\
Define the bias vector $B_\delta(\vx) := \E[\hat \vg_\delta(\vx)] - \nabla f(\vx) = \E\big[\hat{\vg}_{\delta}(\vx) - \vg(\vx)\big]$.
Taking the norm and moving it inside the expectation by Jensen's inequality, we decompose the error over the three disjoint regions $A$, $B$, and $C$:
\begin{align}
    \|B_\delta(\vx)\| &\le \E\big[\|\hat{\vg}_{\delta}(\vx) - \vg(\vx)\|\big] \nonumber \\
    &= \E\big[\|\hat{\vg}_{\delta}(\vx) - \vg(\vx)\| \1_A\big] + \E\big[\|\hat{\vg}_{\delta}(\vx) - \vg(\vx)\| \1_B\big] + \E\big[\|\hat{\vg}_{\delta}(\vx) - \vg(\vx)\| \1_C\big].
\end{align}

Because the estimator is unclipped in region $A$, we have $\hat \vg_\delta(\vx) = \vg(\vx)$, which means the first term evaluates to exactly zero. Thus, the error decomposes over just $B$ and $C$:
\begin{equation}\label{eq:Bias_final_new-s}
    \|B_\delta(\vx)\| \le \E\big[\|\vg(\vx) - \hat \vg_\delta(\vx)\| \1_B\big] + \E\big[\|\vg(\vx) - \hat \vg_\delta(\vx)\| \1_C\big].
\end{equation}

\paragraph{Region B: $\gamma < \|\vg\| \le \gamma/\delta$}
On $B$, the $\delta$-GClip estimator performs standard clipping, scaling the gradient to have an exact norm of $\gamma$:
\[ \hat{\vg}_\delta(\vx) = \gamma \frac{\vg(\vx)}{\|\vg(\vx)\|}. \]
Therefore,
\begin{align*}
    \|\vg - \hat \vg_\delta\| 
    &= \left\| \vg - \gamma \frac{\vg}{\|\vg\|} \right\| 
    = \left\| \left(1 - \frac{\gamma}{\|\vg\|}\right) \vg \right\| = \left| 1 - \frac{\gamma}{\|\vg\|} \right| \|\vg\|.
\end{align*}
In Region B, we know $\|\vg\| > \gamma$, which guarantees that the scalar coefficient $(1 - \gamma/\|\vg\|) > 0$. Removing the absolute value and distributing the norm gives:
\[ \|\vg - \hat \vg_\delta\| = \left( 1 - \frac{\gamma}{\|\vg\|} \right) \|\vg\| = \|\vg\| - \gamma < \|\vg\|. \]
\begin{align}
    \E\big[\|\vg - \hat \vg_\delta\| \1_B\big] 
    &\le \E\big[\|\vg\| \1_B\big] \nonumber \\
    &= \E\big[\|\nabla f + \vxi\| \1_B\big] \nonumber \\
    &\le \|\nabla f\| \Pr(B) + \E\big[\|\vxi\| \1_B\big] \nonumber \\
    &\le \theta\gamma \Pr(B) + \big(\E\|\vxi\|^\alpha\big)^{1/\alpha} \Pr(B)^{\frac{\alpha-1}{\alpha}} \nonumber \\
    &\le \theta\gamma \left( \frac{\sigma^\alpha}{(1-\theta)^\alpha \gamma^\alpha} \right) + \sigma \left( \frac{\sigma^\alpha}{(1-\theta)^\alpha \gamma^\alpha} \right)^{\frac{\alpha-1}{\alpha}} \nonumber \\
    &= \frac{\theta}{(1-\theta)^\alpha} \sigma^\alpha \gamma^{1-\alpha} + \frac{1}{(1-\theta)^{\alpha-1}} \sigma^\alpha \gamma^{1-\alpha} \nonumber \\
    &= \frac{1}{(1-\theta)^\alpha} \sigma^\alpha \gamma^{1-\alpha}. \label{eq:bias_B_unified-3}
\end{align}

The first equality substitutes $\vg = \nabla f + \vxi$. The subsequent inequality applies the triangle inequality. We then apply the Case 1 condition $\|\nabla f\| \le \theta\gamma$ and Hölder's inequality. Finally, we substitute the probability bound for region $B$ from Lemma \ref{lem:prob_bounds-3}, inequality \ref{eq:prob_B_lem}, along with Assumption 1.

\paragraph{Region C: $\|\vg(\vx)\|>\gamma/\delta$}
On $C$, the $\delta$-GClip estimator is :
\[ \hat{\vg}_\delta(\vx) = \delta \vg(\vx). \]
Therefore,
\[ \|\vg - \hat \vg_\delta\| = \|\vg - \delta\vg\| = \|(1-\delta)\vg\| = (1-\delta)\|\vg\|. \]

Taking expectation we get,
\begin{align}
    \E\big[\|\vg - \hat \vg_\delta\| \1_C\big] 
    &= (1-\delta) \E\big[\|\nabla f + \vxi\| \1_C\big] \nonumber \\
    &\le (1-\delta) \|\nabla f\| \Pr(C) + (1-\delta) \E\big[\|\vxi\| \1_C\big] \nonumber \\
    &\le (1-\delta) \theta\gamma \Pr(C) + (1-\delta) \big(\E\|\vxi\|^\alpha\big)^{1/\alpha} \Pr(C)^{\frac{\alpha-1}{\alpha}} \nonumber \\
    &\le (1-\delta) \theta\gamma \left[ \left(\frac{\delta}{1-\theta\delta}\right)^\alpha \frac{\sigma^\alpha}{\gamma^\alpha} \right] + (1-\delta) \sigma \left[ \left(\frac{\delta}{1-\theta\delta}\right)^\alpha \frac{\sigma^\alpha}{\gamma^\alpha} \right]^{\frac{\alpha-1}{\alpha}} \nonumber \\
    &= (1-\delta) \sigma^\alpha \gamma^{1-\alpha} \left[ \theta\left(\frac{\delta}{1-\theta\delta}\right)^\alpha + \left(\frac{\delta}{1-\theta\delta}\right)^{\alpha-1} \right]. \label{eq:final_bias_C-s-3}
\end{align}

We substitute $\vg = \nabla f + \vxi$, apply the triangle inequality. We then apply the Case 1 condition, Hölder's inequality, and substitute the probability bound for region $C$ from Lemma \ref{lem:prob_bounds-3}, inequality \eqref{eq:prob_C_lem}, and Assumption 1.

\paragraph{Combine}
Summing the bounds from inequality \ref{eq:bias_B_unified-3} and \ref{eq:final_bias_C-s-3} directly yields the final Bias Bound:
\begin{equation}\label{eq:final_bias_bound-s-3}
    \|B_\delta(\vx)\| \le \left( \frac{1}{(1-\theta)^\alpha} + (1-\delta) \left[ \theta\left(\frac{\delta}{1-\theta\delta}\right)^\alpha + \left(\frac{\delta}{1-\theta\delta}\right)^{\alpha-1} \right] \right) \sigma^\alpha \gamma^{1-\alpha}.
\end{equation}

Finally, to compute the Squared Bias Bound, we square inequality \ref{eq:final_bias_bound-s-3} and apply the elementary inequality $(a+b)^2 \le 2a^2 + 2b^2$:
\begin{equation}\label{eq:final_squared_bias-s-3}
    \|B_\delta(\vx)\|^2 \le 2 \left( \frac{1}{(1-\theta)^{2\alpha}} \right) \sigma^{2\alpha} \gamma^{2(1-\alpha)} + 2 (1-\delta)^2 \left[ \theta\left(\frac{\delta}{1-\theta\delta}\right)^\alpha + \left(\frac{\delta}{1-\theta\delta}\right)^{\alpha-1} \right]^2 \sigma^{2\alpha} \gamma^{2(1-\alpha)}.
\end{equation}
\end{proof}

\subsection{Proof of lemma \ref{lem:inner-gclip-case2-3}}\label{subsec:ip-bound-case2}

\begin{proof}

Throughout this proof we fix the current iterate $x$ and write
$\nabla f := \nabla f(x)$, $g := \nabla f + \xi$ for the stochastic
gradient, and $\xi := g - \nabla f$ for the zero-mean noise.
All conditional expectations are implicitly conditioned on $x$.

\paragraph{Step 1: Partition of the sample space.}

Depending on the magnitude of the stochastic gradient $g$, the
$\delta$-clipped estimator $\hat{g}_\delta$ takes three structurally
different forms. Accordingly, we partition the sample space into the
three disjoint regions
\[
A \;=\; \{\|g\| \le \gamma\},
\qquad
B \;=\; \{\gamma < \|g\| \le \tfrac{\gamma}{\delta}\},
\qquad
C \;=\; \{\|g\| > \tfrac{\gamma}{\delta}\},
\]
with explicit estimator forms on each region:
\begin{align*}
\text{on } A: &\quad \hat{g}_\delta = g = \nabla f + \xi
  \quad \text{(no clipping),}\\
\text{on } B: &\quad \hat{g}_\delta = \gamma \dfrac{g}{\|g\|}
  \quad \text{(norm clipping to } \gamma\text{),}\\
\text{on } C: &\quad \hat{g}_\delta = \delta g = \delta(\nabla f + \xi)
  \quad \text{(scaling down by } \delta\text{).}
\end{align*}

Since $A$, $B$, $C$ partition the sample space, the expected inner
product decomposes as
\begin{equation}\label{eq:decomp}
\E\big[\langle \nabla f, \hat{g}_\delta\rangle\big]
\;=\;
\underbrace{\E\big[\langle \nabla f, \hat{g}_\delta\rangle \mathbf{1}_A\big]}_{=:\,R_A}
\;+\;
\underbrace{\E\big[\langle \nabla f, \hat{g}_\delta\rangle \mathbf{1}_B\big]}_{=:\,R_B}
\;+\;
\underbrace{\E\big[\langle \nabla f, \hat{g}_\delta\rangle \mathbf{1}_C\big]}_{=:\,R_C}.
\end{equation}

We bound each term from below in turn. Before doing so, we introduce
a small-noise event that will let us handle the noise in region $A$
and part of region $C$ sharply.

\paragraph{Step 2: Small-noise event.}

Define the event
\begin{equation}\label{eq:E-def}
E \;:=\; \left\{\|\xi\| \le \frac{\|\nabla f\|}{16}\right\},
\end{equation}
which captures the regime where the noise is much smaller than the
gradient. Its complement $E^c$ is the heavy-noise event.

The significance of $E$ is that on $E$, the noise contributes at most
a $\frac{1}{16}$ fraction of $\|\nabla f\|^2$ to any inner product
involving $\nabla f$ and $\xi$:
\[
\text{on } E{:}\quad |\langle \nabla f, \xi\rangle|
\;\le\; \|\nabla f\|\,\|\xi\|
\;\le\; \frac{\|\nabla f\|^2}{16}.
\]

We will show later (Step~5) that $\Pr(E^c)$ is small when
$\|\nabla f\| > \theta\gamma$ and $\gamma$ is chosen large enough.

\paragraph{Step 3: Lower bound on $R_A$.}

On $A$ the estimator is unclipped, so $\hat{g}_\delta = \nabla f + \xi$.
Expanding the inner product:
\[
R_A
\;=\;
\E\big[\|\nabla f\|^2 \mathbf{1}_A\big]
\;+\;
\E\big[\langle \nabla f, \xi\rangle \mathbf{1}_A\big].
\]

We split the second expectation according to whether $E$ holds or not:
\[
\E\big[\langle \nabla f, \xi\rangle \mathbf{1}_A\big]
\;=\;
\E\big[\langle \nabla f, \xi\rangle \mathbf{1}_{A \cap E}\big]
\;+\;
\E\big[\langle \nabla f, \xi\rangle \mathbf{1}_{A \cap E^c}\big].
\]

\textit{Contribution from $A \cap E$.}
On this event $\|\xi\| \le \|\nabla f\|/16$, so by Cauchy--Schwarz,
\[
\langle \nabla f, \xi\rangle \ge -\|\nabla f\|\,\|\xi\| \ge -\frac{\|\nabla f\|^2}{16}.
\]
Hence
\[
\E\big[\langle \nabla f, \xi\rangle \mathbf{1}_{A \cap E}\big]
\;\ge\;
-\frac{\|\nabla f\|^2}{16}\,\Pr(A \cap E).
\]

\textit{Contribution from $A \cap E^c$.}
Using only Cauchy--Schwarz:
\[
\E\big[\langle \nabla f, \xi\rangle \mathbf{1}_{A \cap E^c}\big]
\;\ge\;
-\|\nabla f\|\,\E\big[\|\xi\|\,\mathbf{1}_{E^c}\big].
\]

Combining the two contributions:
\begin{align}
R_A
&\;\ge\;
\|\nabla f\|^2 \Pr(A)
- \frac{\|\nabla f\|^2}{16}\Pr(A \cap E)
- \|\nabla f\|\,\E\big[\|\xi\|\,\mathbf{1}_{E^c}\big] \notag\\
&\;\ge\;
\|\nabla f\|^2\!\left(\Pr(A) - \frac{1}{16}\Pr(A)\right)
- \|\nabla f\|\,\E\big[\|\xi\|\,\mathbf{1}_{E^c}\big] \notag\\
&\;=\;
\frac{15}{16}\,\|\nabla f\|^2\,\Pr(A)
- \|\nabla f\|\,\E\big[\|\xi\|\,\mathbf{1}_{E^c}\big].
\label{eq:RA-quad}
\end{align}

Since we are in the regime $\|\nabla f\| > \theta\gamma$, we have
$\|\nabla f\|^2 > \theta\gamma\|\nabla f\|$, so inequality \ref{eq:RA-quad} gives
\begin{equation}\label{eq:RA}
R_A
\;\ge\;
\frac{15}{16}\,\theta\gamma\,\Pr(A)\,\|\nabla f\|
\;-\;
\|\nabla f\|\,\E\big[\|\xi\|\,\mathbf{1}_{E^c}\big].
\end{equation}

\paragraph{Step 4: Lower bounds on $R_B$ and $R_C$.}

\medskip
\noindent\textbf{Region $B$.}
On $B$ the estimator clips the gradient to the sphere of radius
$\gamma$: $\hat{g}_\delta = \gamma\,g/\|g\|$. We apply the
Cutkosky--Mehta directional inequality, which states that for any
vectors $u$ and $v$ with $v = u + \xi$,
\[
\left\langle \frac{u}{\|u\|},\, \frac{v}{\|v\|}\right\rangle
\;\ge\;
\frac{1}{3} \;-\; \frac{8}{3}\,\frac{\|\xi\|}{\|u\|}.
\]
Applying this with $u = \nabla f$, $v = g$:
\[
\left\langle \nabla f,\, \frac{g}{\|g\|}\right\rangle
\;\ge\;
\frac{1}{3}\,\|\nabla f\| \;-\; \frac{8}{3}\,\|\xi\|.
\]
Multiplying through by $\gamma > 0$ and taking expectations over $B$:
\begin{equation}\label{eq:RB}
R_B
\;\ge\;
\frac{1}{3}\,\gamma\,\Pr(B)\,\|\nabla f\|
\;-\;
\frac{8}{3}\,\gamma\,\E\big[\|\xi\|\,\mathbf{1}_B\big].
\end{equation}

\medskip
\noindent\textbf{Region $C$, small noise ($C \cap E$).}
On $C$ the estimator scales: $\hat{g}_\delta = \delta g$. Because
$\|g\| > \gamma/\delta$ on $C$, we have $\delta\|g\| > \gamma$, so
the effective norm $\|\hat{g}_\delta\| = \delta\|g\| > \gamma$.

On $C \cap E$, the noise is small ($\|\xi\| \le \|\nabla f\|/16$), so
we again apply the directional inequality:
\[
\left\langle \nabla f,\, \frac{g}{\|g\|}\right\rangle
\;\ge\;
\frac{1}{3}\,\|\nabla f\| \;-\; \frac{8}{3}\,\|\xi\|.
\]
Moreover, on $C \cap E$ the small-noise bound $\|\xi\| \le \|\nabla f\|/16$ makes this
lower bound strictly positive:
\[
\frac{1}{3}\|\nabla f\| - \frac{8}{3}\|\xi\|
\;\ge\;
\frac{1}{3}\|\nabla f\| - \frac{8}{3}\cdot\frac{\|\nabla f\|}{16}
= \frac{1}{3}\|\nabla f\| - \frac{1}{6}\|\nabla f\|
= \frac{1}{6}\|\nabla f\| \;>\; 0.
\]
Since $\hat{g}_\delta = \delta g$ with $\delta\|g\| > \gamma$, and this lower bound is
nonnegative, replacing the factor $\delta\|g\|$ by the smaller $\gamma$ preserves the
inequality:
\[
\langle \nabla f, \hat{g}_\delta\rangle
= \delta\|g\|\left\langle \nabla f, \frac{g}{\|g\|}\right\rangle
\;\ge\;
\delta\|g\|\left(\frac{1}{3}\|\nabla f\| - \frac{8}{3}\|\xi\|\right)
\;\ge\;
\gamma\left(\frac{1}{3}\|\nabla f\| - \frac{8}{3}\|\xi\|\right).
\]
Taking expectations over $C \cap E$:
\begin{equation}\label{eq:RCE}
R_{C \cap E}
\;\ge\;
\frac{1}{3}\,\gamma\,\Pr(C \cap E)\,\|\nabla f\|
\;-\;
\frac{8}{3}\,\gamma\,\E\big[\|\xi\|\,\mathbf{1}_{C \cap E}\big].
\end{equation}

\medskip
\noindent\textbf{Region $C$, large noise ($C \cap E^c$).}
Here the noise is large and we use only a crude bound.
Since $\hat{g}_\delta = \delta(\nabla f + \xi)$, Cauchy--Schwarz gives
\[
\langle \nabla f, \hat{g}_\delta\rangle
= \delta\big(\|\nabla f\|^2 + \langle \nabla f, \xi\rangle\big)
\;\ge\;
-\delta\,\|\nabla f\|\,\|\xi\|
\;\ge\;
-\|\nabla f\|\,\|\xi\|,
\]
where in the last step we used $\delta \le 1$. Taking expectations:
\begin{equation}\label{eq:RCEc}
R_{C \cap E^c}
\;\ge\;
-\|\nabla f\|\,\E\big[\|\xi\|\,\mathbf{1}_{C \cap E^c}\big].
\end{equation}

\paragraph{Step 5: Combining all regions.}

Adding inequalities \ref{eq:RA}, \ref{eq:RB}, \ref{eq:RCE}, and
\ref{eq:RCEc}:
\begin{align}
\E\big[\langle \nabla f, \hat{g}_\delta\rangle\big]
&\;\ge\;
\gamma\|\nabla f\|
\Big(
\tfrac{15}{16}\theta\,\Pr(A)
+ \tfrac{1}{3}\,\Pr(B)
+ \tfrac{1}{3}\,\Pr(C \cap E)
\Big) \notag\\
&\quad\;
-\;\|\nabla f\|\,\E\big[\|\xi\|\,\mathbf{1}_{E^c}\big] \notag\\
&\quad\;
-\;\tfrac{8}{3}\,\gamma\,\E\big[\|\xi\|\,\mathbf{1}_{B \cup (C \cap E)}\big],
\label{eq:combined}
\end{align}
where we collected the $E^c$ noise terms from $R_A$ and $R_{C \cap E^c}$
into a single term, and the directional-inequality noise terms from
$R_B$ and $R_{C \cap E}$ into a single term. We now bound each of the
three negative terms on the right.

\paragraph{Step 6: Bounding $\Pr(E^c)$ and the probability sum.}

By Markov's inequality applied to $\|\xi\|^\alpha$:
\[
\Pr(E^c)
= \Pr\!\left(\|\xi\| > \frac{\|\nabla f\|}{16}\right)
\;\le\;
\frac{\E[\|\xi\|^\alpha]}{(\|\nabla f\|/16)^\alpha}
\;\le\;
\frac{16^\alpha\,\sigma^\alpha}{\|\nabla f\|^\alpha},
\]
where we used the moment bound $\E[\|\xi\|^\alpha] \le \sigma^\alpha$.
Since $\|\nabla f\| > \theta\gamma$, this gives
\begin{equation}\label{eq:Ec-bound}
\Pr(E^c)
\;\le\;
\frac{16^\alpha\,\sigma^\alpha}{(\theta\gamma)^\alpha}.
\end{equation}

Choosing $\gamma$ large enough so that
\begin{equation}\label{eq:gamma-cond1}
\gamma \;\ge\; \frac{16 \cdot 2^{1/\alpha}}{\theta}\,\sigma,
\end{equation}
ensures $\Pr(E^c) \le 1/2$, and therefore
\begin{equation}\label{eq:prob-sum}
\Pr(A) + \Pr(B) + \Pr(C \cap E)
\ge 1 - \Pr(E^c) \;\ge\; \tfrac{1}{2}.
\end{equation}

Now define
\begin{equation}\label{eq:M-def}
M \;:=\; \min\!\left\{\frac{15}{16}\,\theta,\;\frac{1}{3}\right\} \;>\; 0.
\end{equation}

Since $M$ lower-bounds each coefficient in the probability sum, we have
\begin{equation}\label{eq:prob-coeff}
\frac{15}{16}\theta\,\Pr(A) + \frac{1}{3}\,\Pr(B) + \frac{1}{3}\,\Pr(C \cap E)
\;\ge\;
M\,\big(\Pr(A) + \Pr(B) + \Pr(C \cap E)\big)
\;\ge\;
\frac{M}{2}.
\end{equation}

The first term in inequality \ref{eq:combined} is therefore at least
$\frac{M}{2}\,\gamma\|\nabla f\|$.

\paragraph{Step 7: Bounding $\|\nabla f\|\,\E[\|\xi\|\mathbf{1}_{E^c}]$.}

We bound the $E^c$ noise moment using the $\alpha$-th moment assumption.
Write
\[
\E\big[\|\xi\|\,\mathbf{1}_{E^c}\big]
= \E\!\left[\frac{\|\xi\|^\alpha}{\|\xi\|^{\alpha-1}}\,\mathbf{1}_{E^c}\right].
\]
On $E^c$, we have $\|\xi\| > \|\nabla f\|/16$, so
$\|\xi\|^{\alpha-1} > (\|\nabla f\|/16)^{\alpha-1}$.
Hence $1/\|\xi\|^{\alpha-1} < (16/\|\nabla f\|)^{\alpha-1}$, giving
\[
\E\big[\|\xi\|\,\mathbf{1}_{E^c}\big]
\;\le\;
\frac{16^{\alpha-1}}{\|\nabla f\|^{\alpha-1}}\,
\E\big[\|\xi\|^\alpha\,\mathbf{1}_{E^c}\big]
\;\le\;
\frac{16^{\alpha-1}\,\sigma^\alpha}{\|\nabla f\|^{\alpha-1}}.
\]

Using $\|\nabla f\| > \theta\gamma$:
\begin{equation}\label{eq:noise-Ec}
\|\nabla f\|\,\E\big[\|\xi\|\,\mathbf{1}_{E^c}\big]
\;\le\;
\frac{16^{\alpha-1}\,\sigma^\alpha}{\|\nabla f\|^{\alpha-1}}
\cdot \|\nabla f\|
= \frac{16^{\alpha-1}\,\sigma^\alpha}{\|\nabla f\|^{\alpha-2}}
\;\le\;
\frac{16^{\alpha-1}\,\sigma^\alpha}{(\theta\gamma)^{\alpha-1}}\,\|\nabla f\|.
\end{equation}

We want this to be at most $\frac{M}{4}\,\gamma\|\nabla f\|$. This requires
\[
\frac{16^{\alpha-1}\,\sigma^\alpha}{(\theta\gamma)^{\alpha-1}}
\;\le\; \frac{M}{4}\,\gamma,
\]
i.e.,
\[
\gamma^\alpha \;\ge\; \frac{4 \cdot 16^{\alpha-1}}{M\,\theta^{\alpha-1}}\,\sigma^\alpha,
\]
which is satisfied when
\begin{equation}\label{eq:gamma-cond2}
\gamma \;\ge\; \left(\frac{4 \cdot 16^{\alpha-1}}{M\,\theta^{\alpha-1}}\right)^{\!1/\alpha}\!\sigma.
\end{equation}

Under condition \eqref{eq:gamma-cond2}:
\begin{equation}\label{eq:noise-Ec-final}
\|\nabla f\|\,\E\big[\|\xi\|\,\mathbf{1}_{E^c}\big]
\;\le\;
\frac{M}{4}\,\gamma\|\nabla f\|.
\end{equation}

\paragraph{Step 8: Bounding $\frac{8}{3}\gamma\,\E[\|\xi\|\mathbf{1}_{B\cup(C\cap E)}]$.}

On $B \cup (C \cap E)$ the noise need not be small. We use the crude
bound $\E[\|\xi\|\,\mathbf{1}_{B \cup (C \cap E)}] \le \E[\|\xi\|]$.
By Jensen's inequality (since $\alpha \ge 1$):
\[
\E[\|\xi\|]
\;\le\;
\big(\E[\|\xi\|^\alpha]\big)^{1/\alpha}
\;\le\;
\sigma.
\]

Hence the noise contribution from $B \cup (C \cap E)$ satisfies
\begin{equation}\label{eq:noise-B}
\frac{8}{3}\,\gamma\,\E\big[\|\xi\|\,\mathbf{1}_{B \cup (C \cap E)}\big]
\;\le\;
\frac{8}{3}\,\gamma\,\sigma.
\end{equation}

We want this to be at most $\frac{M}{8}\,\gamma\|\nabla f\|$.
Since $\|\nabla f\| > \theta\gamma$, it suffices to require
\[
\frac{8}{3}\,\sigma
\;\le\;
\frac{M\,\theta}{8}\,\gamma,
\]
i.e.,
\begin{equation}\label{eq:gamma-cond3}
\gamma \;\ge\; \frac{64}{3\,M\,\theta}\,\sigma.
\end{equation}

Under condition \ref{eq:gamma-cond3}:
\begin{equation}\label{eq:noise-B-final}
\frac{8}{3}\,\gamma\,\E\big[\|\xi\|\,\mathbf{1}_{B \cup (C \cap E)}\big]
\;\le\;
\frac{M}{8}\,\gamma\|\nabla f\|.
\end{equation}

\paragraph{Step 9: Assembling the final bound.}

Substituting inequalities \ref{eq:prob-coeff}, \ref{eq:noise-Ec-final}, and
\ref{eq:noise-B-final} into inequality \ref{eq:combined}:
\begin{align}
\E\big[\langle \nabla f, \hat{g}_\delta\rangle\big]
&\;\ge\;
\frac{M}{2}\,\gamma\|\nabla f\|
\;-\;
\frac{M}{4}\,\gamma\|\nabla f\|
\;-\;
\frac{M}{8}\,\gamma\|\nabla f\| \notag\\[4pt]
&\;=\;
\left(\frac{1}{2} - \frac{1}{4} - \frac{1}{8}\right)M\,\gamma\|\nabla f\| \notag\\[4pt]
&\;=\;
\frac{1}{8}\,M\,\gamma\|\nabla f\|.
\label{eq:final}
\end{align}

Setting $c_0 := \frac{1}{8}M > 0$, we conclude that under the
gradient magnitude condition $\|\nabla f(x)\| > \theta\gamma$ and
the three clipping threshold conditions \ref{eq:gamma-cond1},
\ref{eq:gamma-cond2}, \ref{eq:gamma-cond3} on $\gamma$, the
following lower bound holds:
\[
\E\big[\langle \nabla f(x),\,\hat{g}_\delta\rangle \mid x\big]
\;\ge\;
c_0\,\gamma\,\|\nabla f(x)\|,
\]
which is the desired result.

\end{proof}

\subsection{Proof of lemma \ref{lem:moment-gclip-case2-3}}\label{subsec:var-bound-case2}
\begin{proof}
We produce a strict upper bound for the conditional $(1+s)$-th moment $\E\big[\|\hat{\vg}_\delta(\vx)\|^{1+s} \mid \vx\big]$. 
We define the centered noise $\vxi = \vg - \nabla f$ and partition the space into our standard disjoint clipping regions:
\[ A=\{\|\vg\|\le \gamma\},\qquad B=\{\gamma<\|\vg\|\le \gamma/\delta\},\qquad C=\{\|\vg\|>\gamma/\delta\}. \]
We decompose the expected moment into two primary contributions: the strictly bounded regions ($A \cup B$) and the $\delta$-scaled region ($C$).
\begin{equation}\label{eq:moment_split}
    \E\big[\|\hat{\vg}_\delta\|^{1+s} \mid \vx\big] = \E\big[\|\hat{\vg}_\delta\|^{1+s} \1_{A \cup B} \mid \vx\big] + \E\big[\|\hat{\vg}_\delta\|^{1+s} \1_C \mid \vx\big].
\end{equation}

\paragraph{Upper Bound for Regions A and B}
In Region A, the estimator is unclipped ($\hat{\vg}_\delta = \vg$), meaning $\|\hat{\vg}_\delta\| = \|\vg\| \le \gamma$.
In Region B, the estimator is clipped to the boundary ($\hat{\vg}_\delta = \gamma \frac{\vg}{\|\vg\|}$), meaning $\|\hat{\vg}_\delta\| = \gamma$.
Therefore, across the entire union $A \cup B$, the magnitude of the estimator is strictly bounded by $\gamma$. We evaluate this expectation:
\begin{align}
    \E\big[\|\hat{\vg}_\delta\|^{1+s} \1_{A \cup B} \mid \vx\big] 
    &\le \E\big[\gamma^{1+s} \1_{A \cup B} \mid \vx\big] \nonumber \\
    &= \gamma^{1+s} \Pr(A \cup B \mid \vx) \nonumber \\
    &\le \gamma^{1+s}. \label{eq:var_AB_bound}
\end{align}

\paragraph{Upper Bound for Region C}
On event $C$, the estimator applies pure linear scaling: $\hat{\vg}_\delta = \delta\vg = \delta(\nabla f + \vxi)$. 
Taking the $(1+s)$-th power and applying the elementary inequality $\|a+b\|^p \le 2^{p-1}(\|a\|^p + \|b\|^p)$ with $p = 1+s$, we obtain:
\begin{align}
    \E\big[\|\hat{\vg}_\delta\|^{1+s} \1_C \mid \vx\big] 
    &= \delta^{1+s} \E\big[\|\nabla f + \vxi\|^{1+s} \1_C \mid \vx\big] \nonumber \\
    &\le 2^s \delta^{1+s} \E\big[ \big(\|\nabla f\|^{1+s} + \|\vxi\|^{1+s}\big) \1_C \mid \vx \big] \nonumber \\
    &= 2^s \delta^{1+s} \|\nabla f\|^{1+s} \Pr(C \mid \vx) + 2^s \delta^{1+s} \E\big[\|\vxi\|^{1+s} \1_C \mid \vx\big].
\end{align}
We relax the probability $\Pr(C \mid \vx) \le 1$. By Lyapunov's inequality and Assumption 1, because $1+s \le \alpha$, the truncated moment is strictly bounded by the total noise moment: $\E[\|\vxi\|^{1+s} \1_C \mid \vx] \le \E[\|\vxi\|^{1+s} \mid \vx] \le \sigma^{1+s}$. Substituting these yields:
\begin{equation} \label{eq:var_C_bound}
    \E\big[\|\hat{\vg}_\delta\|^{1+s} \1_C \mid \vx\big] \le 2^s \delta^{1+s} \|\nabla f\|^{1+s} + 2^s \delta^{1+s} \sigma^{1+s}.
\end{equation}

\paragraph{Combine and Apply Case 2 Condition}
Summing the regional bounds \ref{eq:var_AB_bound} and \ref{eq:var_C_bound}, we assemble the total moment:
\begin{equation}
    \E\big[\|\hat{\vg}_\delta\|^{1+s} \mid \vx\big] \le \gamma^{1+s} + 2^s \delta^{1+s} \|\nabla f\|^{1+s} + 2^s \delta^{1+s} \sigma^{1+s}.
\end{equation}
Because we are specifically analyzing Case 2 ($\|\nabla f\| > \theta\gamma$), we can algebraically absorb the clipping threshold into the true gradient. By rearranging the condition, we have $\gamma < \frac{1}{\theta}\|\nabla f\|$. Raising this to the $(1+s)$-th power yields:
\begin{equation}
    \gamma^{1+s} < \frac{1}{\theta^{1+s}} \|\nabla f\|^{1+s}.
\end{equation}
Substituting this directly into our combined bound seamlessly consolidates the signal terms:
\begin{align}
    \E\big[\|\hat{\vg}_\delta\|^{1+s} \mid \vx\big] 
    &\le \frac{1}{\theta^{1+s}} \|\nabla f\|^{1+s} + 2^s \delta^{1+s} \|\nabla f\|^{1+s} + 2^s \delta^{1+s} \sigma^{1+s} \nonumber \\
    &= \left( \frac{1}{\theta^{1+s}} + 2^s \delta^{1+s} \right) \|\nabla f(\vx)\|^{1+s} + 2^s \delta^{1+s} \sigma^{1+s}.
\end{align}
\end{proof}

\section{Proof of Theorem~\ref{thm:gclip}}\label{subsec:gclip-proof}

\begin{proof}[Proof of Theorem~\ref{thm:gclip}]
We prove the two parts separately. Throughout a single step we write $\vx=\vx_{k-1}$,
$\eta=\eta_{k-1}$, $\hat{\vg}=\hat{\vg}(\vx_{k-1})$, and set $c=\tfrac{5M}{16}$.

\subsubsection*{Part (i): regime $1+s\le\alpha$}

\noindent\textbf{Descent inequality.}\; By the $(L,s)$-H\"older descent inequality
(Definition~\ref{def:holder-descent}) applied to $\vx_k=\vx_{k-1}-\eta\hat{\vg}$ together with
$\|\hat{\vg}\|\le\gamma$, taking conditional expectation,
\begin{equation}\label{eq:desc-i}
\E[f(\vx_k)\mid\vx_{k-1}]\le f(\vx_{k-1})
-\eta\big\langle\nabla f(\vx_{k-1}),\E[\hat{\vg}\mid\vx_{k-1}]\big\rangle
+\frac{L\eta^{1+s}}{1+s}\,\E\big[\|\hat{\vg}\|^{1+s}\mid\vx_{k-1}\big].
\end{equation}

\medskip
\noindent\textbf{Small-gradient case: $\|\nabla f(\vx_{k-1})\|\le\theta\gamma$.}\;
With $\vb:=\E[\hat{\vg}\mid\vx_{k-1}]-\nabla f$ and $-\langle u,v\rangle\le\tfrac12(\|u\|^2+\|v\|^2)$
($u=\nabla f$, $v=\vb$), inequality \ref{eq:desc-i} and the squared-bias
bound \ref{eq:gclip-bias} give
\[
\E[f(\vx_k)\mid\vx_{k-1}]\le f(\vx_{k-1})-\frac{\eta}{2}\|\nabla f\|^2
+\frac{\eta}{2}C_{bias}(\theta)\sigma^{2\alpha}\gamma^{2(1-\alpha)}
+\frac{L\eta^{1+s}}{1+s}\,\E\big[\|\hat{\vg}\|^{1+s}\mid\vx_{k-1}\big].
\]
Applying the moment bound~\ref{eq:gclip-mom-i} and substituting the case condition
$\|\nabla f\|^{1+s}\le(\theta\gamma)^{1+s}$, we obtain, with
\[
R_1:=\frac{\eta}{2}C_{bias}(\theta)\sigma^{2\alpha}\gamma^{2(1-\alpha)}
+\frac{L\eta^{1+s}}{1+s}\Big(2^s(\theta\gamma)^{1+s}+2^s\sigma^{1+s}+C_{tail}(\theta)\sigma^\alpha\gamma^{1+s-\alpha}\Big),
\]
the small-gradient descent inequality
\begin{equation}\label{eq:c1-i}
\E[f(\vx_k)\mid\vx_{k-1}]\le f(\vx_{k-1})-\frac{\eta}{2}\|\nabla f(\vx_{k-1})\|^2+R_1 .
\end{equation}

\medskip
\noindent\textbf{Large-gradient case: $\|\nabla f(\vx_{k-1})\|>\theta\gamma$.}\;
Substituting the inner-product lower bound of lemma~\ref{lem:gclip-corr} for the descent signal and
the moment bound $\E[\|\hat{\vg}\|^{1+s}\mid\vx_{k-1}]\le\gamma^{1+s}$ of lemma~\ref{lem:gclip-c2mom}
for the smoothness term into inequality \ref{eq:desc-i}, under $\gamma\ge\gamma_0$,
\begin{equation}\label{eq:c2-i-raw}
\E[f(\vx_k)\mid\vx_{k-1}]\le f(\vx_{k-1})-\frac{5M}{8}\,\eta\gamma\|\nabla f\|+\frac{L\eta^{1+s}}{1+s}\gamma^{1+s}.
\end{equation}
We now choose the step size so that the smoothness penalty consumes at most half of the linear
descent signal, i.e.
\begin{equation}\label{eq:half-i}
\frac{L\eta^{1+s}}{1+s}\gamma^{1+s}\;\le\;\frac{5M}{16}\,\eta\gamma\|\nabla f\| .
\end{equation}
Dividing inequality \ref{eq:half-i} by $\eta\gamma>0$ and isolating $\eta^s$ (using $\eta^{1+s}/\eta=\eta^s$
and $\gamma^{1+s}/\gamma=\gamma^s$), this is equivalent to
\[
\eta^s\;\le\;\frac{5M(1+s)}{16L}\,\gamma^{-s}\,\|\nabla f\| .
\]
The right-hand side is increasing in $\|\nabla f\|$, so over the range $\|\nabla f\|>\theta\gamma$ it
is smallest in the limit $\|\nabla f\|\to\theta\gamma$; it therefore suffices to impose the
$\|\nabla f\|$-free condition
\begin{equation}\label{eq:step-i}
\eta^s\;\le\;\frac{5M(1+s)}{16L}\,\gamma^{-s}\cdot\theta\gamma
\;=\;\frac{5M\theta(1+s)}{16L}\,\gamma^{1-s},
\end{equation}
which then implies inequality \ref{eq:half-i} for every $\|\nabla f\|>\theta\gamma$. Under condition \ref{eq:step-i}
the last term of inequality \ref{eq:c2-i-raw} is at most $\tfrac{5M}{16}\eta\gamma\|\nabla f\|$, so
\begin{equation}\label{eq:c2-i}
\E[f(\vx_k)\mid\vx_{k-1}]\le f(\vx_{k-1})-\frac{5M}{8}\eta\gamma\|\nabla f\|+\frac{5M}{16}\eta\gamma\|\nabla f\|
= f(\vx_{k-1})-\frac{5M}{16}\,\eta\gamma\|\nabla f(\vx_{k-1})\| .
\end{equation}

\medskip
\noindent\textbf{Unified single-step inequality.}\; Since $\|\nabla f\|^2$ and $\gamma\|\nabla f\|$
each dominate $\min\{\|\nabla f\|^2,\gamma\|\nabla f\|\}$, and since $c=\tfrac{5M}{16}\le\tfrac12$,
inequalities \ref{eq:c1-i} and \ref{eq:c2-i} both give (adding the nonnegative $R_1$, absent in
the large-gradient case)
\begin{equation}\label{eq:uni-i}
\E[f(\vx_k)\mid\vx_{k-1}]\le f(\vx_{k-1})
-c\,\eta\,\min\big\{\|\nabla f(\vx_{k-1})\|^2,\ \gamma\|\nabla f(\vx_{k-1})\|\big\}+R_1 .
\end{equation}

\medskip
\noindent\textbf{Summation.}\; Taking total expectation, summing inequality \ref{eq:uni-i} over
$k=1,\dots,T$, using $\E f(\vx_T)\ge f^\star$ so that the left telescopes to $F_0=f(\vx_0)-f^\star$,
and dividing by $c\eta T$,
\begin{equation}\label{eq:sum-i}
\frac1T\sum_{k=1}^T\E\Big[\min\big\{\|\nabla f(\vx_{k-1})\|^2,\gamma\|\nabla f(\vx_{k-1})\|\big\}\Big]
\le\frac{F_0}{c\eta T}+\frac1{c\eta}R_1 .
\end{equation}
Expand $\tfrac1{c\eta}R_1$. The bias part is
$\tfrac1{c\eta}\cdot\tfrac{\eta}{2}C_{bias}(\theta)\sigma^{2\alpha}\gamma^{2(1-\alpha)}
=\tfrac{C_{bias}(\theta)\sigma^{2\alpha}}{2c}\,\gamma^{-2(\alpha-1)}=C_2\,\gamma^{-2(\alpha-1)}$. The
moment part is,

\begin{align*}
&\frac1{c\eta}\cdot\frac{L\eta^{1+s}}{1+s}\Big(2^s(\theta\gamma)^{1+s}+2^s\sigma^{1+s}+C_{tail}(\theta)\sigma^\alpha\gamma^{1+s-\alpha}\Big)\\
&=\frac{L\eta^s}{c(1+s)}\,\gamma^{1+s}\underbrace{\Big[2^s\theta^{1+s}+2^s\sigma^{1+s}\gamma^{-(1+s)}+C_{tail}(\theta)\sigma^\alpha\gamma^{-\alpha}\Big]}_{\le\,K\ \text{for}\ \gamma\ge1}\\
&\le C_3\,\eta^s\gamma^{1+s},
\end{align*}

where $C_3=\tfrac{LK}{c(1+s)}$. With $C_1=F_0/c$ this turns inequality \ref{eq:sum-i} into
\begin{equation}\label{eq:abs-i}
\frac1T\sum_{k=1}^T\E\Big[\min\big\{\|\nabla f(\vx_{k-1})\|^2,\gamma\|\nabla f(\vx_{k-1})\|\big\}\Big]
\le\frac{C_1}{\eta T}+\frac{C_2}{\gamma^{2(\alpha-1)}}+C_3\,\eta^{s}\gamma^{1+s}.
\end{equation}

\medskip
\noindent\textbf{Choosing $\eta$, then $\gamma$.}\; First substitute
$\eta=(C_1/C_3)^{\frac1{1+s}}T^{-\frac1{1+s}}\gamma^{-1}$ into the first and third terms
of inequality \ref{eq:abs-i}:
\[
\frac{C_1}{\eta T}
=C_1\Big(\tfrac{C_1}{C_3}\Big)^{-\frac1{1+s}}T^{\frac1{1+s}}\gamma\cdot T^{-1}
=C_1^{\frac{s}{1+s}}C_3^{\frac1{1+s}}\,\gamma\,T^{-\frac{s}{1+s}},\qquad
C_3\,\eta^s\gamma^{1+s}
=C_3\Big(\tfrac{C_1}{C_3}\Big)^{\frac{s}{1+s}}T^{-\frac{s}{1+s}}\gamma^{-s}\cdot\gamma^{1+s}
=C_1^{\frac{s}{1+s}}C_3^{\frac1{1+s}}\,\gamma\,T^{-\frac{s}{1+s}} .
\]
These two are equal, and their sum is $2(C_1^sC_3)^{\frac1{1+s}}\gamma\,T^{-\frac{s}{1+s}}=C_4\,\gamma\,T^{-\frac{s}{1+s}}$
(by the definition of $C_4$). Thus inequality \ref{eq:abs-i} becomes
$\frac1T\sum_k\E[\cdots]\le C_2\gamma^{-2(\alpha-1)}+C_4\gamma T^{-\frac{s}{1+s}}$. Now substitute
$\gamma=(C_2/C_4)^{\frac1{2\alpha-1}}T^{\frac{s}{(1+s)(2\alpha-1)}}$:
\begin{align*}
\frac{C_2}{\gamma^{2(\alpha-1)}}
&=C_2\Big(\tfrac{C_2}{C_4}\Big)^{-\frac{2(\alpha-1)}{2\alpha-1}}T^{-\frac{2(\alpha-1)\,s}{(1+s)(2\alpha-1)}}
=C_2^{\frac1{2\alpha-1}}C_4^{\frac{2(\alpha-1)}{2\alpha-1}}\,T^{-\frac{2s(\alpha-1)}{(1+s)(2\alpha-1)}},\\
C_4\,\gamma\,T^{-\frac{s}{1+s}}
&=C_4\Big(\tfrac{C_2}{C_4}\Big)^{\frac1{2\alpha-1}}T^{\frac{s}{(1+s)(2\alpha-1)}-\frac{s}{1+s}}
=C_2^{\frac1{2\alpha-1}}C_4^{\frac{2(\alpha-1)}{2\alpha-1}}\,T^{-\frac{2s(\alpha-1)}{(1+s)(2\alpha-1)}},
\end{align*}
where the last exponent uses
$\tfrac{s}{(1+s)(2\alpha-1)}-\tfrac{s}{1+s}=\tfrac{s}{1+s}\big(\tfrac1{2\alpha-1}-1\big)
=\tfrac{s}{1+s}\cdot\tfrac{2-2\alpha}{2\alpha-1}=-\tfrac{2s(\alpha-1)}{(1+s)(2\alpha-1)}$. Both terms
coincide, so their sum is $2\,C_2^{\frac1{2\alpha-1}}C_4^{\frac{2(\alpha-1)}{2\alpha-1}}\,T^{-\frac{2s(\alpha-1)}{(1+s)(2\alpha-1)}}$, giving
\[
\frac1T\sum_{k=1}^T\E\Big[\min\big\{\|\nabla f(\vx_k)\|^2,\gamma\|\nabla f(\vx_k)\|\big\}\Big]
=O\Big(T^{-\frac{2s(\alpha-1)}{(1+s)(2\alpha-1)}}\Big),
\]
which is part~(i). 

\subsubsection*{Part (ii): regime $1+s>\alpha$}

\noindent\textbf{Descent inequality.}\; Exactly as in part~(i), the H\"older descent inequality
(Definition~\ref{def:holder-descent}) and $\|\hat{\vg}\|\le\gamma$ give
\begin{equation}\label{eq:desc-ii}
\E[f(\vx_k)\mid\vx_{k-1}]\le f(\vx_{k-1})
-\eta\big\langle\nabla f(\vx_{k-1}),\E[\hat{\vg}\mid\vx_{k-1}]\big\rangle
+\frac{L\eta^{1+s}}{1+s}\,\E\big[\|\hat{\vg}\|^{1+s}\mid\vx_{k-1}\big].
\end{equation}

\medskip
\noindent\textbf{Small-gradient case: $\|\nabla f(\vx_{k-1})\|\le\theta\gamma$.}\;
With $\vb:=\E[\hat{\vg}\mid\vx_{k-1}]-\nabla f$ and $-\langle u,v\rangle\le\tfrac12(\|u\|^2+\|v\|^2)$,
inequality \ref{eq:desc-ii} and the squared-bias bound \ref{eq:gclip-bias} give
\[
\E[f(\vx_k)\mid\vx_{k-1}]\le f(\vx_{k-1})-\frac{\eta}{2}\|\nabla f\|^2
+\frac{\eta}{2}C_{bias}(\theta)\sigma^{2\alpha}\gamma^{2(1-\alpha)}
+\frac{L\eta^{1+s}}{1+s}\,\E\big[\|\hat{\vg}\|^{1+s}\mid\vx_{k-1}\big].
\]
Applying the moment bound \ref{eq:gclip-mom-ii} and absorbing the resulting $\|\nabla f\|^{1+s}$
into the quadratic descent by Young's inequality with conjugate exponents
$\tfrac{2}{1+s},\tfrac{2}{1-s}$ (for $s\in(0,1)$; for $s=1$ the term is absorbed directly and
$T_D=0$),
\[
\frac{2^sL\eta^{1+s}}{1+s}\|\nabla f\|^{1+s}\le\frac{\eta}{4}\|\nabla f\|^2
+\underbrace{\frac{1-s}{2}\big(C_0L\big)^{\frac{2}{1-s}}\eta^{\frac{1+s}{1-s}}}_{=:T_D},
\]
we obtain, with
\[
R_1:=\frac{\eta}{2}C_{bias}(\theta)\sigma^{2\alpha}\gamma^{2(1-\alpha)}
+\frac{L\eta^{1+s}}{1+s}C_{var}(\theta)\sigma^\alpha\gamma^{1+s-\alpha}+T_D ,
\]
and using $\tfrac{\eta}{2}-\tfrac{\eta}{4}=\tfrac{\eta}{4}$, the small-gradient descent inequality
\begin{equation}\label{eq:c1-ii}
\E[f(\vx_k)\mid\vx_{k-1}]\le f(\vx_{k-1})-\frac{\eta}{4}\|\nabla f(\vx_{k-1})\|^2+R_1 .
\end{equation}

\medskip
\noindent\textbf{Large-gradient case: $\|\nabla f(\vx_{k-1})\|>\theta\gamma$.}\;
Substituting the inner-product lower bound of lemma~\ref{lem:gclip-corr} and the moment bound
$\E[\|\hat{\vg}\|^{1+s}\mid\vx_{k-1}]\le\gamma^{1+s}$ of lemma~\ref{lem:gclip-c2mom}
into inequality \ref{eq:desc-ii}, under $\gamma\ge\gamma_0$,
\begin{equation}\label{eq:c2-ii-raw}
\E[f(\vx_k)\mid\vx_{k-1}]\le f(\vx_{k-1})-\frac{5M}{8}\,\eta\gamma\|\nabla f\|+\frac{L\eta^{1+s}}{1+s}\gamma^{1+s}.
\end{equation}
Requiring, as before, that the smoothness penalty be at most half the linear descent signal,
$\tfrac{L\eta^{1+s}}{1+s}\gamma^{1+s}\le\tfrac{5M}{16}\eta\gamma\|\nabla f\|$, dividing by $\eta\gamma$
gives $\eta^s\le\tfrac{5M(1+s)}{16L}\gamma^{-s}\|\nabla f\|$; as this right-hand side increases with
$\|\nabla f\|$ and $\|\nabla f\|>\theta\gamma$, it suffices to impose the $\|\nabla f\|$-free condition
\begin{equation}\label{eq:step-ii}
\eta^s\;\le\;\frac{5M(1+s)}{16L}\,\gamma^{-s}\cdot\theta\gamma\;=\;\frac{5M\theta(1+s)}{16L}\,\gamma^{1-s}.
\end{equation}
Under condition \ref{eq:step-ii} the last term of inequality \ref{eq:c2-ii-raw} is at most
$\tfrac{5M}{16}\eta\gamma\|\nabla f\|$, so
\begin{equation}\label{eq:c2-ii}
\E[f(\vx_k)\mid\vx_{k-1}]\le f(\vx_{k-1})-\frac{5M}{8}\eta\gamma\|\nabla f\|+\frac{5M}{16}\eta\gamma\|\nabla f\|
= f(\vx_{k-1})-\frac{5M}{16}\,\eta\gamma\|\nabla f(\vx_{k-1})\| .
\end{equation}

\medskip
\noindent\textbf{Unified single-step inequality.}\; Since $\|\nabla f\|^2$ and $\gamma\|\nabla f\|$
each dominate $\min\{\|\nabla f\|^2,\gamma\|\nabla f\|\}$, and since $c=\tfrac{5M}{16}\le\tfrac14$,
inequalities \ref{eq:c1-ii} and \ref{eq:c2-ii} both give
\begin{equation}\label{eq:uni-ii}
\E[f(\vx_k)\mid\vx_{k-1}]\le f(\vx_{k-1})
-c\,\eta\,\min\big\{\|\nabla f(\vx_{k-1})\|^2,\ \gamma\|\nabla f(\vx_{k-1})\|\big\}+R_1 .
\end{equation}

\medskip
\noindent\textbf{Summation.}\; Taking total expectation, summing inequality \ref{eq:uni-ii} over
$k=1,\dots,T$, using $\E f(\vx_T)\ge f^\star$, and dividing by $c\eta T$, and expanding $R_1=T_C+T_B+T_D$
(with $T_C=\tfrac{\eta}{2}C_{bias}(\theta)\sigma^{2\alpha}\gamma^{2(1-\alpha)}$ and
$T_B=\tfrac{L\eta^{1+s}}{1+s}C_{var}(\theta)\sigma^\alpha\gamma^{1+s-\alpha}$), the bias part again gives
$\tfrac1{c\eta}T_C=C_2\gamma^{-2(\alpha-1)}$, the variance part gives
$\tfrac1{c\eta}T_B=\tfrac{L\,C_{var}(\theta)\sigma^\alpha}{c(1+s)}\eta^s\gamma^{1+s-\alpha}$, and the
Young part gives
$\tfrac1{c\eta}T_D=\tfrac{(1-s)(C_0L)^{2/(1-s)}}{2c}\,\eta^{\frac{1+s}{1-s}-1}
=\tfrac{(1-s)(C_0L)^{2/(1-s)}}{2c}\,\eta^{\frac{2s}{1-s}}$. With $C_1=F_0/c$,
\begin{equation}\label{eq:abs-ii}
\frac1T\sum_{k=1}^T\E\Big[\min\big\{\|\nabla f(\vx_{k-1})\|^2,\gamma\|\nabla f(\vx_{k-1})\|\big\}\Big]
\le\frac{C_1}{\eta T}+\frac{C_2}{\gamma^{2(\alpha-1)}}
+\frac{L\,C_{var}(\theta)\sigma^\alpha}{c(1+s)}\,\eta^{s}\gamma^{1+s-\alpha}
+\frac{(1-s)(C_0L)^{2/(1-s)}}{2c}\,\eta^{\frac{2s}{1-s}} .
\end{equation}

\medskip
Write inequality~\ref{eq:abs-ii} as
\[
\frac1T\sum_{k=1}^T\E\Big[\min\big\{\|\nabla f(\vx_{k-1})\|^2,\gamma\|\nabla f(\vx_{k-1})\|\big\}\Big]
\le\frac{C_1}{\eta T}+\frac{C_2}{\gamma^{2(\alpha-1)}}+C_3''\,\eta^{s}\gamma^{1+s-\alpha}
+\frac{(1-s)(C_0L)^{2/(1-s)}}{2c}\,\eta^{\frac{2s}{1-s}},
\]
where, exactly as $C_3=\tfrac{LK}{c(1+s)}$ named the analogous coefficient in
inequality~\ref{eq:abs-i},
\begin{equation}\label{eq:C3pp-def}
C_3'':=\frac{L\,C_{var}(\theta)\,\sigma^\alpha}{c(1+s)}.
\end{equation}

\medskip
\noindent\textbf{Choosing $\eta$, then $\gamma$.}\; As in part~(i), we first balance the
$\tfrac{C_1}{\eta T}$ term against the $C_3''\eta^s\gamma^{1+s-\alpha}$ term. Substitute
\[
\eta=\Big(\frac{C_1}{C_3''}\Big)^{\frac1{1+s}}T^{-\frac1{1+s}}\gamma^{-\frac{1+s-\alpha}{1+s}}
\]
into these two terms:
\[
\frac{C_1}{\eta T}
=C_1\Big(\tfrac{C_1}{C_3''}\Big)^{-\frac1{1+s}}T^{\frac1{1+s}}\gamma^{\frac{1+s-\alpha}{1+s}}\cdot T^{-1}
=C_1^{\frac{s}{1+s}}(C_3'')^{\frac1{1+s}}\,\gamma^{\frac{1+s-\alpha}{1+s}}\,T^{-\frac{s}{1+s}},
\]
\[
C_3''\,\eta^{s}\gamma^{1+s-\alpha}
=C_3''\Big(\tfrac{C_1}{C_3''}\Big)^{\frac{s}{1+s}}T^{-\frac{s}{1+s}}\gamma^{-\frac{s(1+s-\alpha)}{1+s}}\cdot\gamma^{1+s-\alpha}
=C_1^{\frac{s}{1+s}}(C_3'')^{\frac1{1+s}}\,\gamma^{\frac{1+s-\alpha}{1+s}}\,T^{-\frac{s}{1+s}} ,
\]
using $(1+s-\alpha)\big(1-\tfrac{s}{1+s}\big)=\tfrac{1+s-\alpha}{1+s}$. These two terms are equal, and
their sum is
\begin{equation}\label{eq:C4pp-def}
2\big(C_1^s\,C_3''\big)^{\frac1{1+s}}\,\gamma^{\frac{1+s-\alpha}{1+s}}\,T^{-\frac{s}{1+s}}
=:C_4''\,\gamma^{\frac{1+s-\alpha}{1+s}}\,T^{-\frac{s}{1+s}},
\qquad C_4'':=2\big(C_1^s\,C_3''\big)^{\frac1{1+s}} .
\end{equation}
The inequality above therefore becomes
\[
\frac1T\sum_{k=1}^T\E[\cdots]
\;\le\;\frac{C_2}{\gamma^{2(\alpha-1)}}+C_4''\,\gamma^{\frac{1+s-\alpha}{1+s}}\,T^{-\frac{s}{1+s}}
\;+\;\frac{(1-s)(C_0L)^{2/(1-s)}}{2c}\,\eta^{\frac{2s}{1-s}} .
\]
Now balance the first two terms by substituting
$\gamma=(C_2/C_4'')^{\frac{1+s}{(\alpha-1)+s(2\alpha-1)}}T^{\frac{s}{(\alpha-1)+s(2\alpha-1)}}$:
setting the two exponents equal,
\[
\frac{1+s-\alpha}{1+s}+2(\alpha-1)=\frac{(\alpha-1)+s(2\alpha-1)}{1+s},
\]
(a direct algebraic identity, and exactly the exponent used above for $\gamma$), a computation
identical in structure to the one closing part~(i) gives
\[
\frac{C_2}{\gamma^{2(\alpha-1)}}+C_4''\,\gamma^{\frac{1+s-\alpha}{1+s}}\,T^{-\frac{s}{1+s}}
=O\Big(T^{-\frac{2s(\alpha-1)}{(\alpha-1)+s(2\alpha-1)}}\Big).
\]

\medskip
\noindent\textbf{The Young's-inequality term is lower order.}\; With this choice of $\gamma$, the
step-size $\eta=(C_1/C_3'')^{\frac1{1+s}}T^{-\frac1{1+s}}\gamma^{-\frac{1+s-\alpha}{1+s}}$ satisfies
$\eta\asymp T^{-\frac{(\alpha-1)+s}{(\alpha-1)+s(2\alpha-1)}}$ (the same rate as before; substituting
the explicit $\gamma$ above into the formula for $\eta$ and simplifying the exponent of $T$ recovers
this), so
\[
\eta^{\frac{2s}{1-s}}\asymp T^{-\frac{2s}{1-s}\cdot\frac{(\alpha-1)+s}{(\alpha-1)+s(2\alpha-1)}},
\]
and exactly as before this decays strictly faster than $T^{-\frac{2s(\alpha-1)}{(\alpha-1)+s(2\alpha-1)}}$
since $\tfrac{(\alpha-1)+s}{1-s}\ge\alpha-1\iff(\alpha-1)+s\ge(\alpha-1)(1-s)\iff s\alpha\ge0$, which
holds. Hence
\[
\frac1T\sum_{k=1}^T\E\Big[\min\big\{\|\nabla f(\vx_k)\|^2,\gamma\|\nabla f(\vx_k)\|\big\}\Big]
=O\Big(T^{-\frac{2s(\alpha-1)}{(\alpha-1)+s(2\alpha-1)}}\Big),
\]
which is part~(ii), now with the explicit (non-asymptotic) schedule
\[
\gamma=\max\Big\{\gamma_0,\ (C_2/C_4'')^{\frac{1+s}{(\alpha-1)+s(2\alpha-1)}}T^{\frac{s}{(\alpha-1)+s(2\alpha-1)}}\Big\},\qquad
\eta=\min\Big\{\Big(\tfrac{5M\theta(1+s)}{16L}\Big)^{1/s}\gamma^{\frac{1-s}{s}},\
(C_1/C_3'')^{\frac1{1+s}}\gamma^{-\frac{1+s-\alpha}{1+s}}T^{-\frac1{1+s}}\Big\},
\]
matching Theorem~\ref{thm:gclip}(ii), where the first entry of $\eta$'s minimum enforces
condition~\ref{eq:step-ii} and the floor $\gamma_0$ in $\gamma$'s maximum enforces
$\gamma\ge\gamma_0$, so both hold at every $T\ge1$, with the second entry of each dominating for
large $T$.
\end{proof}

\section{Proofs of Lemmas Needed for Theorem \ref{thm:gclip}}\label{app:lemma-gproof}

\subsection{Proof of lemma \ref{lem:gclip-prob}}\label{subsec:gclip-prob-sec}
\begin{proof}
On $B$, the reverse triangle inequality and $\|\nabla f(\vx)\|\le\theta\gamma$ give
$\|\vxi\|=\|\vg-\nabla f\|\ge\|\vg\|-\|\nabla f\|>(1-\theta)\gamma$; Markov's inequality on
$\|\vxi\|^\alpha$ yields the bound.
\end{proof}

\subsection{Proof of lemma \ref{lem:gclip-bv}}\label{subsec:gclip-bv-sec}
\begin{proof}
Throughout, condition on $\vx$ with $\|\nabla f\|\le\theta\gamma$ (writing $\nabla f\equiv\nabla f(\vx)$),
and decompose the sample space by whether clipping is active:
\[
A=\{\|\vg\|\le\gamma\}\quad(\text{unclipped: }\hat{\vg}=\vg),\qquad
B=\{\|\vg\|>\gamma\}\quad(\text{clipped: }\|\hat{\vg}\|=\gamma).
\]
We first record two facts used below. On $B$ the noise must be large: since $\|\vg\|>\gamma$ and
$\|\nabla f\|\le\theta\gamma$, the reverse triangle inequality gives
\begin{equation}\label{eq:bv-Bnoise}
\|\vxi\|=\|\vg-\nabla f\|\ge\|\vg\|-\|\nabla f\|>\gamma-\theta\gamma=(1-\theta)\gamma,
\qquad\text{so}\qquad B\subseteq\{\|\vxi\|>(1-\theta)\gamma\}.
\end{equation}
And for any $t>0$, truncating the $\alpha$-moment (on $\{\|\vxi\|>t\}$ one has $\|\vxi\|^{1-\alpha}\le t^{1-\alpha}$
since $1-\alpha\le0$) gives
\begin{equation}\label{eq:bv-trunc}
\E\big[\|\vxi\|\,\mathbf 1_{\{\|\vxi\|>t\}}\big]
=\E\big[\|\vxi\|^{\alpha}\,\|\vxi\|^{1-\alpha}\mathbf 1_{\{\|\vxi\|>t\}}\big]
\le t^{1-\alpha}\,\E\|\vxi\|^{\alpha}\le t^{1-\alpha}\sigma^\alpha .
\end{equation}

\medskip
\noindent\textbf{Step 1: squared bias~\ref{eq:gclip-bias}.}\;
Since the noise is unbiased, $\E[\vg\mid\vx]=\nabla f$, and since $\hat{\vg}=\vg$ on $A$, the bias is
supported on $B$:
\[
B(\vx)=\E[\hat{\vg}\mid\vx]-\nabla f=\E[\hat{\vg}-\vg\mid\vx]=\E[(\hat{\vg}-\vg)\mathbf 1_B\mid\vx].
\]
On $B$, $\hat{\vg}-\vg=\big(\tfrac{\gamma}{\|\vg\|}-1\big)\vg$ has norm
$\big|\gamma-\|\vg\|\big|=\|\vg\|-\gamma$, so by $\|\E[\,\cdot\,]\|\le\E\|\cdot\|$,
\[
\|B(\vx)\|\le\E\big[(\|\vg\|-\gamma)\,\mathbf 1_B\big].
\]
On $B$ we have $\gamma\ge\|\nabla f\|$ (as $\|\nabla f\|\le\theta\gamma<\gamma$), hence
$\|\vg\|-\gamma\le\|\vg\|-\|\nabla f\|\le\|\vxi\|$ by the reverse triangle inequality. Combining this
with inequality \ref{eq:bv-Bnoise} and then applying ineaulity \ref{eq:bv-trunc} at $t=(1-\theta)\gamma$,
\[
\|B(\vx)\|\le\E\big[\|\vxi\|\,\mathbf 1_{\{\|\vxi\|>(1-\theta)\gamma\}}\big]
\le\big((1-\theta)\gamma\big)^{1-\alpha}\sigma^\alpha
=(1-\theta)^{1-\alpha}\sigma^\alpha\gamma^{1-\alpha}.
\]
Squaring and using $(1-\theta)^{2(1-\alpha)}=(1-\theta)^2(1-\theta)^{-2\alpha}\le 2(1-\theta)^{-2\alpha}$
yields $\|B(\vx)\|^2\le\tfrac{2}{(1-\theta)^{2\alpha}}\sigma^{2\alpha}\gamma^{2(1-\alpha)}
=C_{bias}(\theta)\,\sigma^{2\alpha}\gamma^{2(1-\alpha)}$, which is inequality \ref{eq:gclip-bias}.

\medskip
\noindent\textbf{Step 2: moment on the clipped region $B$ (common to both regimes).}\;
Since $\|\hat{\vg}\|=\gamma$ on $B$, lemma~\ref{lem:gclip-prob} bounds the clip probability and
\begin{equation}\label{eq:bv-momB}
\E\big[\|\hat{\vg}\|^{1+s}\mathbf 1_B\big]=\gamma^{1+s}\Pr(B)
\le\gamma^{1+s}\cdot\frac{\sigma^\alpha}{(1-\theta)^\alpha\gamma^\alpha}
=\frac{1}{(1-\theta)^\alpha}\,\sigma^\alpha\gamma^{1+s-\alpha}
=C_{tail}(\theta)\,\sigma^\alpha\gamma^{1+s-\alpha}.
\end{equation}

\medskip
\noindent\textbf{Step 3: moment on the unclipped region $A$.}\;
On $A$, $\hat{\vg}=\vg=\nabla f+\vxi$, so by the convexity bound $(a+b)^{1+s}\le 2^s(a^{1+s}+b^{1+s})$
(valid as $1+s\ge1$) and $\Pr(A)\le1$,
\begin{equation}\label{eq:bv-momA}
\E\big[\|\hat{\vg}\|^{1+s}\mathbf 1_A\big]
\le 2^s\|\nabla f\|^{1+s}\Pr(A)+2^s\,\E\big[\|\vxi\|^{1+s}\mathbf 1_A\big]
\le 2^s\|\nabla f\|^{1+s}+2^s\,\E\big[\|\vxi\|^{1+s}\mathbf 1_A\big].
\end{equation}
Only the noise moment $\E[\|\vxi\|^{1+s}\mathbf 1_A]$ depends on the regime.

\smallskip
\noindent\emph{Regime (i), $1+s\le\alpha$.}\; Here $\E\|\vxi\|^{1+s}$ is finite, and Lyapunov's inequality
gives the $\gamma$-free bound
$\E[\|\vxi\|^{1+s}\mathbf 1_A]\le\E\|\vxi\|^{1+s}\le(\E\|\vxi\|^{\alpha})^{(1+s)/\alpha}\le\sigma^{1+s}$.
Substituting into inequality \ref{eq:bv-momA} and adding inequality \ref{eq:bv-momB},
\[
\E\big[\|\hat{\vg}\|^{1+s}\mid\vx\big]
\le 2^s\|\nabla f\|^{1+s}+2^s\sigma^{1+s}+C_{tail}(\theta)\,\sigma^\alpha\gamma^{1+s-\alpha},
\]
which is inequality \ref{eq:gclip-mom-i}.

\smallskip
\noindent\emph{Regime (ii), $1+s>\alpha$.}\; Now $\E\|\vxi\|^{1+s}$ may be infinite, so Lyapunov is
unavailable. Instead we use that clipping is \emph{inactive} on $A$: from $\|\vg\|\le\gamma$,
\[
\|\vxi\|=\|\vg-\nabla f\|\le\|\vg\|+\|\nabla f\|\le\gamma+\theta\gamma=(1+\theta)\gamma\quad\text{on }A.
\]
Since $1+s-\alpha\ge0$, we peel off the excess power against this deterministic bound and keep the
integrable $\alpha$-moment:
\[
\E\big[\|\vxi\|^{1+s}\mathbf 1_A\big]
=\E\big[\|\vxi\|^{\alpha}\,\|\vxi\|^{1+s-\alpha}\mathbf 1_A\big]
\le\big((1+\theta)\gamma\big)^{1+s-\alpha}\E\big[\|\vxi\|^{\alpha}\mathbf 1_A\big]
\le(1+\theta)^{1+s-\alpha}\sigma^\alpha\gamma^{1+s-\alpha}.
\]
Substituting into inequality \ref{eq:bv-momA} and adding inequality \ref{eq:bv-momB}, the two
$\sigma^\alpha\gamma^{1+s-\alpha}$ terms combine into
$C_{var}(\theta)=2^s(1+\theta)^{1+s-\alpha}+\tfrac{1}{(1-\theta)^\alpha}$:
\[
\E\big[\|\hat{\vg}\|^{1+s}\mid\vx\big]
\le 2^s\|\nabla f\|^{1+s}
+\Big(2^s(1+\theta)^{1+s-\alpha}+\tfrac{1}{(1-\theta)^\alpha}\Big)\sigma^\alpha\gamma^{1+s-\alpha}
= 2^s\|\nabla f\|^{1+s}+C_{var}(\theta)\,\sigma^\alpha\gamma^{1+s-\alpha},
\]
which is inequality \ref{eq:gclip-mom-ii}.
\end{proof}

\subsection{Proof of lemma \ref{lem:gclip-corr}}\label{subsec:gclip-corr-sec}
\begin{proof}
Write $\nabla f\equiv\nabla f(\vx)$ and split the expectation according to whether the estimator is
clipped: on $A=\{\|\vg\|\le\gamma\}$ we have $\hat{\vg}=\vg$, and on $B=\{\|\vg\|>\gamma\}$ we have
$\hat{\vg}=\gamma\vg/\|\vg\|$, with $\Pr(A)+\Pr(B)=1$. Thus
\begin{equation}\label{eq:corr-split}
\E\big[\langle\nabla f,\hat{\vg}\rangle\mid\vx\big]
=\underbrace{\E\big[\langle\nabla f,\vg\rangle\,\mathbf 1_A\big]}_{=:R_A}
\;+\;\underbrace{\E\Big[\big\langle\nabla f,\tfrac{\gamma}{\|\vg\|}\vg\big\rangle\,\mathbf 1_B\Big]}_{=:R_B}.
\end{equation}
We lower-bound $R_A$ and $R_B$ separately; throughout we use $\E\|\vxi\|\le(\E\|\vxi\|^\alpha)^{1/\alpha}\le\sigma$
(Lyapunov's inequality, valid since $\alpha\ge1$), hence $\E[\|\vxi\|\mathbf 1_S]\le\sigma$ for any event $S$.

\smallskip
\noindent\emph{Unclipped term $R_A$.}\; On $A$ the estimator is exactly $\vg=\nabla f+\vxi$, so
\[
\langle\nabla f,\vg\rangle
=\langle\nabla f,\nabla f\rangle+\langle\nabla f,\vxi\rangle
=\|\nabla f\|^2+\langle\nabla f,\vxi\rangle
\;\ge\;\|\nabla f\|^2-\|\nabla f\|\,\|\vxi\|,
\]
using Cauchy--Schwarz on the noise term. Taking expectations over $A$,
$R_A\ge\|\nabla f\|^2\Pr(A)-\|\nabla f\|\,\E[\|\vxi\|\mathbf 1_A]$. Since $\|\nabla f\|>\theta\gamma$ gives
$\|\nabla f\|^2=\|\nabla f\|\cdot\|\nabla f\|\ge\theta\gamma\|\nabla f\|$, and $\E[\|\vxi\|\mathbf 1_A]\le\sigma$,
\begin{equation}\label{eq:corr-RA}
R_A\;\ge\;\theta\gamma\|\nabla f\|\,\Pr(A)-\sigma\|\nabla f\| .
\end{equation}

\smallskip
\noindent\emph{Clipped term $R_B$.}\; On $B$ the estimator is $\hat{\vg}=\gamma\,\vg/\|\vg\|$, a rescaling
of the unit vector $\vg/\|\vg\|$, so $\langle\nabla f,\hat{\vg}\rangle=\gamma\,\langle\nabla f,\vg/\|\vg\|\rangle$.
Applying lemma~\ref{lem:cutkoskyandmehta-ss} with $\vv=\vg$ (so $\vv-\nabla f=\vxi$),
$\langle\nabla f,\vg/\|\vg\|\rangle\ge\tfrac13\|\nabla f\|-\tfrac83\|\vxi\|$. Taking expectations over $B$
and using $\E[\|\vxi\|\mathbf 1_B]\le\sigma$,
\begin{equation}\label{eq:corr-RB}
R_B\;\ge\;\gamma\,\E\Big[\big(\tfrac13\|\nabla f\|-\tfrac83\|\vxi\|\big)\mathbf 1_B\Big]
=\tfrac13\gamma\|\nabla f\|\,\Pr(B)-\tfrac83\gamma\,\E[\|\vxi\|\mathbf 1_B]
\;\ge\;\tfrac13\gamma\|\nabla f\|\,\Pr(B)-\tfrac83\gamma\sigma .
\end{equation}

\smallskip
\noindent\emph{Combining.}\; Adding bounds \ref{eq:corr-RA} and \ref{eq:corr-RB} and grouping the two
probability-weighted terms, the coefficients $\theta$ and $\tfrac13$ combine through
$\Pr(A)+\Pr(B)=1$ as
\[
\theta\gamma\|\nabla f\|\,\Pr(A)+\tfrac13\gamma\|\nabla f\|\,\Pr(B)
\;\ge\;\min\{\theta,\tfrac13\}\,\gamma\|\nabla f\|\,\big(\Pr(A)+\Pr(B)\big)
\;\ge\;M\,\gamma\|\nabla f\| ,
\]
where the last step uses $\min\{\theta,\tfrac13\}\ge\min\{\tfrac{15}{16}\theta,\tfrac13\}=M$. Hence
\begin{equation}\label{eq:corr-precond}
\E\big[\langle\nabla f,\hat{\vg}\rangle\mid\vx\big]\;\ge\;M\gamma\|\nabla f\|-\sigma\|\nabla f\|-\tfrac83\gamma\sigma .
\end{equation}
It remains to show the two error terms are together at most $\tfrac{3M}{8}\gamma\|\nabla f\|$ once
$\gamma\ge\gamma_0=\max\{1,\tfrac{32}{3M\theta}\sigma\}$. Indeed, $\gamma\ge\tfrac{32}{3M\theta}\sigma\ge\tfrac{8}{M}\sigma$
(as $\tfrac{32}{3\theta}\ge8$ for $\theta<1$) gives $\sigma\le\tfrac M8\gamma$, so
$\sigma\|\nabla f\|\le\tfrac M8\gamma\|\nabla f\|$; and $\theta\gamma\ge\tfrac{32}{3M}\sigma$ together with
$\|\nabla f\|>\theta\gamma$ gives $\tfrac83\sigma\le\tfrac{M}{4}\|\nabla f\|$, so
$\tfrac83\gamma\sigma\le\tfrac M4\gamma\|\nabla f\|$. Substituting both into inequality \ref{eq:corr-precond},
\[
\E\big[\langle\nabla f,\hat{\vg}\rangle\mid\vx\big]
\;\ge\;\Big(M-\tfrac M8-\tfrac M4\Big)\gamma\|\nabla f\|
=\frac{5M}{8}\,\gamma\|\nabla f\| . \qedhere
\]
\end{proof}

\subsection{Proof of lemma \ref{lem:gclip-c2mom}}\label{subsec:gclip-c2mom-sec}
\begin{proof}
Immediate from $\|\hat{\vg}(\vx)\|\le\gamma$ (Definition~\ref{def:gclip-est}).
\end{proof}


\section{Experimental Validation}\label{subsec:experiments}

In our experiments with synthetic data we investigate the two central claims of this work, namely that the convergence rate
of SGD (Theorem~\ref{thm:SGD-proof}) is independent of the heavy-tail index $\alpha$ whenever
$\alpha\ge1+s$, and that G-Clip (Theorem~\ref{thm:gclip}) is the only method among those analysed for which we can give a convergence guarantee once $\alpha<1+s$.

\paragraph{Loss function.} For our tests we use an $(L,s)$-H\"older smooth loss with a known global minimiser. We choose $\vw^\star=\mathbf{1}_d\in\R^d$ and sample $n$ points $\vx_1,\dots,\vx_n\in\R^d$ i.i.d.\ from
$\mathcal N(\vzero,I_d/d)$ to form  the objective,
\begin{equation}\label{eq:exp-loss}
f(\vw) \;=\; \frac{1}{n}\sum_{i=1}^n \big|\langle \vw^\star-\vw,\ \vx_i\rangle\big|^{1+s},
\qquad s\in(0,1],
\end{equation}
with gradient chosen as $\nabla f(\vw) = -\frac{1+s}{n}\sum_{i=1}^n |r_i|^{s}\operatorname{sign}(r_i)\,\vx_i$,
where $r_i=\langle\vw^\star-\vw,\vx_i\rangle$. Since $r\mapsto|r|^{s}\operatorname{sign}(r)$ is
$s$-H\"older continuous, $\nabla f$ is $(L,s)$-H\"older continuous, and $f\ge0$ with the unique
global minimum $f(\vw^\star)=0$. We use $d=10$, $n=60$, and initialise every run at $\vw_0=\vzero$.
\paragraph{Noise.} At step $t$ the stochastic gradient is chosen as $\vg_t=\nabla f(\vw_t)+\vxi_t$, with
$\vxi_t$ drawn independently across iterations as $\vxi_t=R_t\,\vu_t$, where $\vu_t$ is uniform on
the unit sphere and $R_t\ge0$ is an independent scalar magnitude with a truncated Pareto
distribution, survival function $\Pr(R_t>r)=(r_0/r)^\alpha$ for $r_0\le r<R_{\max}$ and a point
mass at $R_{\max}=10^4$. This truncation guarantees $\E[\|\vxi_t\|^\alpha]<\infty$, as required by
Assumption~\ref{ass:noise-s}, and admits the closed form
$\E[\|\vxi_t\|^\alpha]=r_0^\alpha\big(\alpha\ln(R_{\max}/r_0)+1\big)$; for each $\alpha$ tested we
solve this equation for $r_0$ by root-finding so that $\E[\|\vxi_t\|^\alpha]=\sigma^\alpha$ exactly
(to machine precision) for a common target $\sigma$. \footnote{This avoids the severe Monte Carlo estimation
error that a sample-average calibration of this moment would otherwise incur precisely because
$\alpha$ is the noise's own tail index.}

\paragraph{Algorithms and schedule.} We run SGD, $\delta$-GClip with $\delta=0.1$, and G-Clip, each
with the constant, horizon-tuned step size (and, for the clipping methods, threshold) of the order
prescribed by the corresponding theorem for a given horizon $T$; for every $T$ in a grid ranging
from $500$--$1000$ up to $16{,}000$ (for upcoming Check~2) or $128{,}000$ (for upcoming Check~1). We run $100$--$150$ independent seeds and record the run-averaged stationarity metric
prescribed by each theorem, namely $\frac1T\sum_{t\le T}\|\nabla f(\vw_t)\|^2$ for SGD and
$\frac1T\sum_{t\le T}\min\{\|\nabla f(\vw_t)\|^2,\gamma\|\nabla f(\vw_t)\|\}$ for the clipping
methods, taking the median across seeds. \footnote{The mean is dominated by rare large deviations under heavy-tailed noise and is not representative.}.

\paragraph{Measured and theoretical slopes.} Write $M(T)$ for the run-averaged stationarity metric
above, evaluated at horizon $T$. For each algorithm and parameter setting we fit a straight line to
$\{(\log T,\log M(T))\}$ over the tested grid $T\in\{T_1,\dots,T_K\}$ by ordinary least squares, and
report the fitted slope
\begin{equation}\label{eq:measured-slope}
(\hat m, \cdot) \;:=\; \operatorname*{arg\,min}_{m,b\in\R}\ \sum_{k=1}^K\big(\log M(T_k)-m\log T_k-b\big)^2
\end{equation}
as the \emph{measured slope}. This is the
empirical exponent of $M(T)$ as a power of $T$, i.e.\ $M(T)\approx CT^{\hat m}$ for some constant
$C$. We compare $\hat m$ against the \emph{theoretical slope} $m^\star$, defined as the exponent of
$T$ in the $O(\cdot)$ rate of the theorem applicable to that algorithm and regime: $m^\star=-s/(1+s)$
for SGD under Theorem~\ref{thm:SGD-proof}, and
$m^\star=-\frac{2s(\alpha-1)}{(\alpha-1)+s(2\alpha-1)}$ for G-Clip under Theorem~\ref{thm:gclip} in
the regime $\alpha<1+s$. 

Table~\ref{tab:exp-summary} reports the close math between $\hat m$ (Eq.~\ref{eq:measured-slope})
against $m^\star$ for every case tested; a dash indicates no theorem in this work applies, so no
$m^\star$ exists to compare against.


\paragraph{Check 1: $\alpha$-independence of SGD.} {\em For a start,} we fix $s=0.25$ and scanning
$\alpha\in\{1.3,1.5,1.7,2.0\}$, all of which satisfy $\alpha\ge1+s$. The measured
slope $\hat m$ approaches the theoretical slope $m^\star=-s/(1+s)=-0.20$ as $\alpha$ moves away from
the boundary $\alpha=1+s$ ($\hat m=-0.166,-0.188,-0.193,-0.196$ for $\alpha=1.3,1.5,1.7,2.0$
respectively). We traced this to a genuine, theoretically consistent finite-horizon effect
rather than an artifact: at a fixed noise level $\sigma$, Lyapunov's inequality
$\E[\|\vxi\|^{1+s}]\le\sigma^{1+s}$ (used in the proof of Theorem~\ref{thm:SGD-proof}) is nearly
tight when $\alpha$ is close to $1+s$ and increasingly loose as $\alpha$ grows; concretely, for our
calibrated noise the realised moment $\E[\|\vxi\|^{1+s}]/\sigma^{1+s}$ equals $83\%$ at $\alpha=1.3$
but only $38\%$ at $\alpha=2.0$. Since the constant multiplying the theorem's rate depends on
$\sigma^{1+s}$, an $\alpha$ close to the boundary realises noise close to the worst case the bound
allows, inflating the constant (and hence the burn-in horizon) without changing $m^\star$ --- we
confirmed this by running $\alpha=1.3$ out to $T=512{,}000$, where the measured slope remains
pinned at $\hat m\approx-0.16$ throughout, consistent with a longer, not a different, transient.

{\em Secondly,} for a less smooth and less heavy-tailed case with  ($s=0.2$,
$\alpha\in\{1.5,1.7,1.85,2.0\}$, all with gap $\ge0.3$) gives measured slopes that agree with one
another to within $0.006$ ($\hat m=-0.150,-0.155,-0.155,-0.156$), which is the cleaner and more
direct empirical signature of the $\alpha$-independence asserted by Theorem~\ref{thm:SGD-proof}. 

Both the above experiments are reported in Figure~\ref{fig:check1}.

\begin{figure}[ht]
\centering
\includegraphics[width=0.98\linewidth]{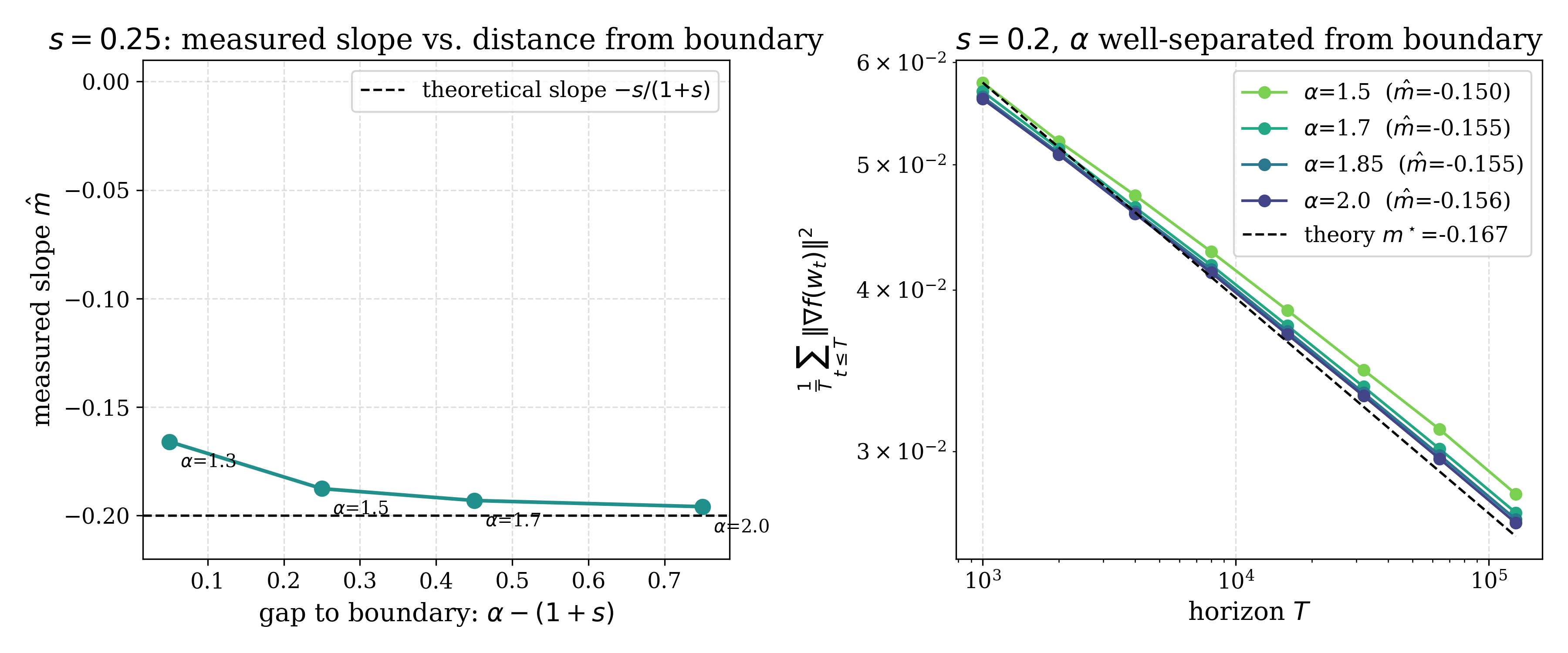}
\caption{Check 1. \emph{Left:} measured slope $\hat m$ (Eq.~\ref{eq:measured-slope}) of the
run-averaged $\|\nabla f(\vw_t)\|^2$ against $T$, plotted against the gap $\alpha-(1+s)$ to the
boundary of Assumption~\ref{ass:noise-s}, for $s=0.25$; $\hat m$ converges to the theoretical slope
$m^\star=-s/(1+s)=-0.20$ (dashed) as the gap grows, reflecting a longer finite-horizon transient
rather than a different asymptotic rate near the boundary. \emph{Right:} with $\alpha$ held at a
healthy distance from the boundary of $1+s$ with ($s=0.2$, $\alpha\in\{1.5,1.7,1.85,2.0\}$), the four curves
nearly coincide and track the theoretical slope $m^\star=-s/(1+s)=-0.167$ (dashed), directly
exhibiting the $\alpha$-independence of Theorem~\ref{thm:SGD-proof}.}
\label{fig:check1}
\end{figure}

\paragraph{Check 2: only G-Clip converges when $\alpha<1+s$.} Fixing $s=0.8$ and $\alpha=1.3$, so
that $1+s=1.8>\alpha$ ensures that neither Theorem~\ref{thm:SGD-proof} nor Theorem~\ref{thm:delta-gclip-s-3}
applies. In this setting we compare SGD, $\delta$-GClip with $\delta=0.1$, and G-Clip. SGD fails to converge: its
convergence metric does not decrease with $T$. 
G-Clip converges with $\hat m=-0.29$, almost matching its theoretical slope
$m^\star=-\frac{2s(\alpha-1)}{(\alpha-1)+s(2\alpha-1)}=-0.30$ (Theorem~\ref{thm:gclip}), with
essentially no gap between its final mean and median ($0.0188$ vs.\ $0.0188$). 

$\delta$-GClip with
$\delta=0.1$, for which no guarantee is available in this regime, nonetheless tracks G-Clip closely
in this experiment (with $\hat m=-0.28$). Thus the experiment is suggestive of there being a regime where $\delta$-GClip converges for $\alpha < 1+s$ but which is not yet visible in theory. 

\begin{figure}[ht]
\centering
\includegraphics[width=0.98\linewidth]{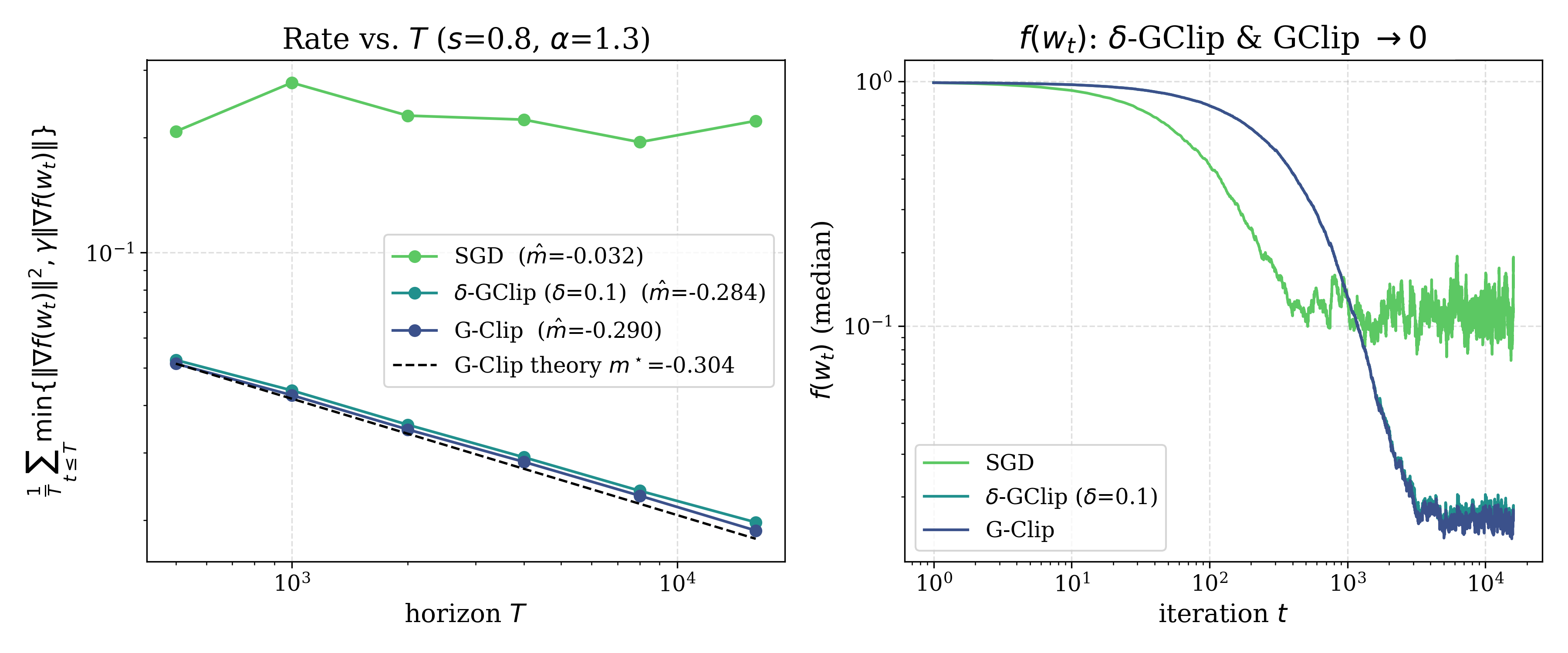}
\caption{Check 2 ($s=0.8$, $\alpha=1.3<1+s$). \emph{Left:} run-averaged stationarity metric against
the horizon $T$ for SGD, $\delta$-GClip ($\delta=0.1$), and G-Clip, with measured slopes $\hat m$
(Eq.~\ref{eq:measured-slope}) in the legend; only the clipping methods descend, with G-Clip tracking
its theoretical slope $m^\star=-0.30$ (dashed). \emph{Right:} the corresponding $f(\vw_t)$
trajectories, showing SGD stalling around $f\approx0.1$--$0.3$ while both clipping methods continue
to decrease toward the global minimum.}
\label{fig:check2}
\end{figure}



\begin{table}[ht]
\centering
\begin{tabular}{llccc}
\toprule
Check & Method & Measured slope $\hat m$ & Theory slope $m^\star$ & Guarantee? \\
\midrule
1 ($s=0.25$, near boundary) & SGD, $\alpha=1.3$ 
& $-0.166$ & $-0.20$ & \checkmark \\
1 ($s=0.25$, near boundary) & SGD, $\alpha=1.5$ 
& $-0.188$ & $-0.20$ & \checkmark \\
1 ($s=0.25$, near boundary) & SGD, $\alpha=1.7$ 
& $-0.193$ & $-0.20$ & \checkmark \\
1 ($s=0.25$, near boundary) & SGD, $\alpha=2.0$ 
& $-0.196$ & $-0.20$ & \checkmark \\
\midrule
1 ($s=0.2$, well separated) & SGD, $\alpha=1.5$ 
& $-0.150$ & $-0.167$ & \checkmark \\
1 ($s=0.2$, well separated) & SGD, $\alpha=1.7$ 
& $-0.155$ & $-0.167$ & \checkmark \\
1 ($s=0.2$, well separated) & SGD, $\alpha=1.85$ 
& $-0.155$ & $-0.167$ & \checkmark \\
1 ($s=0.2$, well separated) & SGD, $\alpha=2.0$ 
& $-0.156$ & $-0.167$ & \checkmark \\
\midrule
2 ($s=0.8,\alpha=1.3$) & SGD                            & $-0.03$ & --- & \ding{55} \\
2 ($s=0.8,\alpha=1.3$) & $\delta$-GClip ($\delta=0.1$)  & $-0.28$ & --- & \ding{55} \\
2 ($s=0.8,\alpha=1.3$) & G-Clip                         & $-0.29$ & $-0.30$ & \checkmark \\
\bottomrule
\end{tabular}
\caption{The run-averaged stationarity metric against horizon $T$ i.e the measured slope $\hat m$ (Eq.~\ref{eq:measured-slope}) vs.\ theoretical slope $m^\star$ of
, for every algorithm and parameter setting
tested. ``near boundary'' denotes that $\alpha$ is only slightly above $(1+s)$ while ``well-separate'' denotes the $\alpha$ being at least $0.3$ above $1+s$. A \checkmark\ in the last column indicates the corresponding theorem
in this work applies; a \ding{55}\ indicates no guarantee is available for that method in that
regime (so no $m^\star$ is listed).}
\label{tab:exp-summary}
\end{table}
\end{document}